\PassOptionsToPackage{table}{xcolor}
\documentclass[final,12pt]{colt2025} 

\title[Data-dependent Bounds in MABs with BOBW using SPM]{Data-dependent Bounds with $T$-Optimal Best-of-Both-Worlds Guarantees in Multi-Armed Bandits using Stability-Penalty Matching}
\usepackage{times}

\usepackage[utf8]{inputenc}
\usepackage{import}
\usepackage{amsfonts}
\usepackage{amsmath, bm, amssymb}
\allowdisplaybreaks
\usepackage[framemethod=TikZ]{mdframed}
\usepackage{multirow, multicol, array}
\usepackage{booktabs}
\usepackage{commath}
\usepackage{mathtools}
\usepackage{tikz}
\usetikzlibrary{shapes,arrows,positioning,intersections}
\usepackage[bb=dsserif]{mathalpha}

\usepackage{thm-restate}
\usepackage{mwe}

\providecommand{\P}{}
\renewcommand{\P}{\mathbb{P}}

\usepackage{amsmath,amsfonts,bm}

\def\ceil#1{\lceil #1 \rceil}

\def\1{\bm{1}}

\def\eps{{\epsilon}}
\def\sqrtfrac#1#2{\sqrt{\frac{#1}{#2}}}

\DeclareMathAlphabet{\mathsfit}{\encodingdefault}{\sfdefault}{m}{sl}
\SetMathAlphabet{\mathsfit}{bold}{\encodingdefault}{\sfdefault}{bx}{n}

\def\gA{{\mathcal{A}}}

\def\sA{{\mathbb{A}}}

\def\sF{{\mathbb{F}}}

\def\sN{{\mathbb{N}}}

\def\sS{{\mathbb{S}}}

\newcommand{\E}{\mathbb{E}}

\newcommand{\R}{\mathbb{R}}

\newcommand{\Var}{\mathrm{Var}}

\newcommand{\I}[1]{\mathbb{1}\{#1\}}
\newcommand{\defeq}{\vcentcolon=} 

\DeclareMathOperator*{\argmax}{arg\,max}
\DeclareMathOperator*{\argmin}{arg\,min}

\DeclarePairedDelimiterX{\infdivx}[2]{(}{)}{%
  #1\;\delimsize\|\;#2%
}
\newcommand{\infdiv}{KL\infdivx}
\DeclarePairedDelimiterX{\inp}[2]{\langle}{\rangle}{#1, #2}

\newcounter{protocol}

\newcommand{\vect}[1]{\ensuremath{\mathbf{#1}}}

\newenvironment{nalign}{
    \begin{equation}
    \begin{aligned}
}{
    \end{aligned}
    \end{equation}
    \ignorespacesafterend
}
\newcommand\scalemath[2]{\scalebox{#1}{\mbox{\ensuremath{\displaystyle #2}}}}

\coltauthor{%
 \Name{Quan Nguyen} \Email{manhquan233@gmail.com}\\
 \addr University of Victoria\footnote{The majority of this work was done when Quan Nguyen was at RIKEN AIP.}
 \AND
 \Name{Shinji Ito} \Email{shinji@mist.i.u-tokyo.ac.jp}\\
 \addr University of Tokyo and RIKEN AIP
 \AND
 \Name{Junpei Komiyama} \Email{junpei@komiyama.info} \\
 \addr New York University and RIKEN AIP
 \AND
 \Name{Nishant A. Mehta} \Email{nishantmehta.x@gmail.com} \\
 \addr University of Victoria
}

\begin{document}

\maketitle

\begin{abstract}%
  Existing data-dependent and best-of-both-worlds regret bounds for multi-armed bandits problems have limited adaptivity as they are either data-dependent but not best-of-both-worlds (BOBW), BOBW but not data-dependent or have sub-optimal $O(\sqrt{T\ln{T}})$ worst-case guarantee in the adversarial regime.
    To overcome these limitations, we propose real-time stability-penalty matching (SPM), a new method for obtaining regret bounds that are simultaneously data-dependent, best-of-both-worlds and $T$-optimal for multi-armed bandits problems.
    In particular, we show that real-time SPM obtains bounds with worst-case guarantees of order $O(\sqrt{T})$ in the adversarial regime and $O(\ln{T})$ in the stochastic regime while simultaneously being adaptive to data-dependent quantities such as sparsity, variations, and small losses.
    Our results are obtained by extending the SPM technique for tuning the learning rates in the follow-the-regularized-leader (FTRL) framework, which further indicates that the combination of SPM and FTRL is a promising approach for proving new adaptive bounds in online learning problems.
\end{abstract}

\begin{keywords}%
    multi-armed bandits, adaptive bounds, best-of-both-worlds, stability-penalty matching
\end{keywords}

\section{Introduction}
The multi-armed bandits problem~\citep{LaiAndRobbins1985,Auer2002a} is one of the most fundamental frameworks for modeling sequential decision making problems under limited feedback.
In this problem, a learner sequentially interacts with the environment in $T$ rounds. In round $t = 1, 2, \dots$, the learner chooses an action $I_t$ from a set of $K$ available actions and observes a numerical feedback $\ell_{t,I_t} \in \R$. 
This $\ell_{t,I_t}$ is an element of a hidden vector $\ell_t \in \R^K$ chosen at the beginning of round $t$ by an oblivious adversary. 
The performance of the learner is its \emph{pseudo-regret}
\begin{align}
    R_T = \max_{a \in [K]}R_{T, a} = \max_{a \in [K]}\E\left[\sum_{t=1}^T \ell_{t,I_t} - \ell_{t,a}\right],
    \label{eq:regretDefinition}
\end{align}
where $\E$ denote the expectation taken over all randomness from all $T$ rounds.
Existing works have constructed algorithms with \textit{worst-case} regret bounds that hold under the assumption on whether the adversary is adversarial (i.e., $(\ell_t)_t$ are arbitrary) or stochastic (i.e., $(\ell_{t})_t$ are drawn i.i.d. from some distribution)~\citep{LaiAndRobbins1985,Auer2002a,EXP3Auer2002b}, \textit{best-of-both-worlds} (BOBW) bounds that have worst-case guarantees simultaneously for adversarial and stochastic adversaries~\citep[e.g.][]{Bubeck2012bSAO,Zimmert2021TsallisINF,DannCOLT2023aBlackbox,ItoCOLT2024}, or \textit{data-dependent} bounds that are adaptive to the sequence $(\ell_t)_t$~\citep[e.g.][]{WeiAndLuo2018aBroadOMD,Bubeck2018ALT, Ito2021HybridDataMABBound,ItoCOLT2022aVariance,Tsuchiya2023stabilitypenaltyadaptive}.
Despite this vast amount of literature on different types of worst-case and adaptive bounds for multi-armed bandits, we are not aware of any works that establish bounds that are \emph{simultaneously} data-dependent, best-of-both-worlds \emph{and} have optimal dependency on $T$. 
In particular, existing works suffer from at least one of three limitations: being data-dependent but not BOBW~\citep{HazanAndKale11a,Bubeck2018ALT,WeiAndLuo2018aBroadOMD}, being BOBW but not data-dependent~\citep{Zimmert2021TsallisINF, DannCOLT2023aBlackbox} or having sub-optimal dependency on $T$~\citep{HazanAndKale11a,WeiAndLuo2018aBroadOMD,Tsuchiya2023stabilitypenaltyadaptive, ItoCOLT2024}. In this work, we close this gap in the literature by introducing novel algorithms with regret bounds that are simultaneously BOBW, data-dependent and $T$-optimal.

All of our algorithms are established in the Follow-the-Regularized-Leader (FTRL) framework~\citep[see e.g.][]{BanditAlgorithmsBook2020}, in which the time-varying learning rates are tuned by the Stability-Penalty Matching (SPM) method. 
SPM was originally proposed by~\cite{ItoCOLT2024} as a principled method for tuning learning rates in FTRL using both the \emph{penalty} and \emph{stability} terms.
More specifically, in round $t$, our algorithms compute a probability vector
\begin{align}
    q_t = \argmin_{x \in \Delta_K} \inp{L_{t-1}}{x} + \phi_t(x),
\end{align}
where $\Delta_K = \{x \in \R^K_+: \sum_{i=1}^K x_i = 1\}$ denotes the $K$-dimensional simplex, $L_{t-1} \in \R^K$ is the estimated cumulative loss vector up to round $t-1$ and $\phi_t(x): \Delta_K \to \R$ is the regularization function. We use the following specific form for $\phi_t$:
\begin{align}
    \phi_t(x) = \beta_t f(x) + \gamma u(x),
    \label{eq:formOfPhiT}
\end{align}
where $f(x): \Delta_K \to \R_{-}, u(x): \Delta_K \to \R_{+}$ are convex, $\beta_t > 0$ is the learning rate in round $t$ and $\gamma$ is a constant. Then, the learner draws an arm $I_t \sim q_t$ according to $q_t$ (or some $p_t \in \Delta_K$ derived from $q_t$) and computes an estimated loss vector $\hat{\ell}_t$.
Let $D_t(x, y) = \phi_t(x) - \phi_t(y) - \inp{\nabla \phi_t(y)}{x - y}$ denote the Bregman divergence associated with $\phi_t$.
The standard analysis of FTRL~\citep[e.g.][Exercise 28.12]{BanditAlgorithmsBook2020} implies that
\begin{align*}
    R_{T,a} 
     &\lesssim \phi_{T+1}(e_a) - \phi_1(q_1) + \scalemath{0.99}{
     \E\left[\sum_{t=1}^T (\phi_t(q_{t+1}) - \phi_{t+1}(q_{t+1})) + \sum_{t=1}^T (\inp{\hat{\ell}_t}{q_t - q_{t+1}} - D_{t}(q_{t+1}, q_t))\right]
     } \\
    &\lesssim \gamma u(e_a) - \beta_1f(q_1) + \underbrace{
    \E\left[\sum_{t=1}^T (\beta_{t+1} - \beta_t)h_{t+1}\right]
    }_{\text{penalty term}} + \underbrace{
    \E\left[\sum_{t=1}^T\frac{z_t}{\beta_t}\right]
    }_{\text{stability term}}
\end{align*}
where $e_a$ is the $a$-th vector in the standard basis of $\R^K$, $h_{t+1}$ satisfies $ (-f(q_{t+1})) \lesssim h_{t+1}$ and $z_t$ satisfies $\beta_t \E[\inp{\hat{\ell}_t}{q_t - q_{t+1}} - D_{t}(q_{t+1}, q_t)] \lesssim \E[z_t]$. 
SPM carefully chooses $\beta_1, z_t$ and $h_t$ so that $h_{t+1} \leq O(h_t)$ and sets the learning rate of the next round to be
\begin{align}
    \beta_{t+1} = \beta_t + \frac{z_t}{\beta_t h_t}.
    \label{eq:SPM}
\end{align}
This makes $(\beta_{t+1} - \beta_t)h_{t+1}$ match with $\frac{z_t}{\beta_t}$ and implies 
$
    R_{T,a} \lesssim \gamma u(e_a) - \beta_1 f(q_1) + \E\left[\sum_{t=1}^T\frac{z_t}{\beta_t}\right].
$
An important insight in SPM is that by picking $f(x)$ and $u(x)$ appropriately, $\E\left[\sum_{t=1}^T \frac{z_t}{\beta_t}\right] $ is naturally adaptive to the adversarial or stochastic nature of the environment~\citep{ItoCOLT2024}. 
In our work, we will show that SPM can be made adaptive not only to the nature of the environment but also to the underlying structure of the sequence of losses such as sparsity and total variation.

\subsection{Main Contributions and Techniques}
Throughout the paper, we will write $O(\square \ln(T), \square \sqrt{T})$ to denote a BOBW bound that holds for stochastic and adversarial regimes, respectively, where $\square$ contains problem-dependent terms.
The original SPM method~\citep{ItoCOLT2024} used $z_t = \Omega(\beta_t \E_{I_t}[\inp{\hat{\ell}_t}{q_t - q_{t+1}} - D_{t}(q_{t+1}, q_t)])$, where $\E_{I_t}$ denotes an expectation taken over $I_t$. 
Because only one out of $K$ arms is observable in each round $t$, this in-expectation form of $z_t$ inevitably requires taking the trivial bounds (e.g. $1$) of the losses into its computation, thus limiting its adaptivity to $(\ell_t)_{t}$. Our work overcomes this limitation by setting
\begin{align}
    z_t = \Omega(\beta_t (\inp{\hat{\ell}_t}{q_t - q_{t+1}} - D_{t}(q_{t+1}, q_t))).
    \label{eq:realtimeZt}
\end{align}
We call this \emph{real-time SPM}, since $z_t$ depends on the observed arm $I_t$.
The main technical challenge is now $z_t$ can be very large since it grows with $\mathrm{poly}(\frac{1}{p_{t,I_t}})$. 
At the same time, we need to limit the amount of explicit exploration to obtain a BOBW bound for stochastic bandits. 
Table~\ref{table:results} summarizes our main results, showing that real-time SPM can be controlled effectively to give BOBW \emph{and} data-dependent bounds with optimal dependency on $T$.
Our results also hold for the more general adversarial regime with self-bounding constraint setting~\citep{Zimmert2021TsallisINF}.
Appendix~\ref{sec:RelatedWorks} gives a more detailed discussion on related works.
Our paper is organized as follows: 
\begin{itemize}
    \item Section~\ref{sec:SPMwithRealTimeStability} introduces the real-time SPM method and states Lemma~\ref{lemma:refinedBoundF}, a key technical lemma for bounding $\E\left[\sum_{t=1}^T \frac{z_t}{\beta_t} \right]$.
    While the original analysis of SPM~\citep{ItoCOLT2024} relies on having a small $\max_{t \in [T]}z_t$ and thus cannot be applied to real-time SPM, our Lemma~\ref{lemma:refinedBoundF} instead shows that real-time SPM incurs an additional regret of at most $O\left(\max_{t \in [T]}\frac{z_t}{\beta_t}\ln{\sum_{t=1}^T \frac{z_t}{h_t}}\right)$. 
    Moreover, both $\frac{z_t}{\beta_t}$ and $\frac{z_t}{h_t}$ can be effectively controlled by appropriate choices of $\phi_t(x)$.
    \item Section~\ref{sec:BOBWboundsSparseLosses} considers the bandits problems with signed sparse losses, where $\ell_{t,i} \in [-1, 1]$ and $\norm{\ell_t}_0 \leq S$. We show that using $\alpha$-Tsallis entropy and log-barrier functions in place of $f$ and $g$ in~\eqref{eq:formOfPhiT} leads to an $O\left(\frac{(K^{1-\alpha}-1)S^\alpha \ln(T)}{\alpha (1-\alpha) \Delta_{\min}}, \left(\sqrt{\frac{(K^{1-\alpha} - 1)S^\alpha T}{\alpha(1-\alpha)}}\right)\right)$. 
    This bound is $T$-optimal and improves upon the best known bound for this setting established by~\citet{Tsuchiya2023stabilitypenaltyadaptive}.
    When $S$ is known, we show that the adversarial bound is improved to $O(\sqrt{ST\ln(K/S)})$, resolving an open question in~\citet{KwonAndPerchet2016JMLRSparsity}. 
    Furthermore, we prove a near-matching lower bound for problems in which the sparsity constraint holds in expectation. 
    \item \looseness=-1 Section~\ref{sec:SPMwOFTRLwReservoirSampling} considers problems with small total variation $Q$ (defined in Section~\ref{sec:ProblemSetup}) and presents a new algorithm obtaining a $O\left(\frac{(K- 1)^{1-\alpha} K^\alpha \ln(T)}{\alpha(1-\alpha)\Delta_{\min}}, \sqrt{Q\ln(K)}\right)$ BOBW bound. 
    In the adversarial regime, the $O(\sqrt{Q\ln(K)})$ bound matches the best known bound in~\citet{Bubeck2018ALT} while having the advantage of not requiring knowledge of $Q$.
    \item Section~\ref{sec:CoordinateWiseSPM} introduces a new SPM method called coordinate-wise SPM, which maintain arm-dependent learning rates $\beta_{t,i}$ and performs real-time SPM on each arm separately. 
    We show that coordinate-wise SPM achieves a BOBW bound with order $O\left(\frac{1}{\alpha (1-\alpha)}\sum_{i \neq i^*} \frac{\ln(T)}{\Delta_i}\right)$ in stochastic bandits and $O\left(\min\left\{ \sqrt{K\ln(T)\min(Q_\infty, L^*, T-L^*)}, K^{\frac{\alpha}{2}}\sqrt{KT}\right\}\right)$ in adversarial bandits,
    where $Q_\infty$ and $L^*$ are $\ell_\infty$-norm total variation and total loss of the best arm, respectively (see Section~\ref{sec:ProblemSetup} for their formal definitions).
\end{itemize}
\begin{table}[htbp]
    \setlength{\tabcolsep}{1.5pt}
    \caption{Summary of data-dependent results in existing and ours works.
    The three blocks of rows show bounds dependent on sparsity $S$, total variation $Q$ and a combination of $Q_\infty$ and $L^*$, respectively (formal definitions are in Section~\ref{sec:ProblemSetup}).
    We use $H^*_{\infty}= \min(Q_\infty, L^*, T-L^*)$.
    ``$T$-opt BOBW'' denote whether a bound is BOBW and $T$-optimal.
    ``Param-free'' denote whether a bound requires knowledge of the data-dependent quantities.
    } 
    \label{table:results}
    \centering
    \begin{tabular}{lllcc}
       \toprule      
      Algorithms & Stochastic & Adversarial & $T$-opt BOBW? & Param-free? \\
      \midrule
    \citet{Bubeck2018ALT} & $-$ & $\sqrt{ST\ln{K}}$ & $\times$ & $\times$ \\    
    \citet{Tsuchiya2023stabilitypenaltyadaptive} & $\frac{S\ln(T)\ln(KT)}{\Delta_{\min}} $ & $\sqrt{ST\ln{T}\ln{K}}$ & $\times$ & $\checkmark$ \\
    \cellcolor{lightgray}
    {Theorem~\ref{thm:BOBWbounds}} & $\frac{S\ln{T}\ln{K}}{\Delta_{\min}} $ & $\sqrt{ST\ln{K}}$ & $\checkmark$ &$\checkmark$ \\
    \midrule
    \citet{HazanAndKale11a} & $-$ & $\sqrt{Q\ln{T}\ln{K}}$ & $\times$ & $\checkmark$ \\
    \citet{Bubeck2018ALT} & $-$ & $\sqrt{Q\ln{K}}$  & $\times$ & $\times$ \\
    \cellcolor{lightgray}
    {Theorem~\ref{thm:BOBWsqrtQLnKBound}} & $\frac{K\ln{T}}{\Delta_{\min}}$ & $\sqrt{Q\ln{K}}$ & $\checkmark$ & $\checkmark$ \\
    \midrule
    \citet{WeiAndLuo2018aBroadOMD} & $\frac{K\ln{T}}{\Delta_{\min}}$ & $\sqrt{KL^*\ln{T}}$ & $\times$ &  $\checkmark$\\
    \citet{Ito2021HybridDataMABBound} & $\sum\limits_{i \neq i^*}\frac{\ln{T}}{\Delta_i}$ & $\sqrt{K\min(Q_\infty, L^*)\ln{T}}$ & $\times$ &  $\checkmark$\\
    \citet{ItoCOLT2022aVariance} & $\sum\limits_{i \neq i^*}(\frac{\sigma^2_i}{\Delta_i} + 1)\ln{T}$ & $\sqrt{KH^*_{\infty}\ln{T}}$ & $\times$ &  $\checkmark$\\
    \cellcolor{lightgray}
    {Theorem~\ref{thm:CWSPMbounds}} & $\sum\limits_{i \neq i^*}\frac{\ln{T}}{\Delta_i}$ & $\min(\sqrt{KH^*_{\infty}\ln{T}},\sqrt{K^{1+\alpha}T})$ & $\checkmark$ & $\checkmark$\\
      \bottomrule
    \end{tabular}
  \end{table}

\subsection{Problem Setup}
\label{sec:ProblemSetup}
For an integer $N$, let $[N] = \{1, 2, \dots, N\}$ denote the set of integers from $1$ to $N$. 
%
We study the multi-armed bandits problem~\citep{LaiAndRobbins1985,Auer2002a} in which a learner is given $K$ arms and interacts with the environment in $T$ rounds.  
In each round $t$, an adversary selects a hidden vector $\ell_{t} = (\ell_{t, 1}, \ell_{t,2}, \dots, \ell_{t,K})^\top$. 
The learner chooses one arm $I_t \in [K]$ and observes its loss $\ell_{t,I_t}$. We assume $\abs{\ell_{t,i}} \leq 1$ for all $t \in [T], i \in [K]$.
The learner aims to minimize its 
{regret} $R_T$ over $T$ rounds, defined by Equation~\eqref{eq:regretDefinition}.

We are interested in developing learning algorithms with provable upper bounds on $R_T$ that hold simultaneously for two regimes: adversarial~\citep{Auer2002a} and adversarial with a $(\Delta, C, T)$ self-bounding constraint~\citep{Zimmert2021TsallisINF}. In the \textit{adversarial regime}, no assumption is made on how the adversary generates $(\ell_t)_{t \in [T]}$. The adversarial regime with a $(\Delta, C, T)$ self-bounding constraint~\citep{Zimmert2021TsallisINF} is given below.
\begin{definition} (Adversarial regime with a self-bounding constraint)
    For $T \geq 1, \Delta \in [0,1]^K$ and $C \geq 0$, the problem is in adversarial regime with a $(\Delta, C, T)$ self-bounding constraint if the regret of any algorithm at time $T$ satisfies
    $
        R_T \geq \sum_{t=1}^T\sum_{i=1}^K \Delta_i \P(I_t = i) - C.
    $
    \label{def:selfboundingconstraint}
\end{definition}
As noted in~\citet{Zimmert2021TsallisINF}, the stochastic bandits setting~\citep{LaiAndRobbins1985} satisfies Definition~\ref{def:selfboundingconstraint}.
We also use the common assumption that there exists an optimal arm $i^*$ such that $\Delta_i > 0$ for all $i \neq i^*$, that is, the optimal arm is unique. Let $\Delta_{\min} = \min_{i \in [K]}\{\Delta_i: \Delta_i > 0\}$.

We focus on obtaining bounds that are adaptive not only to the adversary's regime but also to the data-dependent properties of the loss sequence $(\ell_t)_{t \in [T]}$.
The following data-dependent quantities are considered in our work.
\begin{itemize}
    \item \textbf{Sparsity of losses}~\citep{KwonAndPerchet2016JMLRSparsity}. All loss vectors have at most $1 \leq S \leq K$ non-zero elements, i.e., $\norm{\ell_t}_0 \leq S$, where $S$ is unknown.
    \item \textbf{Variation of losses}~\citep{HazanAndKale11a,ItoCOLT2022aVariance} The total variation of the sequence $(\ell_t)_t$ is $Q = \sum_{t=1}^T \norm{\ell_t - \frac{1}{T}\sum_{s=1}^T \ell_s}^2_2$. The $\ell_\infty$-norm total variation is $Q_\infty = \min_{\bar{\ell} \in [0,1]^K}\sum_{t=1}^T \norm{\ell_t - \bar{\ell}}^2_\infty$.
    \item \textbf{Best-arm loss}. For non-negative losses, we consider the cumulative loss of the best arm $L_{*} = \min_{i \in [K]}\sum_{t=1}^T \ell_{t,i}$.
\end{itemize}

\section{Stability-Penalty Matching with Real-Time Stability Term}
\label{sec:SPMwithRealTimeStability}
Let $\tilde{p} = \min(1 - p, p)$ for $p \in [0, 1]$. We use the notation $f \lesssim g$ to denote $f = O(g)$.
To obtain data-dependent bounds using SPM, we use SPM where the stability term is a function of the \emph{observed} loss, i.e., $z_t$ satisfies Equation~\eqref{eq:realtimeZt}.
Note that $z_t$ grows with $\ell_{t,I_t}^2$ and $\frac{1}{p_{t,I_t}}$.
The benefit of this real-time $z_t$ is that data-dependent quantities 
such as sparsity and total variation naturally come out of $\E[z_t]$.
For example, in Algorithm~\ref{algo:optimalBOBWSparsity} for bandits with sparse losses $\norm{\ell_t}_0 \leq S$, we use $z_t = O(\tilde{p}_{t,I_t}^{2-\alpha}\frac{\ell_{t,I_t}^2}{p_{t,I_t}^2})$ for some $\alpha \in (0,1)$.
It follows that $\E[z_t] = O( \sum_{i: \ell_{t,i} \neq 0} \tilde{p}_{t,i}^{1-\alpha}\ell_{t,i}^2) \leq O(S^\alpha)$, leading to the $O(\sqrt{ST\ln{K}})$ bound.
The main challenge in using the real-time $z_t$ is the value of $z_t$ can be unbounded whenever $p_{t,I_t}$ is very small.
It follows that $z_{\max} = \max_{t \in [T]} z_t$ can be unbounded, which makes it difficult to apply existing techniques in~\citet[Lemma 10]{ItoCOLT2024} that bounds $\E[\sum_{t=1}^T \frac{z_t}{\beta_t}]$ by a quantity that grows with $\E[z_{\max}]$.
We resolve this challenge by using the following technical lemma.
\begin{lemma}
    For any $T \geq 1, z_{1:T} \geq 0, h_{1:T} > 0$ and a sequence $\beta_{1:T}$ defined by Equation~\eqref{eq:SPM}, let $F(z_{1:T}, h_{1:T}) = \sum_{t=1}^T \frac{z_t}{\beta_t}$ and $G(z_{1:T}, h_{1:T}) = \sum_{t=1}^T \frac{z_t}{\sqrt{\sum_{s=1}^t \frac{z_s}{h_s}}}$. We have
    \begin{align}
          F(z_{1:T}, h_{1:T}) \lesssim G(z_{1:T}, h_{1:T}) + \left(\max_{t \in [T]}\frac{z_t}{\beta_t}\right)\ln\left(\sum_{t=1}^T \frac{z_t}{h_t}\right).
    \end{align}
    \label{lemma:refinedBoundF}
\end{lemma}
\begin{proof}(Sketch)
    Our proof extends from the proof of~\citet[Lemma 3]{ItoCOLT2024}.
    Similar to the proof of~\citet[Lemma 10]{ItoCOLT2024}, we define a new sequence $\beta'_t = \sqrt{\beta_1^2 + 2\sum_{s=1}^{t-1}\frac{z_s}{h_s}}$ and consider the set of rounds $E = \{t \in [T]: \beta'_{t+1} \geq \sqrt{2}\beta'_t\}$. 
    The complement of $E$ is $E^c = [T] \setminus E$. We have
    \begin{align*}
        F(z_{1:T}, h_{1:T}) &= \underbrace{\sum_{t \in E^c}\frac{z_t}{\beta_t}}_{(a)} + \underbrace{\sum_{t \in E}\frac{z_t}{\beta_t}}_{(b)},
    \end{align*}
    where $(a)$ is bounded by $G(z_{1:T}, h_{1:T})$ as in~\citet[Lemma 10]{ItoCOLT2024}, and $(b)$ is bounded by
    \begin{align*}
        (b) \leq \left(\max_{t \in [T]}\frac{z_t}{\beta_t}\right)\abs{E} \leq \left(\max_{t \in [T]}\frac{z_t}{\beta_t}\right)\log_{\sqrt{2}}\left(\frac{\beta'_{T+1}}{\beta'_1}\right) \lesssim \left(\max_{t \in [T]}\frac{z_t}{\beta_t}\right)\ln\left(\sum_{t=1}^T \frac{z_t}{h_t}\right),
    \end{align*}
    where the second inequality is from the fact that $\beta'_t$ is multiplied by at least $\sqrt{2}$ after every round in $E$; thus there can be at most $\log_{\sqrt{2}}\frac{\beta'_{T+1}}{\beta'_1}$ such multiplications.
\end{proof}
Lemma~\ref{lemma:refinedBoundF} implies that if (I) the sum $\E[\sum_{t=1}^T \frac{z_t}{h_t}]$ grows with $\mathrm{poly}(T)$ and (II) $\max_{t}\frac{z_t}{\beta_t}$ is small, then $\E[F(z_{1:T}, h_{1:T})]$ grows dominantly with $\E[G(z_{1:T}, h_{1:T})]$ plus an $O(\ln(T))$ term. Hence, we can safely ignore other terms and focus only on bounding $G(z_{1:T}, h_{1:T})$. The proof of~\citet[Lemma 10]{ItoCOLT2024} already showed that \begin{align}
    G(z_{1:T}, h_{1:T}) \lesssim \min\left\{ \sqrt{\ln(T)\sum_{t=1}^T h_tz_t} + \sqrt{\frac{1}{T}h_{\max}\sum_{t=1}^T z_t},\sqrt{h_{\max}\sum_{t=1}^T z_t}\right\}.
    \label{eq:boundG}
\end{align}
In the rest of the paper, we will show that different choices of the (hybrid) regularization function lead to specific forms of $z_t$ and $h_t$ such that not only do both conditions (I) and (II) hold but also they 
{imply} BOBW data-dependent bounds with optimal dependency on $T$ from~\eqref{eq:boundG}.

\section{Application I: BOBW Bounds for Bandits with Sparse Losses}
\label{sec:BOBWboundsSparseLosses}
\begin{algorithm}[t]
	\KwIn{$K \geq 3, T \geq 4K, \alpha \in (0, 1), \beta_1 = \frac{8K}{1-\alpha}, \gamma = \max(6,48\sqrtfrac{\alpha}{1-\alpha}), d = 2$.}
    Initialize $L_{0,i} = 0$ for $i \in [K]$ \;
	
    \For{each round $t = 1, \dots, T$}{        
        Compute $q_t = \argmin_{p \in \Delta_K} \inp{L_{t-1}}{p} + \beta_t\left(\frac{1}{\alpha}(1-\sum_{i=1}^K p_i^\alpha)\right) - \gamma\sum_{i=1}^K \ln(p_i)$\;

        Compute $p_t = \left(1 - \frac{K}{T}\right)q_t + \frac{1}{T}\1$\;

		Draw $I_t \sim p_t$ and observe $\ell_{t, I_t}$\; 

		Compute loss estimate $\hat{\ell}_{t, i} = \frac{\ell_{t,i}\I{I_t = i}}{p_{t,i}}$ and update $L_{t,i} = L_{t-1,i} + \hat{\ell}_{t,i}$\;

        Compute $z_t =  \min\left( \frac{(6d)^{2-\alpha}}{2(1-\alpha)}\min(p_{t,I_t}, 1 - p_{t,I_t})^{2-\alpha}\hat{\ell}^2_{t,I_t}, \frac{\beta_t18d^2}{\gamma} \ell_{t,I_t}^2  \right)$ \;

        Compute $h_t = \left(\frac{1}{\alpha}(\sum_{i=1}^K p_{t,i}^\alpha - 1)\right)$\;

        Compute $\beta_{t+1} = \beta_t + \frac{z_t}{\beta_t h_t}$\;
	}
	\caption{Real-time SPM with hybrid regularization for losses in $[-1,1]$.}
	\label{algo:optimalBOBWSparsity}
\end{algorithm}
We consider the multi-armed bandits setting with sparse losses~\citep{KwonAndPerchet2016JMLRSparsity}, in which the loss vector $\ell_t \in [-1, 1]^K$ has at most $S$ non-zero elements, i.e., $\max_{t \in [T]}\norm{\ell_t}_0 \leq S$. 
Note that $S$ is unknown to the learner.
Let $\psi_{TE}(p) = \frac{1}{\alpha}(1-\sum_{i=1}^K p_i^\alpha)$
be the $\alpha$-Tsallis entropy with some $\alpha \in (0, 1)$ and $\psi_{LB}(p) = -\sum_{i=1}^K \ln(p_i)$ be the log-barrier function.
Our approach for this setting is in Algorithm~\ref{algo:optimalBOBWSparsity}, in which we use the hybrid regularizer $\phi_t(p) = \beta_t\psi_{TE}(p) + \gamma\psi_{LB}(p)$ to obtain 
\begin{align*}
    q_{t} = \argmin_{p \in \Delta_K}\{\inp{L_{t-1}}{p} + \beta_t\psi_{TE}(p) + \gamma\psi_{LB}(p)\},
\end{align*}
Then, we mix $q_t$ with $\frac{1}{T}$-uniform exploration to obtain the sampling probability $p_t$, i.e., $ p_{t} = \left(1 - \frac{K}{T}\right)q_t + \frac{1}{T}\1$.
The learning rates $(\beta_t)_t$ are set by the SPM rule by~\cite{ItoCOLT2024} as
\begin{align}
    \beta_1 = \frac{8K}{1-\alpha}, \,\, \beta_{t+1} = \beta_{t} + \frac{z_t}{\beta_t h_t},
    \label{eq:betatplus1}
\end{align}
where 
\begin{align}
    z_t =  \min\left( \frac{(6d)^{2-\alpha}}{2(1-\alpha)}\min(p_{t,I_t}, 1 - p_{t,I_t})^{2-\alpha}\hat{\ell}^2_{t,I_t}, \beta_t\frac{18d^2}{\gamma} \ell_{t,I_t}^2  \right), \,\, h_t = (-\psi_{TE}(p_{t})),
    \label{eq:setzthtforsparsebandits}
\end{align}
and $\gamma = \max(6,48\sqrtfrac{\alpha}{1-\alpha}), d = 2$.
Note that $\beta_1 \geq \frac{4K}{(\omega -1)(1 - \omega^{\alpha-1})}$ for $\omega = 2$.
The following theorem states the BOBW bounds of Algorithm~\ref{algo:optimalBOBWSparsity}.
\begin{theorem}
    For any $K \geq 4, T \geq 4k$, Algorithm~\ref{algo:optimalBOBWSparsity} guarantees the following bounds simultaneously
    \begin{itemize}
        \item In the adversarial regime: 
        \begin{align}
            R_T \leq O\left(\sqrt{\frac{(K^{1-\alpha} - 1)S^\alpha T}{\alpha(1-\alpha)}}\right)
            \label{eq:boundSparseBanditsAdversarial}
        \end{align}
        \item In the adversarial regime with a self-bounding constraint:
        \begin{align}
            R_T \leq O\left(\frac{(K-1)^{1-\alpha}S^\alpha \ln(T)}{\alpha(1-\alpha)\Delta_{\min}} + \sqrt{C\frac{(K-1)^{1-\alpha}S^\alpha \ln(T)}{\alpha(1-\alpha)\Delta_{\min}}} + \sqrt{\frac{(K-1)^{1-\alpha}S^{\alpha}}{\alpha(1-\alpha)}}\right)
            \label{eq:boundSparseBanditsAdversarialWithSelfBounding}
        \end{align}
    \end{itemize}
    \label{thm:BOBWbounds}
\end{theorem}
\looseness=-1 In Appendix~\ref{appendix:setAlphaCloseto1}, we show that by setting $\alpha = 1 - \frac{1}{2\ln(K)}$, we obtain $\frac{(K^{1-\alpha} - 1)S^{\alpha}}{\alpha(1-\alpha)} \lesssim S^\alpha\ln(K)$ and $\frac{(K-1)^{1-\alpha}S^{\alpha}}{\alpha(1-\alpha)} \lesssim S^\alpha\ln(K)$. 
In the adversarial regime, Theorem~\ref{thm:BOBWbounds} recovers the $O(\sqrt{ST\ln(K)})$ bound in~\cite{Bubeck2018ALT} and~\cite{Tsuchiya2023stabilitypenaltyadaptive} while still being $S$-agnostic. 
In the adversarial regime with a self-bounding constraint, the bound becomes $O(\frac{S\ln(K)\ln(T)}{\Delta_{\min}})$ which has an optimal dependency on $T$. To the best of our knowledge, Theorem~\ref{thm:BOBWbounds} is the first result for bandits with sparse signed losses that is simultaneously $S$-agnostic, $T$-optimal and BOBW.
Also, our approach is more computationally efficient than that of~\cite{Tsuchiya2023stabilitypenaltyadaptive} as we do not need to solve any additional optimization problems to compute the learning rates $(\beta_t)_t$.
\begin{remark}
    When $S$ is known, then $\frac{(K^{1-\alpha} - 1)S^{\alpha}}{\alpha(1-\alpha)}$ 
    can be further bounded by $6S\ln(\frac{K}{S})$. Consider only the case where $S$ is sufficiently small so that $e^2S \leq K$ (the other direction trivially leads to $O(\frac{S}{\alpha(1-\alpha)})$). 
    Letting $\alpha = 1- \frac{1}{\ln(K/S)}$, then $(\frac{K-1}{S})^{1-\alpha} \leq (\frac{K}{S})^{1-\alpha} = e$.  Since $\ln(K/S) \geq 2$, we have $\alpha \geq \frac{1}{2}$. Therefore,
        \begin{align*}
            \frac{(K^{1-\alpha} - 1)S^{\alpha}}{\alpha(1-\alpha)} &\leq \frac{K^{1-\alpha}S^{\alpha}}{\alpha(1-\alpha)} = S \left(\frac{K}{S}\right)^{1-\alpha} \frac{\ln(K/S)}{\alpha} = \frac{eS \ln(K/S)}{\alpha} \leq 6S\ln(K/S).
        \end{align*}
        This result shows that an $O(\sqrt{ST\ln(K/S)})$ upper bound is attainable even for signed losses, which resolves an open question posed in~\cite[Remark 12]{KwonAndPerchet2016JMLRSparsity}.
    \end{remark} 
\subsection{Proof Sketch for Theorem~\ref{thm:BOBWbounds}}
As mentioned in Section~\ref{sec:SPMwithRealTimeStability}, we first show that $z_t$ and $h_t$ in~\eqref{eq:setzthtforsparsebandits} satisfy the two conditions (I) $\E[\sum_{t=1}^T \frac{z_t}{h_t}] = O(\mathrm{poly}(T))$ and (II) $\max_t\frac{z_t}{\beta_t}$ is small. The second condition is straightforward from the definition of $z_t, \gamma$ and $d$, since $
    \frac{z_t}{\beta_t} \leq \frac{18d^2}{\gamma} \leq 6d^2 = 24$
which is a constant. To see that (I) is true, note that $h_t$ is fixed with respect to $I_t$. Hence,
\begin{align*}
    \E\left[\sum_{t=1}^T \frac{z_t}{h_t}\right] &= \E\left[\sum_{t=1}^T \frac{\E_{I_t}[z_t]}{h_t}\right] \leq T\E\left[\frac{\max_{t \in [T]}\E_{I_t}[z_t]}{\min_{t \in [T]}h_t}\right].
\end{align*}
Then, the condition (I) follows from Lemma~\ref{lemma:minhtAndmaxEzt}, which shows that $h_t \geq \frac{1-\alpha}{4\alpha}T^{-\alpha}$ and $E_{I_t}[z_t] \leq \frac{(6d)^{2-\alpha}}{2(1-\alpha)} S^{\alpha}$.
Jensen's inequality implies that $\E\left[\ln\left(\sum_{t=1}^T \frac{z_t}{h_t}\right)\right] \leq \ln\left(\E\left[\sum_{t=1}^T \frac{z_t}{h_t}\right]\right)$. 
Combining this with Lemma~\ref{lemma:minhtAndmaxEzt}, we conclude that $\E[F(z_{1:T}, h_{1:T})]$ grows dominantly with $\E[G(z_{1:T}, h_{1:T})]$. 
The last part of the proof is showing that plugging $z_t$ and $h_t$ from~\eqref{eq:setzthtforsparsebandits} into~\eqref{eq:boundG} yields the desired bounds.
In the adversarial regime, the bound~\eqref{eq:boundSparseBanditsAdversarial} follows directly from~\eqref{eq:boundG}, Lemma~\ref{lemma:minhtAndmaxEzt}, $h_t \leq \frac{K^{1-\alpha}-1}{\alpha}$ and Jensen's inequality $\E[\sqrt{X}] \leq \sqrt{\E[X]}$.
In the adversarial regime with a self-bounding constraint, we can prove~\eqref{eq:boundSparseBanditsAdversarialWithSelfBounding} by first showing that
\begin{align*}
    \E[h_tz_t] &\leq \frac{(6d)^{2-\alpha}}{\alpha(1-\alpha)\Delta_{\min}}{(K-1)^{1-\alpha}S^{\alpha}}\E\left[\sum_{i=1}^K p_{t,i}\Delta_i\right].    
\end{align*}
and then following the same argument as in~\citet{ItoCOLT2024}.

\subsection{A Lower Bound for Problems with Soft Sparsity Constraint}
\looseness=-1 It remains an open question whether the BOBW bounds in Theorem~\ref{thm:BOBWbounds} are tight under the hard constraint $\norm{\ell_t}_0 \leq S$. 
This hard-constraint problem belongs to a broader class of settings with a more relaxed constraint, in which there exists an $\alpha \in (0,1)$ and $1 \leq U \leq K^\alpha$ such that for all $t \in [T]$,
\begin{align}
    \E\left[\left(\sum_{i=1}^K \abs{\ell_{t,i}}^{2/\alpha}\right)^\alpha\right] \leq U.
    \label{eq:softconstraint}
\end{align}
In other words, the sparsity constraint holds in expectation. Obviously, the hard-constraint setting with $\norm{\ell_t}_0 \leq S$ satisfies~\eqref{eq:softconstraint} for any $\alpha \in (0,1)$ and $U = S^\alpha$.
Moreover, by using the same Algorithm~\ref{algo:optimalBOBWSparsity} and straightforward modifications in its proof, we can obtain the corresponding $O(\frac{K^{1-\alpha}U}{\alpha(1-\alpha)\Delta_{\min}}\ln{T})$ and $O(\sqrt{\frac{K^{1-\alpha}}{\alpha (1-\alpha)}UT})$ BOBW bounds for stochastic and adversarial regimes, respectively.
The following theorem, whose proof is in Section~\ref{sec:lowerboundproofs}, shows near-matching lower bounds for problems with soft sparsity constraint defined in~\eqref{eq:softconstraint}.
\begin{theorem}
    (Instance-Dependent Lower Bound) For any consistent algorithm, for any $K \geq 4, \alpha \in (0,1)$ and $1 \leq U \leq \frac{K^\alpha}{4}$, there exists a $K$-armed stochastic bandit instance with $\Delta_{\min} \in (0,1)$ and loss distribution satisfying~\eqref{eq:softconstraint} such that
    \begin{align*}
        \lim_{T \to \infty} \frac{R_T}{\ln(T)} = \Omega\left(\frac{K^{1-\alpha}U}{\Delta_{\min}}\right).
    \end{align*}
    (Minimax Lower Bound) For any algorithm, for any $K \geq 4, \alpha \in (0,1)$ and $U \leq K^\alpha$, there exists an adversarial bandit instance with $K$ arms and loss distribution satisfying~\eqref{eq:softconstraint} such that
    \begin{align*}
        R_T = \Omega(\sqrt{K^{1-\alpha}UT}).
    \end{align*}
    \label{thm:LowerBounds}
\end{theorem}

\subsection{Implications for Bandits with Adversarially Changing Action Sets}
Intuitively, the sparsity constraint $\norm{\ell_t}_0 \leq S$ indicates that there are at most $S$ arms containing non-trivial information in each round, however the learner 
{does not know the} arms with non-trivial information.
In this sense, sparse bandits is conceptually more difficult than adversarial sleeping bandits~\citep{Kleinberg2010Sleeping}, where in each round $t$ the learner is given, by an adversary, a set $\sA_t \subseteq [K]$ of active arms to choose from. 
Note that the learner is not allowed to choose an arm in $[K] \setminus \sA_t$.
The performance of the learner is measured by its per-action regret
\begin{align*}
    R_{T,a} = \sum_{t=1}^T \I{a \in \sA_t}(\ell_{t,I_t} - \ell_{t,a}).
\end{align*}
A natural question is whether Algorithm~\ref{algo:optimalBOBWSparsity} can be extended to this adversarial sleeping bandits setting.
The following theorem answers this question in the positive.
\begin{theorem}
    For any $K \geq 4, T \geq 4k$, Algorithm~\ref{algo:SPMSleepingBandit} (in Appendix~\ref{appendix:SPMSleepingBandits}) guarantees that for all $a \in [K]$,
    \begin{align*}
      \E[R_{T,a}] \leq O\left(\sqrt{\frac{(K^{1-\alpha} - 1)(\max_{t \in [T]}\abs{\sA_t})^\alpha}{\alpha(1-\alpha)}T}\right),
  \end{align*}
  \label{thm:SPMSleepingBanditBound}
\end{theorem}
Our Algorithm~\ref{algo:SPMSleepingBandit} is a combination of Algorithm~\ref{algo:optimalBOBWSparsity} and the SB-EXP3 algorithm in~\cite{NguyenAndMehta2024SBEXP3}. More specifically, Algorithm~\ref{algo:SPMSleepingBandit} uses the estimated cumulative \emph{regret} (instead of losses) to compute $q_t$ in the FTRL update. Then, the sampling probability vector $p_t$ is obtained by a filtering step $p_{t,i} = \frac{q_{t,i}\I{i \in \sA_t}}{\sum_{j=1}^K q_{t,j}\I{j \in \sA_t}}$.
While the bound in Theorem~\ref{thm:SPMSleepingBanditBound} is of the same order as in~\citet[][Theorem 2]{NguyenAndMehta2024SBEXP3}, it has the advantage of not requiring the knowledge of $\max_{t}\abs{\sA_t}$ in advance nor any complicated two-level doubling trick.

\section{Application II: $\sqrt{Q\ln(K)}$ Upper Bound with Unknown $Q$ using Optimistic FTRL}
\label{sec:SPMwOFTRLwReservoirSampling}
\looseness=-1 In this section, we propose a new approach for obtaining a BOBW $O(\frac{K\ln{T}}{\Delta_{\min}}, \sqrt{Q\ln{K}})$-bound with unknown $Q$. 
For ease of exposition, we assume losses are in $[0,1]$ and note that the analysis can be easily extended for losses in $[-1,1]$.
The new approach is based on applying real-time SPM on the Optimistic FTRL framework~\citep{Rakhlin2013OptimisticFTRL}, and then combining with the Reservoir Sampling algorithm~\citep{HazanAndKale11a}.

 In principle, our algorithm follows the same framework as~\citet{HazanAndKale11a,Bubeck2018ALT} where the learner maintains a reservoir $\sS_i$ of observed losses for each arm $i \in [K]$ and then uses the estimated mean $m_{t,i} = \tilde{\mu}_{t,i}$ of these reservoirs as the optimistic vector $m_t$ in Optimistic FTRL. 
In each round $t$, the learner chooses to perform either a reservoir sampling step for updating the reservoir $\sS_i$, or a FTRL learning step for minimizing the regret. 
When the FTRL learning step is performed in round $t$, the vector $q_t$ is computed by
\begin{align}
    q_t = \argmin_{x \in \Delta_K} \inp{m_t + L_{t-1}}{x} + \beta_t\left(\frac{1}{\alpha}(1-\sum_{i=1}^K x_i^\alpha)\right) - \gamma\sum_{i=1}^K \ln(x_i).
    \label{eq:qtInOFTRL}
\end{align}
Similar to Algorithm~\ref{algo:optimalBOBWSparsity}, the sampling probability vector $p_t$ is obtained by mixing with $\frac{1}{T}$, i.e, $p_t = \left(1 - \frac{K}{T}\right)q_t + \frac{1}{T}\1$. 
After an arm $I_t \sim p_t$ is drawn, the loss estimates are $\hat{\ell}_{t, i} = m_{t,i} + \frac{(\ell_{t,i} - m_{t,i})\I{I_t = i}}{p_{t,i}}$.
The learning rates $(\beta_t)_t$ are computed by real-time SPM, with $z_t$ and $h_t$ defined as
\begin{align*}
    z_t =  \min\left( \frac{(6d)^{2-\alpha}}{2(1-\alpha)}\tilde{p}_{t,I_t}^{2-\alpha}(\hat{\ell}_{t,I_t} - m_{t,I_t})^2, \frac{\beta_t18d^2}{\gamma}(\ell_{t,I_t} - m_{t,I_t})^2  \right), \quad h_t =\frac{1}{\alpha}\left(\sum_{i=1}^K p_{t,i}^\alpha - 1\right).
\end{align*}
The full procedure is given in~Algorithm~\ref{algo:OFTRLReservoirSampling} in Appendix~\ref{appendix:proofsforSectionOFTRLReservoirSampling}.
The following theorem states the BOBW bound for this approach.
\begin{theorem}
    Algorithm~\ref{algo:OFTRLReservoirSampling} (in Appendix~\ref{appendix:proofsforSectionOFTRLReservoirSampling}) guarantees the following bounds simultaneously
    \begin{itemize}
        \item In the adversarial regime:
        \begin{align}
            R_T \leq O\left( \sqrtfrac{(K^{1-\alpha} - 1)Q}{\alpha (1-\alpha)}\right).
            \label{eq:BOBWsqrtQLnKBoundAdv}
        \end{align}
        \item In the adversarial regime with a self-bounding constraint:
        \begin{align}
            R_T \leq O\left(\frac{(K-1)^{1-\alpha}K^\alpha \ln(T)}{\alpha(1-\alpha)\Delta_{\min}} + \sqrt{C\frac{(K-1)^{1-\alpha}K^\alpha \ln(T)}{\alpha(1-\alpha)\Delta_{\min}}} + \sqrt{\frac{(K-1)^{1-\alpha}K^{\alpha}}{\alpha(1-\alpha)}}\right)
            \label{eq:BOBWsqrtQLnKBoundSto}
        \end{align}
    \end{itemize}
    \label{thm:BOBWsqrtQLnKBound}
\end{theorem}
\begin{proof}(Sketch)
    Our analysis follows from the analysis of Algorithm~\ref{algo:optimalBOBWSparsity} and the observation by~\citet{HazanAndKale11a} that the reservoir sampling steps only add an $O(\ln(T)^2)$ amount to the regret bound.
    As a result, the bound~\eqref{eq:BOBWsqrtQLnKBoundSto} for adversarial regime with a self-bounding constraint follows almost identically to that of Algorithm~\ref{algo:optimalBOBWSparsity}. 
    For the bound~\eqref{eq:BOBWsqrtQLnKBoundAdv} in the adversarial regime, the total variation $Q$ naturally comes out of $\sum_{t=1}^T z_t$  as follows:
    \begin{align*}
        \E\left[\sum_{t=1}^T z_t\right] &\lesssim \frac{1}{1-\alpha}\E\left[\sum_{t=1}^T \sum_{i=1}^K (\ell_{t,i} - m_{t,i})^2\right] =  \frac{1}{1-\alpha} \E\left[\sum_{t=1}^T \norm{\ell_{t} - \tilde{\mu}_{t}}_2^2 \right] \qquad(\text{since } m_t = \tilde{\mu}_{t})\\
        &\leq \frac{1}{1-\alpha}\left(\E\left[\sum_{t=1}^T \norm{\ell_{t} - \mu_{T}}_2^2 \right] + \E\left[\sum_{t=1}^T \norm{\tilde{\mu}_t - \mu_{t}}_2^2 \right]\right) \\
        &\leq \frac{1}{1-\alpha}\left(Q + \sum_{t=1}^T \frac{Q}{t\ln(T)}\right) \leq O\left(\frac{Q}{1-\alpha}\right),
    \end{align*}
     where the second inequality follows from triangle inequality and 
     Lemma 10 in~\cite{HazanAndKale11a}, the third inequality is by Lemma 11 in~\cite{HazanAndKale11a}, and the last inequality is due to $\sum_{t=1}^T \frac{1}{t\ln{T}} \leq O(1)$. Together with~\eqref{eq:boundG} and $h_{\max} \leq \frac{K^{1-\alpha}-1}{\alpha}$, this implies~\eqref{eq:BOBWsqrtQLnKBoundAdv}.
\end{proof}
\begin{remark}
    While existing works~\citep{HazanAndKale11a, Bubeck2018ALT} require either the knowledge of $Q$ or sophisticated doubling tricks to estimate $Q$, our Algorithm~\ref{algo:OFTRLReservoirSampling} does not require such knowledge or any tricks.
    When $\alpha \to 1$, the bound in~\eqref{eq:BOBWsqrtQLnKBoundAdv} becomes $O(\sqrt{Q\ln(K)})$. This bound matches the best known upper bound in~\citet{Bubeck2018ALT} and never exceeds $O(\sqrt{TK\ln(K)})$ in the worst case, all while simultaneously having a $T$-optimal best-of-both-worlds guarantee.
\end{remark}

\section{Coordinate-Wise Stability-Penalty Matching}
\label{sec:CoordinateWiseSPM}
\begin{algorithm}[t]
	\KwIn{$K \geq 3, T \geq 4K, \alpha \in (0,1), \beta_1 = \frac{8K}{1-\alpha}\1, \gamma = \max(6,48\sqrtfrac{\alpha}{1-\alpha}), d = 2$.}
    Initialize $L_{0,i} = 0$ for $i \in [K]$ \;

	\For{each round $t = 1, \dots, T$}{        
        Compute $m_t \in [0,1]^K$ where $m_{t,i} = \frac{1}{1 + \sum_{s=1}^{t-1}\I{I_s = i}}\left(\frac{1}{2} + \sum_{s=1}^{t-1} \I{I_s = i} \ell_{t,i}\right)$\;

        Compute $q_t$ by Equation~\eqref{eq:qtInOFTRL}\;

        Compute $p_t = (1-\frac{K}{T})q_t + \frac{1}{T}\1$\;
		
        Draw $I_t \sim p_t$ and observe $\ell_{t, I_t}$\;

		Compute loss estimate $\hat{\ell}_{t, i} = m_{t,i} + \frac{(\ell_{t,i} - m_{t,i})\I{I_t = i}}{p_{t,i}}$ and update $L_{t,i} = L_{t-1,i} + \hat{\ell}_{t,i}$\;
        
        Compute $z_{t,i} =  \I{i = I_t}(\ell_{t,I_t} - m_{t,I_t})^2\min\left( \frac{(6d)^{2-\alpha}}{2(1-\alpha)}\min\left\{p_{t,I_t}^{-\alpha}, \frac{1 - p_{t,I_t}}{p_{t,I_t}^2}\right\}, \frac{\beta_{t,i}18d^2}{\gamma}  \right)$ \;
        
        Compute $h_{t,i} = \frac{1}{\alpha}p_{t,i}^\alpha$\;
        
        Compute $\beta_{t+1, i} = \beta_{t,i} + \frac{z_{t,i}}{\beta_{t,i}h_{t,i}}$\;
	}
	\caption{Coordinate-wise SPM with hybrid regularization for losses in $[0,1]$.}
	\label{algo:CoordinateWiseSPM}
\end{algorithm}
We further generalize the SPM framework by introducing a new technique called coordinate-wise SPM (\texttt{CoWSPM}).
As the name suggests,~\texttt{CoWSPM} maintains separate learning rate $\beta_{t,i}$, stability term $z_{t,i}$ and penalty term $h_{t,i}$ for each arm $i \in [K]$.
In each round $t$,~\texttt{CoWSPM} updates the learning rates for each arm using the SPM update formula~\eqref{eq:SPM}, i.e.,
\begin{align}
    \beta_{t+1,i} = \beta_{t,i} + \frac{z_{t,i}}{\beta_{t,i}h_{t,i}}.
    \label{eq:CWSPMupdate}
\end{align}
Obviously, if $(z_{t,i})_{i \in [K]}$ and $(h_{t,i})_{i \in [K]}$ take the same values across all arms, this this approach recovers Algorithm~\ref{algo:optimalBOBWSparsity}.
{Instead,} we adopt a different approach where $z_{t,i} = 0$ for all $i \neq I_t$ so that only the learning rate $\beta_{t,I_t}$ of the observed arm $I_t$ is updated in round $t$. 
The full procedure of~\texttt{CoWSPM} is given in Algorithm~\ref{algo:CoordinateWiseSPM}, which uses the Optimistic FTRL framework with
\begin{align}
    \phi_t(x) = \sum_{i=1}^K \beta_{t,i}\left(\frac{-x_i^\alpha}{\alpha} + (1-x_{i})\ln(1-x_{i}) + x_{i}\right) - \gamma\sum_{i=1}^K \ln(x_i).
    \label{eq:CWSPMregularizer}
\end{align}
This regularization function contains not only the $\alpha$-Tsallis entropy, but also a part of the Shannon entropy and a linear term. 
The addition of these terms into the regularizer has been done in~\citet{ItoCOLT2022aVariance} in order to have a regret bound containing the quantity $\tilde{p}_{t,i} = \min(p_{t,i}, 1-p_{t,i})$ for the stochastic setting.
This technique has a similar impact in our work, where it allows us to bound $z_{t,i}$ by a quantity containing $\tilde{p}_{t,i}$.
However, while we use the same technique to introduce $\tilde{p}_{t,i}$ into our bounds, our analysis develops fundamentally different technical lemmas from that of~\cite{ItoCOLT2022aVariance} in order to use this new regularizer in the real-time SPM framework. 
%
Next, to ensure that only $\beta_{t,I_t}$ is updated, we set $z_{t,i}$ by
\begin{align}
    z_{t,i} =  \I{i = I_t}(\ell_{t,I_t} - m_{t,I_t})^2\min\left( \frac{(6d)^{2-\alpha}}{2(1-\alpha)}\min\left\{p_{t,I_t}^{-\alpha}, \frac{1 - p_{t,I_t}}{p_{t,I_t}^2}\right\}, \frac{\beta_{t,i}18d^2}{\gamma}  \right),
    \label{eq:CoordinateWiseSPMzti}
\end{align}
so that $z_{t,I_t} \geq 0$ and $z_{t,i} = 0$ for $i \neq I_t$. 
The following theorem states the BOBW data-dependent bound of Algorithm~\ref{algo:CoordinateWiseSPM}, whose full proof is given in Appendix~\ref{appendix:ProofsForCoordinateWiseSPM}. 
The proof sketch outlines the main technical challenges in the analysis of Algorithm~\ref{algo:CoordinateWiseSPM}.
\begin{theorem}
    For any $K \geq 4, T \geq 4k$, \texttt{CoWSPM} (Algorithm~\ref{algo:CoordinateWiseSPM}) with $\alpha \in (0,1)$ guarantees the following bounds simultaneously
    \begin{itemize}
        \item In the adversarial regime: 
        \begin{align*}
            R_T \lesssim\min\left\{ \sqrt{K\ln(T)\min(Q_\infty, L^*, T-L^*)}, K^{\frac{\alpha}{2}}\sqrt{KT}\right\}.
        \end{align*}
        \item In the stochastic regime:
        \begin{align*}
            R_T \lesssim \frac{1}{\alpha (1-\alpha)}\sum_{i \neq i^*} \frac{\ln(T)}{\Delta_i}.
        \end{align*}
    \end{itemize}
    \label{thm:CWSPMbounds}
\end{theorem}


\begin{proof}(Sketch)
    Intuitively, coordinate-wise SPM consists of $K$ separate real-time SPM processes, one for each arm. 
Similar to~\citet{ItoCOLT2022aVariance}, we find that this more refined approach enables deriving a bound (for the adversarial regime) that is adaptive to simultaneously different data-dependent quantities such as $Q_\infty$ and $L^*$.
However, having separate learning rates introduces several new technical challenges. First, the analysis developed for Algorithm~\ref{algo:optimalBOBWSparsity} that bounds $q_{t+1, i} = O(q_{t, i})$ for all $i \in [K]$ no longer applies because in each round $t$, the learning rates $(\beta_{t,i})_{i \in [K]}$ can be arbitrarily different from each other. 
The \texttt{CoWSPM} algorithm resolves this by using $\beta_{t+1, i} = \beta_{t,i}$ for $i \neq I_T$ so that it only need $q_{t+1, i} = O(q_{t,i})$ to hold for $i = I_t$ since
\begin{align*}
    \phi_t(q_{t+1}) - \phi_{t+1}(q_{t+1}) &= \sum_{i=1}^K (\beta_{t+1,i} - \beta_{t,i})(-f(q_{t+1, i})) = (\beta_{t+1, I_t} - \beta_{t,I_t})f(q_{t+1, I_t}).
\end{align*}

The second and also more important challenge is that even if $q_{t+1, i} = O(q_{t, i})$, the naive decomposition of the $\alpha$-Tsallis entropy into its coordinate-wise form $-\psi_{TE}(x) = \frac{1}{\alpha}\sum_{i=1}^K (x_i^\alpha - x_i)$ and then assigning $h_{t,i} = \frac{1}{\alpha} (x_i^\alpha - x_i)$ does \emph{not} guarantee that $h_{t+1, i} = O(h_{t,i})$. This is because the function $x \mapsto x^\alpha - x$ gets arbitrarily close to $0$ when $x$ gets close to $1$.
This prompts a different choice for $h_{t,i}$ rather than $-f(p_{t,i})$. 
Algorithm~\ref{algo:CoordinateWiseSPM} uses $h_{t,i} = \frac{1}{\alpha}p_{t,i}^\alpha$, which is monotonically increasing and ensures that $h_{t+1,I_t} = O(h_{t,I_t})$ for $q_{t+1,i} = O(q_{t,i})$. 
This choice of $h_{t,i}$ is justified by the technical Lemma~\ref{lemma:xalphadominatesxminus1lnxminus1}, which states that $(x-1)\ln(1-x) \leq x^\alpha$ for any $x, \alpha \in [0,1]$.

Finally, we prove that with $z_{t,i}$ defined in~\eqref{eq:CoordinateWiseSPMzti}, the product $\E[h_{t,i}z_{t,i}]$ is upper bounded by a quantity containing $\tilde{p}_{t,i}$ and thus an $O(\sum_{i \neq i^*}\frac{\ln{T}}{\Delta_i})$ regret bound holds for stochastic bandits.
This is handled by Lemma~\ref{lemma:CoordinateWiseSPMztiIsGoodForStochasticBandits}, which shows that $\E_{I_t}\left[z_{t,i}\right] \leq 2\min(p_{t,i}, 1-p_{t,i})$.
\end{proof}

\begin{remark}
    Theorem~\ref{thm:CWSPMbounds} holds for all $\alpha \in (0,1)$. In particular, for $\alpha \neq \frac{1}{2}$,  we do not require any additional assumptions such as the $\Delta_i$ being known in order to get the $T$-optimal BOBW bound.
    This is a major difference compared to the Tsallis-INF algorithm~\citep{Zimmert2021TsallisINF}.
    On the other hand, the adversarial bound in Theorem~\ref{thm:CWSPMbounds} has an extra factor $\sqrt{K^\alpha}$. It is unclear to us whether this extra factor is a fundamental limitation of~\texttt{CoWSPM} or an artifact of our analysis.
\end{remark}

\section{Conclusion and Future Works}
\label{sec:conclusion}
\looseness=-1 We developed real-time SPM, an extension of the SPM method originally developed for obtaining best-of-both-worlds bounds in bandits problems.
We showed that real-time SPM algorithms achieve novel bounds that are simultaneously best-of-both-worlds, data-dependent and have optimal dependency on $T$ in both stochastic and adversarial regimes. Our bounds also have optimal dependency on the data-dependent quantities such as sparsity or total variation of the loss sequence without knowing them nor using sophisticated estimation tricks. Future work includes applying real-time SPM on other bandits problems, such as contextual linear bandits, and making real-time SPM adaptive towards other challenging data-dependent quantities like $\ell_1$ and $\ell_2$-norm path-length bounds.


\bibliography{colt2025-SPMDataDependentBoundsMAB}

\newpage
\appendix


\section{Related Works}
\label{sec:RelatedWorks}
Due to the vast literature on BOBW and data-dependent bounds in various bandits learning settings, this sections presents only the most relevant works in multi-armed bandits.
A more comprehensive list of related works can be found in~\citet{ItoCOLT2024,Tsuchiya2023stabilitypenaltyadaptive} and references therein.

\textbf{Best-of-both-worlds bounds.} The BOBW bounds in our paper are derived using the SPM method for tuning learning rates in the FTRL framework, originally proposed in~\citet{ItoCOLT2024}. For stochastic bandits, our $O(\frac{K\ln{T}}{\Delta_{\min}})$bound  in Sections~\ref{sec:BOBWboundsSparseLosses} and~\ref{sec:SPMwOFTRLwReservoirSampling} matches that of~\citet{WeiAndLuo2018aBroadOMD,ItoCOLT2024}, and our $O(\sum_{i \neq i^*}\frac{\ln{T}}{\Delta_i})$ bound in Section~\ref{sec:CoordinateWiseSPM} matches that of~\citet{Zimmert2021TsallisINF, Ito2021HybridDataMABBound}. 
Both of these bounds are looser than the $O(\sum_{i \neq i^*}\frac{\sigma_i^2 \ln{T}}{\Delta_i})$ in~\citet{ItoCOLT2022aVariance} obtained by a more specialized approach, where $\sigma_i^2$ is the variance of the losses of a sub-optimal arm $i$.
However, except for~\citet{ItoCOLT2024}, these existing works have an $O(\sqrt{T\ln{T}})$ worst-case bound for adversarial bandits, which contains an extra $\ln{T}$ factor compared to our work. 
Our BOBW bound also have data-dependent guarantees, which is an advantage over~\citet{ItoCOLT2024}. 
For bandits with sparse losses,~\citet{Tsuchiya2023stabilitypenaltyadaptive} similarly obtained bounds that are both BOBW and dependent on the sparsity constraint; however their bounds contain extra factors of $\ln(KT)$ in stochastic bandits and $\sqrt{\ln{T}}$ in adversarial bandits compared to our results.

\textbf{Data-dependent bounds.} We study the following data-dependent quantities: sparsity of losses, total variations and small losses. 
For bandits with sparse negative losses where $\norm{\ell_t} \leq S$ and $S$ is unknown, our $O(\frac{S\ln{T}}{\Delta_{\min}}, \sqrt{ST\ln(K)})$ BOBW bound is the first $S$-agnostic and $T$-optimal BOBW bound for this setting, which improves upon on the bound of~\citet{Tsuchiya2023stabilitypenaltyadaptive} and matches the best known bound for adversarial bandits in~\citet{Bubeck2018ALT}.
When the total variations $Q, Q_\infty$ and/or the loss of the best arm $L^*$ (defined in Section~\ref{sec:ProblemSetup}) are small, our algorithms are based on the optimistic FTRL (OFTRL) framework similar to~\citet{HazanAndKale11a,Bubeck2018ALT,ItoCOLT2022aVariance}. 
The dependency on $Q, Q_\infty$ and $L^*$ in our results match the best known bounds in these works, while our BOBW bounds have an optimal $O(\square\ln{T}, \square\sqrt{T})$ dependency on $T$. 
Particularly, our coordinate-wise real-time SPM algorithm in Section~\ref{sec:CoordinateWiseSPM} can be seen as a $T$-optimal variant of the algorithm in~\citet{ItoCOLT2022aVariance}, which share the idea of using separate learning rates for each arm.


\section{Proofs for Section~\ref{sec:BOBWboundsSparseLosses}}
\subsection{Proof for Theorem~\ref{thm:BOBWbounds}}
Let $D_{TE}(p,q)$ and $D_{LB}(p,q)$ denote the Bregman divergences induced by the $\alpha$-Tsallis entropy and the log-barrier function, respectively. Let $D_t(p,q) = \beta_t D_{TE}(p,q) + \gamma D_{LB}(p, q)$ denote the Bregman divergence induced by the hybrid regularizer $\phi_t(p) = \beta_t\left(\frac{1}{\alpha}(1-\sum_{i=1}^K p_i^\alpha)\right) - \gamma\sum_{i=1}^K \ln(p_i)$.
Let $\hat{\ell}_t = \begin{bmatrix}
    \hat{\ell}_{t,1} \\
    \hat{\ell}_{t,2} \\
    \cdots \\
    \hat{\ell}_{t,K} \\
\end{bmatrix}$ be the estimated loss vector at time $t$. We state the following three stability lemmas, whose proofs are in Section~\ref{sec:stabilityproofs} and Section~\ref{sec:technicallemmas}.
\begin{lemma}
    For any $t \in [T]$, Algorithm~\ref{algo:optimalBOBWSparsity} guarantees
    \begin{align*}
        h_t \geq \frac{1-\alpha}{4\alpha}T^{-\alpha} \quad{\text{and}}\quad E_{I_t}[z_t] \leq \frac{(6d)^{2-\alpha}}{2(1-\alpha)} S^{\alpha}.
    \end{align*}
    \label{lemma:minhtAndmaxEzt}
\end{lemma}
\begin{lemma}
    For any $t \in [T]$, Algorithm~\ref{algo:optimalBOBWSparsity} guarantees
    \begin{align}
        \inp{\hat{\ell}_t}{q_{t} - q_{t+1}} - D_t(q_{t+1}, q_t) &\leq \min\left( \frac{(6d)^{2-\alpha}}{2\beta_t(1-\alpha)}\min(p_{t,I_t}, 1 - p_{t,I_t})^{2-\alpha}\hat{\ell}^2_{t,I_t}, \frac{18d^2}{\gamma} \ell_{t,I_t}^2  \right).
    \end{align}
    \label{lemma:stableTELB}
\end{lemma}
Note that in Lemma~\ref{lemma:stableTELB}, the right-hand side is exactly $\frac{z_t}{\beta_t}$.
\begin{lemma}
    For any $t \in [T]$, Algorithm~\ref{algo:optimalBOBWSparsity} guarantees that for all $i \in [K]$,
    \begin{align}
        q_{t+1, i} \leq 3dq_{t,i} \leq 6dp_{t,i}.
    \end{align}
    Moreover, this implies that $(-\psi_{TE}(q_{t+1})) \leq 3d(-\psi_{TE}(q_t)) \leq 6d(-\psi_{TE}(p_t))$.
    \label{lemma:boundhtplus1byht}
\end{lemma}
\begin{proof}(Of Theorem~\ref{thm:BOBWbounds})
    Next, let 
\begin{align}
    \Phi_t(p) = \beta_t \psi_{TE}(p) + \gamma \psi_{LB}(p)
\end{align}
be the time-varying regularizer in Algorithm~\ref{algo:optimalBOBWSparsity}. For any $a \in [K]$, define 
\begin{align*}
    u_a = \left(1 - \frac{K}{T}\right)e_a + \frac{1}{T}\1.
\end{align*}
The pseudo-regret with respect to arm $a$ is 
\begin{align*}
    R_{T, a} &= \E[\sum_{t=1}^T \inp{\ell_t}{p_t - e_a}] \\
    &= \E[\sum_{t=1}^T \inp{\ell_t}{q_t - u_a}] + \E[\sum_{t=1}^T \inp{\ell_t}{p_t - q_t}] + \E[\sum_{t=1}^T \inp{\ell_t}{u_a - e_a}] \\
    &\leq \E[\sum_{t=1}^T \inp{\ell_t}{q_t - u_a}]  + 4K \\
    &= \E[\sum_{t=1}^T \inp{\hat{\ell}_t}{q_t - u_a}]  + 4K,
\end{align*}
where the inequality is from $\inp{\ell_t}{p_t - q_t} = \frac{1}{T}\sum_{i=1}^K \ell_{t,i}(1-Kq_{t,i}) \leq \frac{2K}{T}$ and 
{$\inp{\ell_t}{u_a - e_a} \leq \frac{2K}{T}$},
and the last equality is from $\E[\hat{\ell}_t] = \ell_t$.
By the standard analysis of FTRL with time-varying regularizer~\citep{BanditAlgorithmsBook2020}, we have
\begin{align*}
    \sum_{t=1}^T \inp{\hat{\ell}_t}{q_t - u_a} &\leq \Phi_{T+1}(u_a) - \min_{p \in \Delta_K} \Phi_1(p) + \sum_{t=1}^T \Phi_t(q_{t+1}) - \Phi_{t+1}(q_{t+1}) \\
    &\quad + \sum_{t=1}^T \inp{\hat{\ell}_t}{q_t - q_{t+1}} - D_t(q_{t+1}, q_t) \\
    &= \Phi_{T+1}(u_a) - \min_{p \in \Delta_K} \Phi_1(p) + \sum_{t=1}^T (\beta_{t+1} - \beta_t)(-\psi_{TE}(q_{t+1})) \\
    &\quad + \sum_{t=1}^T \inp{\hat{\ell}_t}{q_t - q_{t+1}} - D_t(q_{t+1}, q_t) \\
    &\leq \Phi_{T+1}(u_a) - \min_{p \in \Delta_K} \Phi_1(p) + 6d\left(\sum_{t=1}^T (\beta_{t+1} - \beta_t)h_t + \sum_{t=1}^T \frac{z_t}{\beta_t}\right) \\
    &= \Phi_{T+1}(u_a) - \min_{p \in \Delta_K} \Phi_1(p)  + 24\sum_{t=1}^T \frac{z_t}{\beta_t} \\
    &\leq \gamma \psi_{LB}(u_a) - \beta_1\min_{p \in \Delta_K} \psi_{TE}(p) + 24\sum_{t=1}^T \frac{z_t}{\beta_t} \\
    &\leq \gamma K\ln(T) + \frac{\beta_1}{\alpha}(K^{1-\alpha}-1) + 24\sum_{t=1}^T \frac{z_t}{\beta_t},
\end{align*}
where the second inequality is from Lemma~\ref{lemma:stableTELB} and Lemma~\ref{lemma:boundhtplus1byht}, the second equality is from the update rule $(\beta_{t+1} - \beta_t)h_t = \frac{z_t}{\beta_t}$, the third inequality is due to 
{$\psi_{TE}(p) \leq 0$} and $\psi_{LB}(p) > 0$ for all $p \in \Delta_K$, and the last inequality is due to $(u_a)_{i} \geq \frac{1}{T}$.

\subsubsection*{Bounding $\E[\sum_{t=1}^T \frac{z_t}{\beta_t}]$}
Note that we should not directly apply the SPM bound based on $z_{\max} = \max_{t \in [T]}z_t$ in Lemma 3 of~\cite{ItoCOLT2024}, because $z_{\max}$ is of order 
{$\max_t p_{t,I_t}^{-\alpha}$}, which can be very large.
Instead, let
\begin{align*}
    G &= \sum_{t=1}^T \frac{z_t}{\sqrt{\sum_{s=1}^t \frac{z_s}{h_s}}} \\
    h_{\max} &= \max_{t \in [T]}h_t, \\
    z_{\E, \max} &= \max_{t \in [T]}\E_{I_t}[z_t].
\end{align*}
By definitions of $h_t$ and $z_t$, we have $h_t \leq \frac{K^{1-\alpha} - 1}{\alpha}$ and
\begin{align*}
    \E_{I_t}[z_t] &\leq \frac{(6d)^{2-\alpha}}{2(1-\alpha)} S^{\alpha},
\end{align*}
from Lemma~\ref{lemma:minhtAndmaxEzt}. It follows that 
\begin{align*}
    h_{\max} &\leq \frac{K^{1-\alpha} - 1}{\alpha}, \\
    z_{\E, \max} &\leq \frac{(6d)^{2-\alpha}}{2(1-\alpha)} S^{\alpha}.
\end{align*}
Next, let
\begin{align}
\beta'_t &= \sqrt{\beta_{1}^2 + 2\sum_{s=1}^{t-1}\frac{z_{s}}{h_s}}.
\end{align}
Let $E = \{t \in [T]: \beta'_{t+1} \geq \sqrt{2}\beta'_t\}$ and $E^c = [T] \setminus E$. Also, let $N = \abs{E}$ and $j = 1, 2, \dots, N$ be the index running over the rounds in $E$. 
Similar to the proof of Lemma 2 in~\citet{ItoCOLT2024}, squaring both sides of $\beta_t = \beta_{t-1} + \frac{z_{t-1}}{\beta_{t-1}h_{t-1}}$ implies that
\begin{align*}
    \beta_t^2 = \beta_{t-1}^2 + \frac{2z_{t-1}}{h_{t-1}} + \frac{z_{t-1}^2}{\beta_{t-1}^2 h_{t-1}^2} \geq \beta_{t-1}^2 + \frac{2z_{t-1}}{h_{t-1}} \geq \beta_1^2 + 2\sum_{s=1}^{t-1} \frac{z_s}{h_s},
\end{align*}
 which shows that 
 {$\beta'_t \leq \beta_t$}. Furthermore,
 {$\sum_{t \in E^c} \frac{z_t}{\beta_t} \leq G$} from the proof of Lemma 2 in~\citet{ItoCOLT2024}. Therefore, 
\begin{align*}
    \sum_{t=1}^T \frac{z_t}{\beta_t} &= \sum_{t \in E^c} \frac{z_t}{\beta_t} +  \sum_{t \in E} \frac{z_t}{\beta_t} \\
    &\leq G + \sum_{t \in E} \frac{z_t}{\beta_t} \\
    &\leq G + \frac{18d^2}{\gamma}N \\
    &\leq G + \frac{18d^2}{\gamma}\log_{\sqrt{2}}\left(\frac{\beta'_{T+1}}{\beta'_1}\right) \\
    &\leq G + \frac{26d^2}{\gamma}\ln\left(1 + 2\sum_{t=1}^T\frac{z_t}{h_t}\right),
\end{align*}
where the second inequality is from the definition of $z_t$ and the third inequality is from the fact that $\beta'_t$ is multiplied by at least $\sqrt{2}$ after every round in $E$, thus there can be at most $N \leq \log_{\sqrt{2}}\frac{\beta'_{T+1}}{\beta'_1}$ such multiplications.

Taking the expectation over $I_{1:T}$ on both sides and using $E[\ln(X)] \leq \ln(E[X])$, we obtain 
\begin{align*}
    \E\left[\sum_{t=1}^T \frac{z_t}{\beta_t}\right] &\leq \E[G] + \frac{26d^2}{\gamma}\ln\left(1 + 2\E\left[\sum_{t=1}^T \frac{z_t}{h_t}\right]\right) \\
    &= \E[G] + \frac{26d^2}{\gamma}\ln\left(\E\left[1 + 2\sum_{t=1}^T \frac{\E_{I_t}[z_{t}]}{h_t}\right]\right) \\
    &\leq \E[G] + \frac{26d^2}{\gamma}\ln\left(1 +\frac{(6d)^{2-\alpha}S^{\alpha}}{(1-\alpha)}\E\left[\frac{4T}{\min_{t \in [T]}h_t}\right]\right) \\
    &\leq \E[G] + \frac{26d^2}{\gamma}\ln\left(1 + \frac{(6d)^{2-\alpha}S^{\alpha}}{(1-\alpha)}\frac{4\alpha T^{\alpha + 1}}{1-\alpha}\right) \\
    &= \E[G] + O\left(\frac{1}{\gamma}\ln\left(\frac{\alpha S^{\alpha} T}{(1-\alpha)^2}\right)\right)
\end{align*}
where the first equality is because $h_t$ is $\sF_{t-1}$-measurable and the last inequality is due to Lemma~\ref{lemma:minhtAndmaxEzt}.

Next, Equation 45 in~\cite{ItoCOLT2024} shows that $G \leq 2\sqrt{h_{\max}\sum_{t=1}^Tz_t}$. Moreover, Equation 46 in~\cite{ItoCOLT2024} shows that for any fixed $J \geq 1$,
\begin{align*}
    G \leq \sqrt{8J\sum_{t=1}^T h_tz_t} + 2\sqrt{2^{-J}h_{\max}\sum_{t=1}^T z_t}.
\end{align*}
As a result, we obtain the following bound:
\begin{nalign}
    \E\left[\sum_{t=1}^T \frac{z_t}{\beta_t}\right] &\leq \min\left\{ \inf_{J \in \sN}\E\left[\left\{ \sqrt{8J\sum_{t=1}^T h_t z_t} + 2\sqrt{2^{-J}h_{\max}\sum_{t=1}^Tz_t} \right\}\right], 2\E\left[\sqrt{ h_{\max}\sum_{t=1}^T z_t }\right] \right\} \\
    &+ O\left(\frac{1}{\gamma}\ln\left(\frac{\alpha S^{\alpha} T}{(1-\alpha)^2}\right)\right).
    \label{eq:boundztoverbetatbyItoCOLT2024}
\end{nalign}
\subsubsection*{Adversarial Regime:} 
Using Jensen's inequality $E[\sqrt{X}] \leq \sqrt{E[X]}$ and Equation~\eqref{eq:boundztoverbetatbyItoCOLT2024}, we obtain
\begin{align*}
    \E\left[\sum_{t=1}^T \frac{z_t}{\beta_t}\right] &\leq 2\sqrt{\frac{((K-1)^{1-\alpha} - 1)}{\alpha}\sum_{t=1}^T\E[z_t]} + O\left(\frac{1}{\gamma}\ln\left(\frac{\alpha S^{\alpha} T}{(1-\alpha)^2}\right)\right) \\
    &\leq 2\sqrt{\frac{((K-1)^{1-\alpha} - 1)}{\alpha}T\E[z_{\E, \max}]} + O\left(\frac{1}{\gamma}\ln\left(\frac{\alpha S^{\alpha} T}{(1-\alpha)^2}\right)\right)\\
    &= O\left( \sqrt{\frac{(K^{1-\alpha} - 1) S^\alpha T}{\alpha(1-\alpha)}} \right).
\end{align*}

\subsubsection*{Adversarial Regime with a Self-Bounding Constraint:}
Let $R_T = \max_{a \in [K]}R_{T,a}$. In this regime, we have 
{$R_T + C \geq \E[\sum_{t=1}^T\sum_{i=1}^Kq_{t,i}\Delta_i]$}.
Let 
{$i^* \in [K]$} be the unique optimal arm.

Observe that given the sequence of randomly drawn arms until the beginning of round $t$, the quantity $h_t$ is fixed. Therefore, we can write $\E[h_tz_t] = \E[h_t\E_{I_t}[z_t]]$ and obtain
\begin{nalign}
    \E[h_tz_t] &= \E[h_t\E_{I_t}[z_t]] \\
    &\leq \frac{(6d)^{2-\alpha}}{2(1-\alpha)}\E\left[h_t\left(\sum_{i=1}^K (\tilde{p}_{t,i}^{1-\alpha})\ell_{t,i}^2\right)\right] \\
    &= \frac{(6d)^{2-\alpha}}{2(1-\alpha)}\E\left[\frac{1}{\alpha}\left(\sum_{i=1}^K p_{t,i}^\alpha - 1\right)\left(\sum_{i=1}^K (\tilde{p}_{t,i}^{1-\alpha})\ell_{t,i}^2\right)\right] \\
    &\leq \frac{(6d)^{2-\alpha}}{2\alpha(1-\alpha)}\E\left[ \left( \sum_{i=1}^K p_{t,i}^\alpha - 1 \right)\left(  \sum_{\ell_{t,i} \neq 0}\tilde{p}_{t,i}^{1-\alpha}\right) \right],
    \label{eq:boundhtztWorstCaseLosses}
\end{nalign}
where the second inequality is from $\ell_{t,i}^2 \leq 1$.

Using $p_{t,i^*}^\alpha - 1 \leq 0$ and $\sum_{i \neq i^*}p_{t,i}^\alpha \leq (K-1)^{1-\alpha}(\sum_{i \neq i^*} p_{t,i})^\alpha$ by Holder's inequality, we obtain 
\begin{align*}
    \sum_{i \in [K]}p_{t,i}^\alpha - 1 &\leq (K-1)^{1-\alpha} (\sum_{i \neq i^*}p_{t,i})^\alpha\\
    &\leq \frac{(K-1)^{1-\alpha}}{\Delta_{\min}^{\alpha}}(\sum_{i \in [K]} p_{t,i}\Delta_i)^\alpha.
\end{align*}
Next, from 
{$\tilde{p}_{t,i^*} \leq \sum_{i \neq i^*}p_{t,i}$} we have 
\begin{align*}
    \tilde{p}_{t,i*}^{1-\alpha} &\leq \left(\sum_{i \neq i^*}\tilde{p}_{t,i}\right)^{1-\alpha} \\
    &\leq \frac{1}{\Delta_{\min}^{1-\alpha}}\left(\sum_{i \neq i^*}p_{t,i} \Delta_i\right)^{1-\alpha}.
\end{align*}
Therefore, by Holder's inequality,
\begin{align*}
    \sum_{\ell_{t,i} \neq 0}\tilde{p}_{t,i}^{1-\alpha} &\leq \left(\sum_{\ell_{t,i} \neq 0, i \neq i^*}p_{t,i}^{1-\alpha}\right) + \tilde{p}_{t,i*}^{1-\alpha}\\
    &\leq \sum_{\ell_{t,i} \neq 0, i \neq i^*}\left(\Delta_i^{-\frac{1-\alpha}{\alpha}}\right)^{\alpha}(p_{t,i}\Delta_i)^{1-\alpha} + \frac{1}{\Delta_{\min}^{1-\alpha}}\left(\sum_{i \neq i^*}p_{t,i} \Delta_i\right)^{1-\alpha}\\
    &\leq \left(\sum_{\ell_{t,i} \neq 0, i \neq i^*}\Delta_i^{-\frac{1-\alpha}{\alpha}}\right)^\alpha \left(\sum_{\ell_{t,i} \neq 0, i \neq i^*} p_{t,i}\Delta_i\right)^{1-\alpha} + \frac{1}{\Delta_{\min}^{1-\alpha}}\left(\sum_{i \neq i^*}p_{t,i} \Delta_i\right)^{1-\alpha}\\
    &\leq \frac{S^\alpha}{\Delta_{\min}^{1-\alpha}}\left(\sum_{\ell_{t,i} \neq 0, i \neq i^*}p_{t,i}\Delta_i\right)^{1-\alpha}  + \frac{1}{\Delta_{\min}^{1-\alpha}}\left(\sum_{i \neq i^*}p_{t,i} \Delta_i\right)^{1-\alpha}\\
    &\leq \frac{2S^\alpha}{\Delta_{\min}^{1-\alpha}}\left(\sum_{i \in [K]}p_{t,i}\Delta_i\right)^{1-\alpha}.
\end{align*}
Overall, we have
\begin{nalign}
    \E[h_tz_t] &\leq \frac{(6d)^{2-\alpha}}{\alpha(1-\alpha)\Delta_{\min}}{(K-1)^{1-\alpha}S^{\alpha}}\E\left[\sum_{i=1}^K p_{t,i}\Delta_i\right].
    \label{eq:boundhtztbyRplusC}
\end{nalign}
Furthermore, by Jensen's inequality, 
\begin{align}
    \E\left[ \sqrt{2^{-J}h_{\max}\sum_{t=1}^T z_t} \right] &\leq \sqrt{\E\left[ {2^{-J}h_{\max}\sum_{t=1}^T z_t} \right] } \leq \sqrt{2^{-J}T\frac{(K-1)^{1-\alpha}}{\alpha}\frac{(6d)^{2-\alpha}S^{\alpha}}{2(1-\alpha)}}.
    \label{eq:boundhmaxsumzt}
\end{align}
By plugging $J = \ceil{\log_2(T)}$,~\eqref{eq:boundhtztbyRplusC} and~\eqref{eq:boundhmaxsumzt} into~\eqref{eq:boundztoverbetatbyItoCOLT2024}, we obtain
\begin{align*}
    \E\left[\sum_{t=1}^T \frac{z_t}{\beta_t}\right] &\leq O\left( \sqrt{\frac{\ln(T)(K^{1-\alpha}-1)S^{\alpha}\E[\sum_{t=1}^T \sum_{i=1}^K p_{t,i}\Delta_i]}{\alpha(1-\alpha)\Delta_{\min}}} + \sqrt{\frac{(K^{1-\alpha} - 1)S^{\alpha}}{\alpha(1-\alpha)}} \right) \\
    &\leq  O\left( \sqrt{\frac{\ln(T)(K^{1-\alpha}-1)S^{\alpha}(R_T + C)}{\alpha(1-\alpha)\Delta_{\min}}} + \sqrt{\frac{(K^{1-\alpha} - 1)S^{\alpha}}{\alpha(1-\alpha)}} \right).
\end{align*}
In summary, keeping only the dominant $\sqrt{T}$ terms, we have the following BOBW bounds that hold simultaneously:
\begin{itemize}
    \item In adversarial regime,
    \begin{align*}
        R_T \leq O\left(\sqrt{\frac{(K^{1-\alpha} - 1)S^\alpha T}{\alpha(1-\alpha)}}\right)
    \end{align*}
    \item In adversarial regime with a self-bounding constraint:
    \begin{align*}
        R_T \leq O\left(\frac{(K-1)^{1-\alpha}S^\alpha \ln(T)}{\alpha(1-\alpha)\Delta_{\min}} + \sqrt{C\frac{(K-1)^{1-\alpha}S^\alpha \ln(T)}{\alpha(1-\alpha)\Delta_{\min}}} + \sqrt{\frac{(K-1)^{1-\alpha}S^{\alpha}}{\alpha(1-\alpha)}}\right)
    \end{align*}
\end{itemize}

Note that we can explicitly set $\alpha$ sufficiently close $1$ so that $\frac{K^{1-\alpha}-1}{\alpha(1-\alpha)} = O(\ln(K)), \frac{(K-1)^{1-\alpha}}{\alpha(1-\alpha)} = O(\ln(K))$ and while $\gamma = O(\sqrtfrac{\alpha}{1-\alpha})$ grows with $K$ instead of $T$. For example, in Appendix~\ref{appendix:setAlphaCloseto1}, we show that $\alpha = 1 - \frac{1}{2\ln(K)}$ satisfies $\frac{K^{1-\alpha}-1}{\alpha(1-\alpha)} \leq 4\ln(K), \frac{(K-1)^{1-\alpha}}{\alpha (1-\alpha)} \leq 4\ln(K)$ while $\gamma \lesssim \sqrt{\ln(K)}$.
This ensures that 
\begin{align*}
    \gamma K\ln(T) + \frac{\beta_1 (K^{1-\alpha} - 1)}{\alpha} &= \gamma K\ln(T) + \frac{4K (K^{1-\alpha} - 1)}{\alpha (1-\alpha)} \\
    &= O\left(K\ln(K)\ln(T)\right),
\end{align*}
and
\begin{align*}
    \frac{1}{1-\alpha} = O(\ln(K))
\end{align*}
everywhere, so we can safely ignore the terms that do not contain $\sqrt{T}$ (in the adversarial setting) and $\frac{\ln(T)}{\Delta_{\min}}$ (in the stochastic setting).

\end{proof}

\subsection{Stability Proofs}
\label{sec:stabilityproofs}

In this section, we prove Lemma~\ref{lemma:stableTELB} and Lemma~\ref{lemma:boundhtplus1byht}.
First, we state and prove a number of supporting lemmas. In the following, we let 
    \begin{align}
        g_{\beta, \gamma}(t) = \beta t^{\alpha - 1} + \frac{\gamma}{t}.
        \label{eq:gt}
    \end{align}
be a function defined on $(0,1) \to \R_+$. Note that because $\beta > 0, \alpha \in (0,1)$ and $\gamma > 0$, this function $g_{\beta, \gamma}(t)$ is monotonically decreasing in $t$. We will drop the subscripts $\beta$ and $\gamma$ whenever they are clear from the context.

The first lemma shows that $\beta_{t+1} - \beta_t$ is sufficiently small for stabilizing the FTRL update in Algorithm~\ref{algo:optimalBOBWSparsity}.
\begin{lemma}
    For any $t \geq 1$, Algorithm~\ref{algo:optimalBOBWSparsity} guarantees
    \begin{align}
        \beta_{t+1} - \beta_t \leq (1 - \frac{1}{d})\gamma q_{t*}^{-\alpha},
    \end{align}
    where $q_{t*} = \min(\max_{i \in [K]}q_{t,i}, 1 - \max_{i \in [K]}q_{t,i})$.
    \label{lemma:betatplus1isgood}
\end{lemma}
\begin{proof}
    Lemma~\ref{lemma:lowerboundht} shows that $h_t \geq \frac{1-\alpha}{4\alpha}p_{t*}^{\alpha}$. By Lemma~\ref{lemma:2pstarlargerthanqstar}, we have $p_{t*}^{\alpha} \geq 2^{-\alpha}q_{t*}^\alpha$. This implies that 
    {$\frac{1}{h_t} \leq \frac{4\alpha}{1-\alpha}2^{\alpha}q_{t*}^{-\alpha}$}.
    By the definitions of $\beta_{t+1}, z_t$ and $h_t$, we have 
    \begin{align*}
        \beta_{t+1} - \beta_t &= \frac{z_t}{\beta_t h_t} \\
        &\leq \frac{4\alpha z_t}{(1-\alpha)\beta_t}2^{\alpha}q_{t*}^{-\alpha} \\
        &\leq \frac{4\alpha}{(1-\alpha)} \frac{18d^2}{\gamma}\ell_{t,I_t}^22^{\alpha}q_{t*}^{-\alpha} \\
        &\leq (1-\frac{1}{d})\gamma q_{t*}^{-\alpha}
    \end{align*}
    where the last inequality uses
    \begin{align}
        \frac{72\alpha d^2}{(1-\alpha)\gamma}\ell_{t,I_t}^22^{\alpha} \leq \frac{72\alpha d^2}{(1-\alpha)\gamma}2^{\alpha} \leq (1-\frac{1}{d})\gamma
    \end{align}
    for $d = 2$ and $\gamma \geq 48\sqrtfrac{\alpha}{1-\alpha}$.
\end{proof}
\begin{lemma}
    For any $L \in \R^K, \beta > 0, \gamma > 0$ and $h \in [-1, 1]$, let 
    \begin{align*}
        x &= \argmin_{p \in \Delta_K}\inp{L}{p} + \beta\left(\frac{1}{\alpha}(1-\sum_{i=1}^K p_i^\alpha)\right) - \gamma\sum_{i=1}^K \ln(p_i), \\
        y &= \argmin_{p \in \Delta_K}\inp{L + \frac{h}{x'_1}e_1}{p} + \beta\left(\frac{1}{\alpha}(1-\sum_{i=1}^K p_i^\alpha)\right) - \gamma\sum_{i=1}^K \ln(p_i).
    \end{align*}
    Here, $e_1$ is the first vector in the standard basis of $\R^K$. If $4x'_1 \geq x_1$ and $\gamma \geq 6$, then $y_1 \leq 3x_1$.
    \label{lemma:stableSameBetaDiffLoss}
\end{lemma}
\begin{proof}
    Using the Lagrange multiplier method, we have the following equalities that hold for some $Z \in \R$,
    \begin{align}
        \beta\left( y_1^{\alpha-1} - x_1^{\alpha - 1}\right) + \gamma\left(\frac{1}{y_1} - \frac{1}{x_1}\right) = Z + \frac{h}{x'_1}
        \label{eq:g1diff}
    \end{align}
    and for all $i \neq 1$,
    \begin{align}
        \beta\left( y_i^{\alpha-1} - x_i^{\alpha - 1}\right) + \gamma\left(\frac{1}{y_i} - \frac{1}{x_i}\right) = Z.
        \label{eq:gidiff}
    \end{align}
    First, we show that $Z$ and $y_1 - x_1$ has the opposite sign to $h$. 
    We consider two cases:
    \begin{itemize}
        \item If $Z \geq 0$ then from~\eqref{eq:gidiff}, we have $g(y_i) - g(x_i) = Z \geq 0$. This implies $y_i \leq x_i$ and leads to $y_1 \geq x_1$. From~\eqref{eq:g1diff}, we have $Z + \frac{h}{x'_1} = g(y_1) - g(x_1) \leq 0$. Since $Z \geq 0$, this implies $h \leq 0$.
        \item If $Z \leq 0$ then by the same argument, we have $y_i \geq x_i$ and $y_1 \leq x_1$. Therefore, $Z + \frac{h}{x'_1} \geq 0$. Due to $Z \leq 0$, we must have $h \geq 0$.
    \end{itemize}
    In both cases, we have $Zh \leq 0$ and $Z(y_1 - x_1) \geq 0$. It follows that if $h \geq 0$ then we have $y_1 \leq x_1 \leq 2x_1$. If $h < 0$ then $y_1 \geq x_1$, and by rearranging~\eqref{eq:g1diff}, we obtain
    \begin{align*}
        \frac{4}{x_1} \geq -\frac{h}{x'_1} &= \underbrace{Z}_{\geq 0} + \underbrace{\gamma\left(\frac{1}{x_1} - \frac{1}{y_1}\right)}_{\geq 0} + \underbrace{\beta(x_1^{\alpha-1} - y_1^{\alpha-1})}_{\geq 0} \\
        &\geq \gamma\left(\frac{1}{x_1} - \frac{1}{y_1}\right) \\
        &\geq 6\left(\frac{1}{x_1} - \frac{1}{y_1}\right),
    \end{align*}
    where the last inequality is due to $\gamma \geq 6$. This implies that $\frac{3}{y_1} \geq \frac{1}{x_1}$, thus $y_1 \leq 3x_1$.
\end{proof}

\begin{lemma}
    For any $L \in \R^K, \beta > 0, \beta' > 0, \gamma \geq 0$, define 
    \begin{align*}
        x &= \argmin_{p \in \Delta_K}\inp{L}{p} + \beta\left(\frac{1}{\alpha}(1-\sum_{i=1}^K p_i^\alpha)\right) - \gamma\sum_{i=1}^K \ln(p_i), \\
        y &= \argmin_{p \in \Delta_K}\inp{L}{p} + \beta'\left(\frac{1}{\alpha}(1-\sum_{i=1}^K p_i^\alpha)\right) - \gamma\sum_{i=1}^K \ln(p_i).
    \end{align*}
    Let $x_* = \min(\max_{i \in [K]}x_i, 1 - \max_{i \in [K]}x_i)$.
    For any constant $d \geq 2$, if 
    \begin{align}
        0 \leq \beta' - \beta \leq \left(1- \frac{1}{d}\right)\gamma x_*^{-\alpha},
    \end{align} 
    then $y_i \leq dx_i$ for all $i \in [K]$.
    \label{lemma:stableSameLossDiffBeta}
\end{lemma}
\begin{proof}
    Using the Lagrange multiplier method, we have for all $i \in [K]$,
    \begin{align}
        L_i - \beta x_i^{\alpha - 1} - \frac{\gamma}{x_i} = \lambda, \\
        L_i - \beta' y_i^{\alpha - 1} - \frac{\gamma}{y_i} = \lambda'.
    \end{align}
    Subtracting both sides of the two equations, we obtain 
    \begin{align*}
        \lambda - \lambda' + g_{\beta}(x_i) &= g_{\beta'}(y_i)\\
        &= g_{\beta}(y_i) + (\beta' - \beta)y_i^{\alpha - 1}.
    \end{align*}
    If $\lambda - \lambda' < 0$, then because $\beta \leq \beta'$, we have $g_{\beta}(x_i) > g_{\beta}(y_i) + (\beta' - \beta)y_i^{\alpha - 1} \geq g_{\beta}(y_i)$. 
    This implies $x_i < y_i$ for all $i \in [K]$, a contradiction to $\sum_{i=1}^K x_i = \sum_{i=1}^K y_i = 1$. Hence, we have $\lambda - \lambda' \geq 0$, and thus $g_\beta(x_i) \leq g_{\beta'}(y_i)$.
    
    For any $i \in [K]$, if $x_i > x_*$ then $x_i \geq \frac{1}{2}$ and hence $y_i \leq 1 \leq 2x_i \leq dx_i$. From the condition $\beta' - \beta \leq (1 - \frac{1}{d})\gamma x_*^{-\alpha}$, for $x_i \leq x_*$, we have 
      \begin{align*}
        g_{\beta'}(y_i) &\geq g_\beta(x_i) \\
        &= \beta x_i^{\alpha-1} + \frac{\gamma}{x_i} \\
        &\geq (\beta' - (1 - \frac{1}{d})\gamma x_*^{-\alpha})x_i^{\alpha - 1} + \frac{\gamma}{x_i} \\
        &= \beta' x_i^{\alpha - 1} - (1 - \frac{1}{d})\gamma x_*^{-\alpha}x_i^{\alpha - 1} + \frac{\gamma}{x_i} \\
        &\geq \beta' x_i^{\alpha - 1} - (1 - \frac{1}{d})\gamma x_i^{-\alpha}x_i^{\alpha - 1} + \frac{\gamma}{x_i} \\
        &= \beta' x_i^{\alpha - 1} + \frac{\gamma}{dx_i} \\
        &\geq \beta' (d x_i)^{\alpha - 1}+ \frac{\gamma}{d x_i} \\
        &= g_{\beta'}(d x_i),
      \end{align*}
      where the last inequality is due to $(d)^{\alpha - 1} \leq 1$.
      This implies $y_i \leq dx_i$ for all $x_i \leq x_*$.
\end{proof}

Let $\norm{x}_A = \sqrt{x^T Ax}$ be the norm of a vector $x \in \R^K$ induced by a positive definite matrix $A$. The following lemma proves Lemma~\ref{lemma:stableTELB} when the chosen arm $I_t$ satisfies $q_{t,I_t} \leq 1 - q_{t, I_t}$.
\begin{lemma}
    For any $t \in [T]$, Algorithm~\ref{algo:optimalBOBWSparsity} guarantees
    \begin{align}
        \inp{\hat{\ell}_t}{q_{t}  - q_{t+1}} - D_t(q_{t+1}, q_t) \leq \min\left( \frac{(6d)^{2-\alpha}}{2\beta_t(1-\alpha)}p_{t,I_t}^{2-\alpha}\hat{\ell}^2_{t,I_t}, \frac{18d^2}{\gamma} \ell_{t,I_t}^2  \right)
    \end{align}
    \label{lemma:stableItNotMax}
\end{lemma}
\begin{proof}
    Using standard local-norm analysis techniques for FTRL (for example, see Section 7.4 in~\cite{OrabonaIntroToOnlineLearningBook}), we have
    \begin{align}
        \inp{\hat{\ell}_t}{q_{t}  - q_{t+1}} - D_t(q_{t+1}, q_t) \leq \frac{1}{2}\norm{\hat{\ell}_t}^2_{(\nabla^2\phi_t(z_t))^{-1}},
        \label{eq:boundStabilityByLocalNorm}
    \end{align}
    where $z_t$ is a point between $q_t$ and $q_{t+1}$. The Hessian matrix of $\phi_t$ is a diagonal matrix with entries
    \begin{align}
        \nabla^2\phi_t(z_t) = \mathrm{diag}\left(\left(\beta_t(1-\alpha)z_{t,i}^{\alpha - 2} + \frac{\gamma}{z_{t,i}^2}\right)_{i=1,2,\dots,K}\right).
    \end{align}
    Hence, its inverse is the following diagonal matrix
    \begin{align}
        (\nabla^2\phi(z_t))^{-1} = \mathrm{diag}\left(\left(\frac{1}{\beta_t(1-\alpha)z_{t,i}^{\alpha - 2} + \frac{\gamma}{z_{t,i}^2}}\right)_{i=1,2,\dots,K}\right).
    \end{align}
    It follows that 
    \begin{nalign}
        \norm{\hat{\ell}_t}^2_{(\nabla^2\phi_t(z_t))^{-1}} &= \sum_{i=1}^{K} \hat{\ell}_{t,i}^2 \frac{1}{\beta_t(1-\alpha)z_{t,i}^{\alpha - 2} + \frac{\gamma}{z_{t,i}^2}} \\
        &\leq \min\left(\frac{1}{\beta_t(1-\alpha)}\sum_{i=1}^{K} z_{t,i}^{2-\alpha}\hat{\ell}_{t,i}^2, \frac{1}{\gamma}\sum_{i=1}^K z_{t,i}^2\hat{\ell}_{t,i}^2 \right) \\
        &= \min\left( \frac{1}{\beta_t(1-\alpha)} z_{t,I_t}^{2-\alpha}\hat{\ell}_{t,I_t}^2, \frac{z_{t,I_t}^2\hat{\ell}_{t,I_t}^2}{\gamma} \right),
        \label{eq:localnormbound}
    \end{nalign}
    where the last equality is due to $\hat{\ell}_{t,i} = 0$ for $i \neq I_t$. Combining~\eqref{eq:boundStabilityByLocalNorm} and~\eqref{eq:localnormbound}, we obtain
    \begin{align}
        \inp{\hat{\ell}_t}{q_{t}  - q_{t+1}} - D_t(q_{t+1}, q_t) 
        &\leq \min\left(\frac{1}{2\beta_t(1-\alpha)} z_{t,I_t}^{2-\alpha}\hat{\ell}_{t,I_t}^2, \frac{z_{t,I_t}^2\hat{\ell}_{t,I_t}^2}{2\gamma} \right).
    \end{align}
    Since $z_t$ is between $q_{t}$ and $q_{t+1}$, we have $z_{t,I_t} \leq \max(q_{t, I_t}, q_{t+1,I_t})$. 
    The loss estimate in Algorithm~\ref{algo:optimalBOBWSparsity} uses $p_{t,I_t}$ where $2p_{t,I_t} \geq q_{t,I_t}$ by Lemma~\ref{lemma:2pstarlargerthanqstar}, therefore we can combine the results of Lemma~\ref{lemma:betatplus1isgood}, Lemma~\ref{lemma:stableSameLossDiffBeta} and Lemma~\ref{lemma:stableSameBetaDiffLoss} and obtain $q_{t+1, I_t} \leq 3dq_{t,I_t}$. It follows that $z_{t, I_t} \leq 3dq_{t,I_t}  \leq 6dp_{t,I_t}$, and as a result,
    \begin{align}
        \inp{\hat{\ell}_t}{q_{t}  - q_{t+1}} - D_t(q_{t+1}, q_t) &\leq \min\left(\frac{(6d)^{2-\alpha}}{2\beta_t(1-\alpha)} p_{t,I_t}^{2-\alpha}\hat{\ell}_{t,I_t}^2, \frac{36d^2}{2\gamma} p_{t,I_t}^2 \hat{\ell}_{t,I_t}^2 \right) \\
        &\leq \min\left( \frac{(6d)^{2-\alpha}}{2\beta_t(1-\alpha)}p_{t,I_t}^{2-\alpha}\hat{\ell}^2_{t,I_t}, \frac{18d^2}{\gamma} \ell_{t,I_t}^2  \right),
    \end{align}
    where the last equality is due to $p_{t,I_t}^2\hat{\ell}_{t,I_t}^2 = \ell_{t,I_t}^2$.
\end{proof}

The next lemma proves Lemma~\ref{lemma:stableTELB} whenever the chosen arm $I_t$ has the maximum sampling probability. The proof is largely based on Lemma 9 in~\cite{ItoCOLT2024} and Equation 22 in~\cite{Tsuchiya2023stabilitypenaltyadaptive}.
\begin{lemma}
    For any $t \in [T]$, if $I_t \in \argmax_{i \in [K]}p_{t,i}$, Algorithm~\ref{algo:optimalBOBWSparsity} guarantees
    \begin{align}
        \inp{\hat{\ell}_t}{q_{t}  - q_{t+1}} - D_t(q_{t+1}, q_t) \leq \min\left(\frac{4}{\beta_t(1-\alpha)} (1-p_{t, I_t})^{2-\alpha}\hat{\ell}_{t,I_t}^2, \frac{4\ell_{t,I_t}^2}{\gamma} \right).
    \end{align}
    \label{lemma:stableItMax}
\end{lemma}
\begin{proof}
    When $I_t \in \argmax_{i \in [K]}p_{t,i}$, we have $I_t \in \argmax_{i \in [K]}q_{t,i}$ and thus $q_{t,I_t} \geq p_{t,I_t} \geq \frac{1}{K}$. Therefore, 
    \begin{align}
        \frac{\hat{\ell}_{t,I_t}}{\beta_t} &= \frac{\ell_{t,I_t}}{p_{t,I_t}\beta_t}
        \leq \frac{1}{p_{t,I_t}\beta_t}
        \leq \frac{K}{\beta_t}
        \leq \frac{1-\alpha}{4}
        \leq \frac{1-\alpha}{4}(1-q_{t,I_t})^{\alpha - 1},
    \end{align}
    where the third inequality is due to $\beta_t \geq \beta_1 \geq \frac{4K}{1-\alpha}$ by initialization, and the last inequality is from $(1-q_{t,I_t})^{\alpha - 1} \geq 1$ for $\alpha \in (0,1)$. Furthermore, for any $i \in [K] \setminus \{I_t\}$, we have $\frac{\hat{\ell}_{t,i}}{\beta_t} = 0 \geq -\frac{1-\alpha}{4}q_{t,i}^{\alpha-1}$. Therefore, by using Lemma 9 in~\cite{ItoCOLT2024} and noting that $\hat{\ell}_{t,i} = 0$ for $i \neq I_t$, we obtain 
    \begin{align}
        \inp{\frac{1}{\beta_t}\hat{\ell}_t}{q_{t} - q_{t+1}} - D_{TE}(q_{t+1}, q_t) \leq \frac{4}{\beta_t^2(1-\alpha)} (1-q_{t,I_t})^{2-\alpha}\hat{\ell}_{t,I_t}^2.
        \label{eq:invokeLemma9ItoCOLT2024}
    \end{align}
    Furthermore, Equation 22 in~\cite{Tsuchiya2023stabilitypenaltyadaptive} states that if $\frac{q_{t,I_t}\hat{\ell}_{t,I_t}}{\gamma} \geq -1$ and $\hat{\ell}_{t,i} = 0$ for $i \neq I_t$, then
    \begin{align}
        \inp{q_t - q_{t+1}}{\hat{\ell}_t} - \gamma D_{LB}(q_{t+1}, q_t) \leq \frac{q_{t,I_t}\hat{\ell}_{t,I_t}^2}{\gamma}.
        \label{eq:invokeEq22Tsuchiya2023}
    \end{align}
    Indeed, we have $\abs{\frac{q_{t,I_t}\hat{\ell}_{t,I_t}}{\gamma}} = \abs{\frac{q_{t,I_t}\ell_{t,I_t}}{p_{t,I_t}\gamma}} \leq \frac{1}{2\gamma} \leq \frac{1}{8}$ since $q_{t,I_t} \leq 2p_{t,I_t}$ by Lemma~\ref{lemma:2pstarlargerthanqstar} and $\gamma \geq 4$ by definition. 
    Therefore,~\eqref{eq:invokeLemma9ItoCOLT2024} and~\eqref{eq:invokeEq22Tsuchiya2023} together implies that 
    \begin{nalign}
        &\inp{\hat{\ell}_t}{q_{t}  - q_{t+1}} - D_t(q_{t+1}, q_t) \\
        &= \inp{\hat{\ell}_t}{q_{t}  - q_{t+1}} - \beta_tD_{TE}(q_{t+1}, q_t) - \gamma D_{LB}(q_{t+1}, q_t) \\
        &\leq \min\left(\beta_t\left(\inp{\frac{1}{\beta_t}\hat{\ell}_t}{q_{t}  - q_{t+1}} - D_{TE}(q_{t+1}, q_t)\right),  \inp{\hat{\ell}_t}{q_{t}  - q_{t+1}}  - \gamma D_{LB}(q_{t+1}, q_t)\right) \\
        &\leq \min\left(\frac{4}{\beta_t(1-\alpha)}(1-q_{t,I_t})^{2-\alpha}\hat{\ell}_{t,I_t}^2, \frac{q_{t,I_t}^2\hat{\ell}_{t,I_t}^2}{\gamma}\right) \\
        &\leq \min\left(\frac{4}{\beta_t(1-\alpha)}(1-p_{t,I_t})^{2-\alpha}\hat{\ell}_{t,I_t}^2, \frac{4\ell_{t,I_t}^2}{\gamma}\right),
    \end{nalign}
    where the first inequality is because Bregman divergences are non-negative.
\end{proof}

Next, we prove Lemma~\ref{lemma:stableTELB}.
\begin{proof}(Of Lemma~\ref{lemma:stableTELB})
    We consider two cases: 
    \begin{itemize}
        \item If $I_t \notin \argmax_{i \in [K]}p_{t,i}$ or $p_{t,I_t} \leq 1 - p_{t,I_t}$: in this case, we have $p_{t,I_t} = \min(p_{t,I_t}, 1 - p_{t,I_t})$. By Lemma~\ref{lemma:stableItNotMax}, we have
        \begin{align*}
            \inp{\hat{\ell}_t}{q_{t}  - q_{t+1}} - D_t(q_{t+1}, q_t) &\leq  \min\left( \frac{(6d)^{2-\alpha}}{2\beta_t(1-\alpha)}p_{t,I_t}^{2-\alpha}\hat{\ell}^2_{t,I_t}, \frac{18d^2}{\gamma} \ell_{t,I_t}^2  \right) \\
            &= \min\left( \frac{(6d)^{2-\alpha}}{2\beta_t(1-\alpha)}\tilde{p}_{t,I_t}^{2-\alpha}\hat{\ell}^2_{t,I_t}, \frac{18d^2}{\gamma} \ell_{t,I_t}^2  \right).
        \end{align*}
        
        \item If $p_{t,I_t} > 1 - p_{t,I_t}$: in this case, we have $I_t \in \argmax_{i \in [K]}q_{t,i}$. By Lemma~\ref{lemma:stableItMax}, 
        \begin{align*}
            \inp{\hat{\ell}_t}{q_{t}  - q_{t+1}} - D_t(q_{t+1}, q_t) &\leq \min\left(\frac{4}{\beta_t(1-\alpha)} (1-p_{t,I_t})^{2-\alpha}\hat{\ell}_{t,I_t}^2, \frac{4\ell_{t,I_t}^2}{\gamma}  \right) \\
            &= \min\left(\frac{4}{\beta_t(1-\alpha)}\min(p_{t,I_t}, 1 - p_{t,I_t})^{2-\alpha}\hat{\ell}_{t,i}^2, \frac{4\ell_{t,I_t}^2}{\gamma}  \right).
        \end{align*}
    \end{itemize}
    Lemma~\ref{lemma:stableTELB} follows by noting that $\max\left(\frac{(6d)^{2-\alpha}}{2}, 4\right) = \frac{(6d)^{2-\alpha}}{2}$.
\end{proof}
\begin{lemma}
    For any $L \in \R^K, \beta > 0, \gamma > 0$ and $h \in [-1, 1]$, let 
    \begin{align*}
        x &= \argmin_{p \in \Delta_K}\inp{L}{p} + \beta\left(\frac{1}{\alpha}(1-\sum_{i=1}^K p_i^\alpha)\right) - \gamma\sum_{i=1}^K \ln(p_i), \\
        y &= \argmin_{p \in \Delta_K}\inp{L + \frac{h}{x'_1}e_1}{p} + \beta\left(\frac{1}{\alpha}(1-\sum_{i=1}^K p_i^\alpha)\right) - \gamma\sum_{i=1}^K \ln(p_i),
    \end{align*}
    where $4x'_1 \geq x_1$. Fix an arbitrary $\omega \in (1, 2]$. If $\beta \geq \frac{4K}{(\omega-1)(1-\omega^{\alpha-1})}$ and $\gamma \geq 6$, then $y_i \leq 3x_i$ for all $i \in [K]$.
    \label{lemma:stableBigBeta}
\end{lemma}
\begin{proof}
    If $h \leq 0$, then we have $x_1 \leq y_1 \leq 3x_1$ and $y_i \leq x_i \leq 3x_i$ for all $i \neq 1$ from the proof of Lemma~\ref{lemma:stableSameBetaDiffLoss}. Thus, we focus on the case $h > 0$. 
    In this case, we have $y_1 \leq x_1 \leq 3x_1$ and $x_i \leq y_i$ for $i \neq 1$. From~\eqref{eq:g1diff} and~\eqref{eq:gidiff}, we have 
    \begin{align*}
        g(x_i) - g(y_i) = -Z \geq 0
    \end{align*}
    for all $i \neq 1$, and 
    \begin{align*}
        g(y_1) - g(x_1) = Z + \frac{h}{x'_1} \geq 0. 
    \end{align*}
    The latter implies that $-Z \leq \frac{1}{x'_1}$.
    Let $\eps = \frac{1}{\beta(1- \omega^{\alpha - 1})} \leq \frac{\omega - 1}{4K} \leq \frac{1}{4K}$. Similar to the proof of Lemma 13 in~\cite{ItoCOLT2024}, we consider two cases:
    \begin{itemize}
        \item If $x'_1 \geq \eps$, then $-Z \leq \frac{1}{\eps}$. For all $i \neq 1$,
        \begin{align*}
            g(y_i) &= g(x_i) + Z \\
            &\geq g(x_i) - \frac{1}{\eps} \\
            &= \beta x_i^{\alpha - 1} - \beta(1-\omega^{\alpha-1}) + \frac{\gamma}{x_i} \\
            &\geq \beta x_i^{\alpha - 1} - \beta x_i^{\alpha-1}(1-\omega^{\alpha-1}) + \frac{\gamma}{\omega x_i} \\
            &= \beta (\omega x_i)^{\alpha-1} + \frac{\gamma}{\omega x_i} = g(\omega x_i),
        \end{align*}
        where the last inequality is due to $x_i^{\alpha-1} \geq 1$ and $\omega > 1$. This implies that for all $i \neq 1$, $y_i \leq \omega x_i \leq 3x_i$ (since $\omega \leq 2$).

        \item If $x'_1 < \eps$, then we have $x_1 \leq 4x'_1 < \frac{1}{K}$. For any $i^* \in \argmax_{i \in [K]}x_{t,i}$, 
        {we have $i* \neq 1$}. Similar to the proof of Lemma 13 in~\cite{ItoCOLT2024}, we have $i^* \neq 1$ and $1 \leq \frac{y_{i*}}{x_{i*}} \leq \omega$.
        It follows that 
        \begin{align*}
            -Z &= g(x_{i*}) - g(y_{i*}) \\
            &= \beta x_{i*}^{\alpha - 1} + \frac{\gamma}{x_{i*}} - (\beta y_{i*}^{\alpha - 1} + \frac{\gamma}{y_{i*}}) \\
            &\leq \beta x_{i*}^{\alpha - 1} + \frac{\gamma}{x_{i*}} - (\beta \omega^{\alpha - 1}x_{i*}^{\alpha - 1} + \frac{\gamma}{\omega x_{i*}}) \\
            &= g(x_{i*}) - g(\omega x_{i*}).
        \end{align*}
        As the function $g(t) - g(\omega t)$ is decreasing for $\omega > 1$, we conclude that $-Z \leq g(x_i) - g(\omega x_i)$ for all $i \in [K]$. Therefore, $g(y_i) = g(x_i) + Z \geq g(\omega x_i)$ for all $i \neq 1$, which implies $y_i \leq \omega x_i \leq 3x_i$.
    \end{itemize}
    In both cases, we have $y_i \leq 3x_i$ for all $i \neq 1$. Combining this with $y_1 \leq x_1$, we conclude that $y_i \leq 3x_i$ for all $i \in [K]$.
\end{proof}
The following corollary is obtained by combining Lemma~\ref{lemma:stableSameLossDiffBeta} and Lemma~\ref{lemma:stableBigBeta}.
\begin{corollary}
    For any $L \in \R^K, \beta > 0, \gamma > 0$ and $h \in [-1, 1]$, let 
    \begin{align*}
        x &= \argmin_{p \in \Delta_K}\inp{L}{p} + \beta\left(\frac{1}{\alpha}(1-\sum_{i=1}^K p_i^\alpha)\right) - \gamma\sum_{i=1}^K \ln(p_i), \\
        y &= \argmin_{p \in \Delta_K}\inp{L + \frac{h}{x'_1}e_1}{p} + \beta'\left(\frac{1}{\alpha}(1-\sum_{i=1}^K p_i^\alpha)\right) - \gamma\sum_{i=1}^K \ln(p_i).
    \end{align*}
    Let $x_* = \min(\max_{i \in [K]}x_i, 1 - \max_{i \in [K]}x_i)$. For any $\omega \in (1, 2]$ and $d = 2$, if $2x'_1 \geq x_1$, $\gamma \geq 6, \beta \geq \frac{4K}{(\omega-1)(1-\omega^{\alpha-1})}$ and 
    \begin{align}
        0 \leq \beta' - \beta \leq (1 - \frac{1}{d})\gamma x_*^{-\alpha},
    \end{align}     
    then $y_i \leq 3dx_i$ for all $i \in [K]$.
    \label{corollary:stable}
\end{corollary}
\begin{proof}
    Let 
    \begin{align*}
        \bar{x} = \argmin_{p \in \Delta_K}\inp{L}{p} + \beta'\left(\frac{1}{\alpha}(1-\sum_{i=1}^K p_i^\alpha)\right) - \gamma\sum_{i=1}^K \ln(p_i).
    \end{align*}
    Here, $\bar{x}$ differs from $x$ only by the learning rates $\beta' \geq \beta$. 
    Applying Lemma~\ref{lemma:stableSameLossDiffBeta} with $d=2$, we obtain $\bar{x}_i \leq dx_i$ for all $i \in [K]$. In particular, $\bar{x}_1 \leq 2x_1$. Since $x_1 \leq 2x'_1$, we have $\bar{x}_1 \leq 4x'_1$.
    Next, since $\bar{x}$ differs from $y$ only by $\frac{h}{x'_1}e_1$ in the dot product, we apply Lemma~\ref{lemma:stableBigBeta} and obtain $y_i \leq 3\bar{x}_i$ for all $i \in [K]$. 
    Overall, we obtain $y_i \leq 3\bar{x}_i \leq 3dx_i$ for all $i \in [K]$.
\end{proof}
Finally, we are now ready to prove Lemma~\ref{lemma:boundhtplus1byht}.
\begin{proof}(Of Lemma~\ref{lemma:boundhtplus1byht})
    Let $\omega = 2$, we have $\beta_1 = \frac{8K}{1-\alpha} \geq \frac{4K}{(\omega-1)(1-\omega^{\alpha-1})}$ due to $2^\alpha \leq 1 + \alpha$ for $\alpha \in [0,1]$. 
    Since $z_t, h_t \geq 0$, the sequence of learning rates $(\beta_t)_t$ is increasing and hence, $\beta_t \geq \beta_1 \geq \frac{4K}{(\omega-1)(1-\omega^{\alpha-1})}$ for all $t \geq 1$.
    Together with Lemma~\ref{lemma:betatplus1isgood}, we have $\beta_{t+1} - \beta_t \leq \left(1 - \frac{1}{d}\right)\gamma q_{t*}^{-\alpha}$ and $\beta_t \geq \frac{4K}{(\omega-1)(1-\omega^{\alpha-1})}$ for all $t \geq 1$.
    In addition, we have $p_{t,I_t} \geq 2q_{t,I_t}$ by Lemma~\ref{lemma:2pstarlargerthanqstar}.
    Applying Corollary~\ref{corollary:stable} for
    \begin{align*}
        q_t &= \argmin_{x \in \Delta_K} \inp{L_{t-1}}{x} + \beta_t\left(\frac{1}{\alpha}(1-\sum_{i=1}^K x_i^\alpha)\right) - \gamma\sum_{i=1}^K \ln(x_i), \\
        q_{t+1} &= \argmin_{x \in \Delta_K}\inp{L_{t-1} + \frac{\ell_{t,I_t}}{p_{t,I_t}}e_{I_t}}{x} + \beta_{t+1}\left(\frac{1}{\alpha}(1-\sum_{i=1}^K x_i^\alpha)\right) - \gamma\sum_{i=1}^K \ln(x_i),
    \end{align*}
    we obtain $q_{t+1,i} \leq 3dq_{t,i}$ for all $i \in [K]$. 
    
    For the second statement, we apply Lemma 11 in~\cite{ItoCOLT2024}.
\end{proof}

\subsection{Technical Lemmas}
\label{sec:technicallemmas}
\begin{lemma}
    For any $a > 0, b > 0$ and $x \geq 0$, we have 
    \begin{align}
        a^x + b^x &\geq (a+b)^x \qquad\text{if }x \in [0,1] \\
        a^x + b^x &\leq (a+b)^x \qquad\text{if }x \geq 1.
    \end{align}.
    \label{lemma:axbxlargerthanaplusbx}
\end{lemma}
\begin{proof}
    Consider the following function defined on $\R_+$: 
    \begin{align}
        f(x) = \ln(a^x + b^x) - x\ln(a+b).
    \end{align}
    Its derivative is
    \begin{align}
        f'(x) &= \frac{a^x \ln(a) + b^x \ln(b)}{a^x + b^x} - \ln(a+b) \\
        &= \frac{a^x \ln(\frac{a}{a+b}) + b^x \ln(\frac{b}{a+b}) }{a^x + b^x}.
    \end{align}
    Since $\frac{a}{a+b} \leq 1$ and $\frac{b}{a+b} \leq 1$, we have $f'(x) \leq 0$. Therefore, 
    \begin{itemize}
        \item If $x \in [0,1]$: we have $f(x) \geq f(1) = 0$. This implies $\ln(a^x + b^x) \geq x\ln(a+b)$ for all $x \in [0,1]$. Equivalently, $a^x + b^x \geq e^{x\ln(a+b)} = (a+b)^x$.
        \item If $x \geq 1$: we have $f(x) \leq f(1) = 0$. This implies $\ln(a^x + b^x) \leq x\ln(a+b)$, which leads to $a^x + b^x \leq e^{x\ln(a+b)} = (a+b)^x$ for $x \geq 1$.
    \end{itemize}
   
\end{proof}
\begin{lemma}
    Let $0 \leq a, b \leq 1$ and $a + b = 1$. Then, for any $x \in [0, 1]$, we have
    \begin{align}
        (\max(a, b))^x + (\min(a,b))^x 2^{x-1} \geq 1.
    \end{align}
    \label{lemma:apowerxbpowerxtwopowerx}
\end{lemma}
\begin{proof}
    Without loss of generality, assume $a \geq b$. It follows that $2b \leq 1$.   Consider the following function defined on $[0, 1]$:
    \begin{align}
        f(x) = \frac{b^x 2^{x}}{2} + a^x.
    \end{align}
    Its derivative is
    \begin{align}
        f'(x) &= \frac{1}{2}\left[ b^x \ln(b) 2^x + b^x 2^x \ln(2) \right] + a^x\ln(a)  \\
        &= b^x 2^{x-1} \ln(2b) + a^x\ln(a).
    \end{align}
    Since $2b \leq 1$ and $a \leq 1$, we have $f'(x) \leq 0$. Therefore, for all $x \in [0,1]$, we  have 
    \begin{align}
        f(x) \geq f(1) = a + b = 1.
    \end{align}
\end{proof}
\begin{lemma}
    For any $q \in \Delta_K$, we have
    \begin{align}
        (-\psi_{TE}(q_t)) = \frac{1}{\alpha}\left(\sum_{i=1}^K q_i^\alpha  - 1 \right) \geq \frac{q_{*}^{\alpha}}{\alpha}(1-2^{\alpha-1}),
        \label{eq:lowerboundht}
    \end{align}
    where $q_* = \min(\max_{i \in [K]}q_i, 1 - \max_{i \in [K]}q_i)$. This implies $(-\psi_{TE}(q_t)) \geq \frac{q_{*}^{\alpha}}{4\alpha}(1-\alpha)$.
    \label{lemma:lowerboundht}
\end{lemma}
\begin{proof}
    Let $q_{\max} = \max_{i \in [K]}q_i$.
    We consider two cases: $q_{\max} \leq 0.5$ and $q_{\max} > 0.5$.
    \begin{itemize}
        \item When $q_{\max} \leq 0.5$: we have $q_* = q_{\max}$. For any $i_{\max} \in \argmax_{i \in [K]} q_i$, the inequality~\eqref{eq:lowerboundht} is equivalent to
        \begin{align}
            q_*^\alpha 2^{\alpha - 1} + \sum_{i \neq i_{\max}}q_i^\alpha \geq 1,
        \end{align}
        Using 
        \begin{align}
            \sum_{i \neq i_{\max}}q_i^\alpha \geq (\sum_{i \neq i_{\max}}q_i)^\alpha
        \end{align}
        from Lemma~\ref{lemma:axbxlargerthanaplusbx} and combining with Lemma~\ref{lemma:apowerxbpowerxtwopowerx} leads to the desired claim.

        \item When $q_{\max} > 0.5$: in this case, we have $q_* = 1 - q_{\max} = \sum_{i \neq i_{\max}}q_i$. The desired inequality is equivalent to 
        \begin{align}
            (q_{i_{\max}}^\alpha + q_*^\alpha 2^{\alpha-1}) + \sum_{i \neq i_{\max}}q_i^\alpha \geq 1 + (\sum_{i \neq i_{\max}}q_i)^\alpha.
        \end{align}
        Again, this follows directly from 
        \begin{align}
            \sum_{i \neq i_{\max}}q_i^\alpha \geq (\sum_{i \neq i_{\max}}q_i)^\alpha
        \end{align}
        and Lemma~\ref{lemma:apowerxbpowerxtwopowerx}.
    \end{itemize}
    The implication statement follows by $2^{\alpha-1} \leq (\alpha + 3)/4$ for $\alpha \in [0,1]$.
\end{proof}

\begin{lemma}
    Let $q \in \Delta_K$ and $p = (1-\frac{K}{T})q + \frac{1}{T}\1$ where $T \geq 4K$. The following properties hold:
    \begin{itemize}
        \item $q_{i} \leq 2p_{i}$ for any $i \in [K]$.
        \item $q_{*} \leq 2p_{*}$ where $q_* = \min(\max_{i \in [K]}q_i, 1 - \max_{i \in [K]}q_i)$ and $p_* = \min(\max_{i \in [K]}p_i, 1 - \max_{i \in [K]}p_i)$.
        \item $p_* \geq \frac{1}{T}$.
    \end{itemize}
    \label{lemma:2pstarlargerthanqstar}
\end{lemma}
\begin{proof}
    By the definition of $p$, we have
    \begin{align*}
        2p_{i} = 2(1 - \frac{K}{T})q_i + \frac{2}{T} \geq q_i + q_i(1 - \frac{2K}{T}) \geq q_i.
    \end{align*}
    Thus, the first statement holds. 
    
    Next, let $k \in \argmax_{i \in [K]}q_i$. Obviously, we have $k \in \argmax_{i \in [K]}p_i$  due to $p_i \geq p_j$ if $q_i \geq q_j$ . If $q_{k} \leq 0.5$, then we have $q_* = q_{k}$. We also have $q_{k} \geq \frac{1}{K}$ and therefore $p_{k} = q_{k} + \frac{1 - Kq_{k}}{T} \leq  q_{k} \leq 0.5$. Moreover, 
    \begin{align*}
        2p_* = 2p_{k} \geq q_{k} = q_*,
    \end{align*}
    where the inequality is from the first statement. On the other hand, if $q_{k} > 0.5$ then we have $q_* = 1 - q_k$ and
    \begin{align*}
        p_* &= \min(p_k, 1 - p_k) \\
        &= \min\left((1-\frac{K}{T})q_k + \frac{1}{T}, 1 - \left((1-\frac{K}{T})q_k + \frac{1}{T}\right)\right) \\
        &= \min\left(q_k + \frac{1-Kq_k}{T}, 1 - q_k + \frac{Kq_k - 1}{T} \right).
    \end{align*}
    Since $q_k \geq \frac{1}{K}$, we have $1 - q_k + \frac{Kq_k - 1}{T} \geq 1 - q_k \geq \frac{1-q_k}{2}$. In addition, 
    \begin{align*}
        q_k + \frac{1-Kq_k}{T} \geq q_k + \frac{-K}{T} > 0.5 - \frac{1}{4} = \frac{1}{4} \geq \frac{1-q_k}{2}.
    \end{align*}
    Hence, we conclude that $p_* \geq \frac{1-q_k}{2} = \frac{q_*}{2}$. Thus, the second statement holds. The last statement follows from $p_* \geq \min_{i \in [K]}p_i \geq \frac{1}{T}$.
\end{proof}
\begin{proof}(Of Lemma~\ref{lemma:minhtAndmaxEzt})
    Lemma~\ref{lemma:lowerboundht} implies that for all $t$, we have 
    \begin{align*}
        h_t \geq \frac{(p_t)_*^\alpha}{4\alpha} \geq \frac{T^{-\alpha}}{4\alpha},
    \end{align*}
    where the last inequality is from $(p_t)_* \geq \frac{1}{T}$ by Lemma~\ref{lemma:2pstarlargerthanqstar}. Additionally, 
    \begin{align*}
        \E_{I_t}[z_t] &\leq \frac{(6d)^{2-\alpha}}{2(1-\alpha)}\E_{I_t}\left[ (\tilde{p}_{t,I_t})^{2-\alpha} \hat{\ell}_{t,I_t}^2 \right] \\
        &= \frac{(6d)^{2-\alpha}}{2(1-\alpha)}\sum_{i=1}^K \frac{(\tilde{p}_{t,i})^{2-\alpha}\ell_{t,i}^2}{p_{t,i}} \\
        &\leq \frac{(6d)^{2-\alpha}}{2(1-\alpha)}\sum_{i=1}^K (\tilde{p}_{t,i})^{1-\alpha} (\ell_{t,i}^{\frac{2}{\alpha}})^{\alpha} \\
        &\leq \frac{(6d)^{2-\alpha}}{2(1-\alpha)} (\sum_{i=1}^K \tilde{p}_{t,i})^{1-\alpha}(\sum_{i=1}^K \ell_{t,i}^{2/\alpha})^{\alpha} \\
        &\leq \frac{(6d)^{2-\alpha}}{2(1-\alpha)} S^{\alpha},
    \end{align*}
    where the last inequality uses $\sum_{i=1}^K \tilde{p}_{t,i} \leq \sum_{i=1}^K p_{t,i} = 1$ and $S \geq \norm{\ell_t}_0$.
\end{proof}

\section{Proof of the Lower Bounds in Theorem~\ref{thm:LowerBounds}}
\label{sec:lowerboundproofs}
\subsection{Stochastic Lower Bound}

Let $i^* \in [K]$ be fixed. We pick $\Delta_{\min} = \frac{U}{K^\alpha + 1}$ so that $\Delta_{\min} \leq 0.25$.
Our construction is as follows:
\begin{align*}
    \ell_t = \begin{cases}
        -\1 &\qquad\text{with probability } \Delta_{\min},\\
        -e_{i^*} &\qquad\text{with probability } \Delta_{\min}, \\
        \vect{0} &\qquad\text{with probability } 1-2\Delta_{\min},
    \end{cases}
\end{align*}
where $e_{i^*}$ is the $i^*$-th vector in the standard basis of $\R^K$.
The expected loss vector is
\begin{align*}
    \E\left[ \ell_t\right] = -\Delta_{\min}\1 - \Delta_{\min}e_{i^*}.
\end{align*}
It follows that $\Delta_i = \Delta_{\min}$ for all $i \in [K] \setminus \{i^*\}$. In addition, the losses of arm $i^*$ follow a Bernoulli distribution $\mathrm{Ber}(2\Delta_{\min})$ while the losses of sub-optimal arms follow a Bernoulli distribution $\mathrm{Ber}(\Delta_{\min})$.
We verify that the constraint in~\eqref{eq:softconstraint} holds:
\begin{align*}
    \E\left[\left(\sum_{i=1}^K \abs{\ell_{t,i}}^{2/\alpha}\right)^\alpha \right] = \Delta_{\min} K^\alpha + \Delta_{\min} = \Delta_{\min}(K^\alpha + 1) = U.
\end{align*}

For any consistent algorithm such that $R_T = o(T^x)$ for any $x > 0$, by~\cite{LaiAndRobbins1985}, we have
\begin{align*}
    \liminf_{T \to \infty} \frac{R_T}{\ln(T)} &\geq \sum_{i \neq i^*}^K \frac{\Delta_i}{\infdiv{\Delta_{\min}}{2\Delta_{\min}}} \\
    &\geq \frac{K \Delta_{\min}}{2\infdiv{\Delta_{\min}}{2\Delta_{\min}}} \\
    &\geq \frac{K}{2} \\
    &= K^{1-\alpha}\frac{K^\alpha}{2} \\
    &\geq K^{1-\alpha}\frac{K^\alpha + 1}{4} = \frac{K^{1-\alpha}U}{4\Delta_{\min}},
\end{align*}
where the second inequality is from $K-1 \geq \frac{K}{2}$ for all $K \geq 4$, the third inequality is $\infdiv{p}{2p} \leq p$  for $p \leq 0.25$ by Lemma~\ref{lemma:KLp2pisp}, and the last inequality is $K^\alpha \geq \frac{K^\alpha + 1}{2}$. We conclude that 
\begin{align*}
    \liminf_{T \to \infty}\frac{R_T}{\ln(T)} \gtrsim \frac{K^{1-\alpha}U}{2\Delta_{\min}}.
\end{align*}

\begin{lemma}
    For any $p \in (0, 0.25]$, we have 
    \begin{align*}
        \infdiv{p}{2p} \leq p.
    \end{align*}
    \label{lemma:KLp2pisp}
\end{lemma}
\begin{proof}
    \begin{align*}
        \infdiv{p}{2p} &= -p\ln(2) + (1-p)\ln\left(\frac{1-p}{1-2p}\right) \\
        &= -p\ln(2) + (1-p)\ln\left(1 + \frac{p}{1-2p}\right) \\
        &\leq -p\ln(2) + (1-p)\frac{p}{1-2p} \\
        &\leq -p\ln(2) + 1.5p \leq p
    \end{align*}
    where the first inequality is $\ln(1+x) \leq x$ for all $x > -1$ and the second inequality is $\frac{(1-p)p}{1-2p} \leq 1.5p$ for all $p \in (0, 0.25]$.
\end{proof}

\subsection{Adversarial Lower Bound}
For the adversarial lower bound, we construct a neutral environment $V_0$ and $K$ competing environments $V_1, V_2, \dots, V_K$, where:
\begin{itemize}
    \item On $V_0$, the loss function is chosen by
        \begin{align*}
            \ell_t = \begin{cases}
                -\1 &\qquad\text{with probability } \eta,\\
                \vect{0} &\qquad\text{with probability } 1-\eta.
            \end{cases}
        \end{align*}
    It follows that the losses of all arms follow a Bernoulli distribution $\mathrm{Ber}(\eta)$ on $V_0$.

    \item On $V_i$ for $i \in [K]$, the loss function is chosen by
        \begin{align*}
            \ell_t = \begin{cases}
                -\1 &\qquad\text{with probability } \eta,\\
                -e_{i} &\qquad\text{with probability } \eps, \\
                \vect{0} &\qquad\text{with probability } 1-\eta-\eps.
            \end{cases}
        \end{align*}

        It follows that except for arm $i$, the losses of all other arms follow a Bernoulli distribution $\mathrm{Ber}(\eta)$ on $V_i$. The loss of arm $i$ follows $\mathrm{Ber}(\eta + \eps)$.
\end{itemize}
Here, $\eta$ and $\eps$ are constants chosen to be the solution of the following system of (in)equalities:
\begin{itemize}
    \item $\eta + \eps \leq \frac{1}{4}$.
    \item $\eta K^\alpha + \eps = U$.
    \item $\frac{T}{K}\frac{8\eps^2}{\eta} = 1$.
    \item $\eta K^\alpha \geq \frac{U}{2}$.
\end{itemize}
Note that for all $K \geq 4, T \geq 4K$ and $U \leq \frac{K^\alpha}{4}$, the solution 
\begin{nalign}
    &\sqrt{\eta} = \frac{-\sqrtfrac{K}{8T} + \sqrt{\frac{K}{8T} + 4K^\alpha U}}{2K^\alpha}, \\
    &\eps = \sqrtfrac{\eta K}{8T}
    \label{eq:etaAndeps}
\end{nalign}
satisfies the system of inequalities since 
\begin{align*}
    &\sqrt{\eta} = \frac{-\sqrtfrac{K}{8T} + \sqrt{\frac{K}{8T} + 4K^\alpha U}}{2K^\alpha} \leq \frac{\sqrt{1 + \frac{1}{32}}}{2K^\alpha} \leq \frac{1}{8}, \\
    &\eps = \sqrtfrac{\eta K}{8T} \leq \sqrtfrac{K}{64T} \leq \frac{1}{8}.
\end{align*}
Moreover, we have
\begin{align*}
    \eta  &= \frac{1}{4K^{2\alpha}}\left(-\sqrtfrac{K}{8T} + \sqrt{\frac{K}{8T} + 4K^\alpha U}\right)^2 \\
    &= \frac{1}{4K^{2\alpha}} \frac{16K^{2\alpha} U^2}{\left(\sqrtfrac{K}{8T} + \sqrt{\frac{K}{8T} + 4K^\alpha U}\right)^2} \\
    &= \frac{4U^2}{\left(\sqrtfrac{K}{8T} + \sqrt{\frac{K}{8T} + 4K^\alpha U}\right)^2} \\
    &\geq \frac{U^2}{4K^\alpha U} = \frac{1}{4}\frac{U}{K^\alpha} \\
    &\geq \frac{K^{1-2\alpha}}{8T},
\end{align*}
where the second equality is $\sqrt{a} - \sqrt{b} = \frac{a-b}{\sqrt{a}+\sqrt{b}}$, the first inequality is $\frac{K}{T} \leq \frac{K^\alpha U}{4}$ and the last inequality is $\frac{U}{K^\alpha} \geq \frac{K^{1-2\alpha}}{2T}$, both hold for all $T \geq 4K$ and $U \geq 1$. This implies that
\begin{align*}
    \eta^2 K^{2\alpha} = \eta (\eta K^{2\alpha}) \geq \eta \frac{K^{1-2\alpha}}{8T} K^{2\alpha} = \frac{\eta K}{8T} = \eps^2.
\end{align*}
As a result, $\eta K^\alpha \geq \eps = U - \eta K^\alpha$. Hence, $\eta K^\alpha \geq \frac{U}{2}$.

With the choice of $\eta$ and $\eps$ in~\eqref{eq:etaAndeps}, we verify that~\eqref{eq:softconstraint} holds:
\begin{align*}
    \E\left[\left(\sum_{i=1}^K \abs{\ell_{t,i}}^{2/\alpha}\right)^\alpha \right] = \eta K^\alpha + \eps = U.
\end{align*}

Let $\gA$ denote the algorithm of a learner. 
Let $N_i = \sum_{t=1}^{T} \I{I_t = i}$ denote the number of times arm $i$ is pulled by $\gA$. 
Let $\P_0$ and $\P_i$ denote the distribution of the observed losses on $V_0$ and $V_i$, respectively.
Similarly, let $\E_0$ and $\E_i$ denote the expectation taken on $V_0$ and $V_i$, respectively. 

We first run $\gA$ on $V_0$. 
Let $a = \argmin_{i \in [K]} \E_0[N_i]$ be the arm that is pulled the least in expectation. Since $\sum_{i=1}^K N_i = T$, we have $\E_0[N_a] \leq \frac{T}{K}$.

By the standard arguments in establishing lower bounds for adversarial bandits~\citep[e.g.][Equation 28-30]{Auer2002a}, we have 
\begin{align*}
    \E_a[N_a] &\leq \E_0[N_a] + \frac{T}{2}\norm{\P_a - \P_0}_1 \\
    &\leq \E_0[N_a] + \frac{T}{2}\sqrt{2\ln(2)\infdiv{\P_0}{\P_a} } \\
    &\leq \E_0[N_a] + \frac{T}{2}\sqrt{2\ln(2)\E_0[N_a]\infdiv{\eta}{\eta + \eps} } \\
    &\leq \frac{T}{K} + \frac{T}{2}\sqrt{2\ln(2)\frac{T}{K}\infdiv{\eta}{\eta + \eps} } \\
    &\leq \frac{T}{K} + \frac{T}{2}\sqrt{2\ln(2)\frac{T}{K}\frac{4\log_2(e) \eps^2}{\eta} } \\
    &= \frac{T}{K} + \frac{T}{2}\sqrt{\frac{T}{K}\frac{8\eps^2}{\eta} } ,
\end{align*}
where
\begin{itemize}
    \item the second inequality is Pinsker's inequality,
    \item the third inequality is due to the chain rule for KL-divergence~\citep{CoverAndThomas2006} and the fact that $V_0$ and $V_a$ differ only by the loss distribution of arm $a$,
    \item the fourth inequality is because $\E_0[N_a] \leq \frac{T}{K}$,
    \item the last inequality is the reverse Pinsker's inequality~\citep{sason2015reverse}.
\end{itemize}
This further implies that $\E_a[N_a] \leq \frac{T}{K} + \frac{T}{2} \leq \frac{3T}{4}$ for $K \geq 4$. Hence,
\begin{align*}
    \E_a[R_{T,a}] &= \eps (T - \E_a[N_a]) \\
    &\geq \eps (T - \frac{3T}{4}) = \frac{T\eps}{4} \\
    &= \sqrtfrac{TK\eta}{32} \\
    &= \Omega(\sqrt{K^{1-\alpha}UT}),
\end{align*}
where the last equality is due to $\eta K^\alpha \geq \frac{U}{2}$.

\section{Proofs for Section~\ref{sec:SPMwOFTRLwReservoirSampling}}
\label{appendix:proofsforSectionOFTRLReservoirSampling}
Before proving Theorem~\ref{thm:BOBWsqrtQLnKBound}, in Appendix~\ref{appendix:GeneralSPMforOFTRL}, we first establish the BOBW regret bound for an algorithm that combines real-time SPM and Optimistic FTRL without any reservoir samplings. 
The full procedure is given in Algorithm~\ref{algo:OFTRL}.
Later on, in Appendix~\ref{appendix:proofForTheoremBOBWsqrtQLnBound}, we will use this regret bound in the analysis of Algorithm~\ref{algo:OFTRLReservoirSampling}.

\subsection{A General SPM-based Regret Bound for Optimistic FTRL}
\label{appendix:GeneralSPMforOFTRL}
\begin{algorithm}[t]
	\KwIn{$K \geq 1, T \geq 4K, \alpha \in (0, 1), \beta_1 = \frac{4K}{1-\alpha}, \gamma = \max(3,48\sqrtfrac{\alpha}{1-\alpha}), d = 2$.}
    Initialize $L_{0,i} = 0$ for $i \in [K]$ \;
	
    \For{each round $t = 1, \dots, T$}{        
        Compute $m_t \in [0,1]^K$ \;
        
        Compute $q_t = \argmin_{x \in \Delta_K} \inp{m_t + L_{t-1}}{x} + \beta_t\left(\frac{1}{\alpha}(1-\sum_{i=1}^K x_i^\alpha)\right) - \gamma\sum_{i=1}^K \ln(x_i)$\;
        
        Compute $p_t = \left(1 - \frac{K}{T}\right)q_t + \frac{1}{T}\1$\;
		
        Draw $I_t \sim p_t$ and observe $\ell_{t, I_t}$\;
		
        Compute loss estimate $\hat{\ell}_{t, i} = m_{t,i} + \frac{(\ell_{t,i} - m_{t,i})\I{I_t = i}}{p_{t,i}}$ \;
        $L_{t,i} = L_{t-1,i} + \hat{\ell}_{t,i}$\;		
        
        Compute $z_t =  \min\left( \frac{(6d)^{2-\alpha}}{2(1-\alpha)}\tilde{p}_{t,I_t}^{2-\alpha}(\hat{\ell}_{t,I_t} - m_{t,I_t})^2, \frac{\beta_t18d^2}{\gamma}(\ell_{t,I_t} - m_{t,I_t})^2  \right)$ \;
        
        Compute $h_t = \left(\frac{1}{\alpha}(\sum_{i=1}^K p_{t,i}^\alpha - 1)\right)$\;
        
        Compute $\beta_{t+1} = \beta_t + \frac{z_t}{\beta_t h_t}$\;
	}
	\caption{Optimistic FTRL using Tsallis entropy plus log-barrier regularization for losses in $[0,1]$}
	\label{algo:OFTRL}
\end{algorithm}
We consider the adversarial multi-armed bandits with losses in $[0,1]$. Note that the analysis can be trivially extended to the $[-1, 1]$ case by increasing $\beta_t$ and $\gamma$ by a multiplicative factor of $2$.

In round $t$, the learner computes $m_t \in [0,1]^K$ before drawing arm $I_t$ and uses Optimistic FTRL with the hybrid regularizer
\begin{align*}
    q_t &= \argmin_{x \in \Delta_K} \inp{m_t + \sum_{s=1}^{t-1}\hat{\ell}_s}{x} + \phi_t(x) \\
    &= \argmin_{x \in \Delta_K} \inp{m_t + L_{t-1}}{x} + \beta_t\psi_{TE}(x) + \gamma\psi_{LB}(x) \\
    &= \argmin_{x \in \Delta_K} \inp{m_t + L_{t-1}}{x} + \beta_t\left(1 - \sum_{i=1}^K x_i^\alpha\right) - \gamma\sum_{i=1}^K \ln(p_i).
\end{align*}
Let $p_t = \left(1-\frac{K}{T}\right)q_t + \frac{1}{T}\1$.
The learner draws $I_t \sim p_t$ and use the unbiased loss estimator
\begin{align*}
    \hat{\ell}_{t,i} = m_{t,i} + \frac{\ell_{t,i} - m_{t,i}}{p_{t,i}}\I{I_t = i}.
\end{align*}
\textbf{SPM learning rates}: the learning rates are set according to SPM rule~\citep{ItoCOLT2024}, where 
\begin{align*}
    h_t &= (-\psi_{TE})(p_{t}) = \frac{1}{\alpha}\left(\sum_{i=1}^K p_{t,i}^{\alpha}-1\right), \\
    z_t &= \min\left( \frac{(6d)^{2-\alpha}}{2(1-\alpha)}\tilde{p}_{t,I_t}^{2-\alpha}(\hat{\ell}_{t,I_t} - m_{t,I_t})^2, \frac{\beta_t18d^2}{\gamma}(\ell_{t,I_t} - m_{t,I_t})^2  \right).
\end{align*}
Details are given in Algorithm~\ref{algo:OFTRL}.

\subsection{Analysis for Algorithm~\ref{algo:OFTRL}}

Similar to~\cite{ItoCOLT2022aVariance}, the analysis uses 
\begin{align*}
    r_{t+1} &= \argmin_{x \in \Delta_K} \inp{L_{t-1} + \hat{\ell}_t}{x} + \frac{\beta_{t+1}}{\alpha}(1 - \sum_{i=1}^K x_i^\alpha) - \gamma\sum_{i=1}^K \ln(x_i) \\
    &= \argmin_{x \in \Delta_K} \inp{L_{t-1} + m_t + \frac{\ell_{tI_t} - m_{t,I_t}}{p_{t,I_t}}e_{I_t}}{x} + \frac{\beta_{t+1}}{\alpha}(1 - \sum_{i=1}^K x_i^\alpha) - \gamma\sum_{i=1}^K \ln(x_i),
\end{align*}
where $2p_{t,I_t} \geq q_{t,I_t}$.
Observe that $(\ell_{t,I_t} - m_{t,I_t}) \in [-1,1]$ and 
\begin{align*}
    \beta_{t+1} - \beta_t &= \frac{z_t}{\beta_{t}h_t} \\
    &\leq \frac{18d^2}{h_t\gamma} \\
    &\leq \left(1 - \frac{1}{d}\right)\gamma q_{t*}^{-\alpha},
\end{align*} 
similar to the proof of Lemma~\ref{lemma:betatplus1isgood}. 
Therefore, we can invoke Corollary~\ref{corollary:stable} with $\omega = 2$ and obtain $r_{t+1,i} \leq 3q_{t,i} \leq 6p_{t,i}$ for all $i \in [K]$. Combining this with Lemma 1 in~\cite{ItoCOLT2022aVariance} and our Lemma~\ref{lemma:stableTELB}, we obtain
\begin{align*}
    \sum_{t=1}^T \inp{\hat{\ell}_t}{q_t - u} &\leq \phi_{T+1}(u) - \phi_1(r_1)
    + \sum_{t=1}^T(\phi_{t}(r_{t+1}) - \phi_{t+1}(r_{t+1})) \\
    &+ \sum_{t=1}^T \inp{\hat{\ell}_t - m_t}{q_t - r_{t+1}} - D_t(r_{t+1}, q_t) \\
    &\leq \phi_{T+1}(u) - \phi_1(r_1) + \sum_{t=1}^T(\beta_{t+1} - \beta_t)(-\psi_{TE}(r_{t+1})) + \sum_{t=1}^T \frac{z_t}{\beta_t} \\
    &\leq \phi_{T+1}(u) - \phi_1(r_1) + 6\left(\sum_{t=1}^T(\beta_{t+1} - \beta_t)(-\psi_{TE}(p_{t+1})) + \sum_{t=1}^T \frac{z_t}{\beta_t}\right) \\
    &= \phi_{T+1}(u) - \phi_1(r_1) + 6\left(\sum_{t=1}^T(\beta_{t+1} - \beta_t)h_t + \sum_{t=1}^T \frac{z_t}{\beta_t}\right) \\
    &= \phi_{T+1}(u) - \phi_1(r_1) + 12\sum_{t=1}^T \frac{z_t}{\beta_t} \\
    &\leq \gamma K\ln(T) + \frac{\beta_1 (K^{1-\alpha} - 1)}{\alpha} + 12\sum_{t=1}^T \frac{z_t}{\beta_t}.
\end{align*}
It follows that
\begin{align*}
    R_T &\leq \E\left[\sum_{t=1}^T \inp{\hat{\ell}_t}{q_t - u}\right] + 3K \\
    &\lesssim \gamma K\ln(T) + \frac{\beta_1 (K^{1-\alpha} - 1)}{\alpha} + \E\left[\sum_{t=1}^T \frac{z_t}{\beta_t}\right].
\end{align*}
Applying~\eqref{eq:boundztoverbetatbyItoCOLT2024} with $S = K$, we obtain the following bounds on $\sum_{t=1}^T \frac{z_t}{\beta_t}$ in each environment.
\subsubsection*{In adversarial regime with a self-bounding constraint:}
With $J=\log_2(T)$, we have 
\begin{align*}
    \E\left[\sum_{t=1}^T \frac{z_t}{\beta_t}\right] &\lesssim \sqrt{\ln(T) \sum_{t=1}^T \E[h_tz_t]} \\
    &\leq \sqrt{\ln(T) \sum_{t=1}^T \E[h_t\E_{I_t}[z_t]]} \\
        &\leq \frac{(6d)^{2-\alpha}}{2(1-\alpha)}\E\left[h_t\left(\sum_{i=1}^K (\tilde{p}_{t,i}^{1-\alpha})(\ell_{t,i} - m_{t,i})^2\right)\right] \\
        &= \frac{(6d)^{2-\alpha}}{2(1-\alpha)}\E\left[\left(\frac{1}{\alpha}\sum_{i=1}^K p_{t,i}^\alpha - 1\right)\left(\sum_{i=1}^K (\tilde{p}_{t,i}^{1-\alpha})(\ell_{t,i} - m_{t,i})^2\right)\right] \\
        &\leq \frac{(6d)^{2-\alpha}}{2\alpha(1-\alpha)}\E\left[ \left( \sum_{i=1}^K p_{t,i}^\alpha - 1 \right)\left(  \sum_{\ell_{t,i} \neq 0}\tilde{p}_{t,i}^{1-\alpha}\right) \right],
\end{align*}
where the last inequality is from $(\ell_{t,i} - m_{t,i})^2 \leq 1$. Observe that the last bound is exactly the bound in~\eqref{eq:boundhtztWorstCaseLosses}, hence
\begin{align}
    \E\left[\sum_{t=1}^T \frac{z_t}{\beta_t}\right] &\lesssim O\left(\frac{(K-1)^{1-\alpha}K^\alpha \ln(T)}{\alpha(1-\alpha)\Delta_{\min}} + \sqrt{C\frac{(K-1)^{1-\alpha}K^\alpha \ln(T)}{\alpha(1-\alpha)\Delta_{\min}}} + \sqrt{\frac{(K-1)^{1-\alpha}K^{\alpha}}{\alpha(1-\alpha)}}\right)
    \label{eq:stochasticboundOptimisticFTRL}
\end{align}
holds for Optimistic FTRL as well.

\subsubsection*{In adversarial bandits:}
\begin{nalign}
    \sqrt{\E\left[h_{\max}\sum_{t=1}^T z_t\right]} &\lesssim \sqrt{\frac{K^{1-\alpha}-1}{\alpha(1-\alpha)}\E\left[\sum_{t=1}^T p_{t,I_t}^{-\alpha}(\ell_{t,I_t} - m_{t,I_t})^2\right]} \\
    &= \sqrt{\frac{K^{1-\alpha}-1}{\alpha(1-\alpha)}\E\left[\sum_{t=1}^T \sum_{i=1}^K p_{t,i}^{1-\alpha}(\ell_{t,i} - m_{t,i})^2\right]} \\
    &\leq \sqrt{\frac{(K^{1-\alpha}-1)}{\alpha(1-\alpha)}\E\left[\sum_{t=1}^T \sum_{i=1}^K (\ell_{t,i} - m_{t,i})^2\right]},
    \label{eq:boundhmaxsumztwithmt}
\end{nalign}
where the inequality is due to $p_{t,i}^{1-\alpha} \leq 1$.

\subsection{Proof for Theorem~\ref{thm:BOBWsqrtQLnKBound}}
\label{appendix:proofForTheoremBOBWsqrtQLnBound}
\begin{algorithm}[t]
	\KwIn{$K \geq 1, T \geq 4K, \alpha \in (0, 1), \beta_1 = \frac{8K}{1-\alpha}, \gamma = \max(6,48\sqrtfrac{\alpha}{1-\alpha}), d = 2$.}
    Initialize $\sS_i = \emptyset, \tilde{\mu}_{0,i} = 0, L_{0,i} = 0$ for $i \in [K]$ \;
	
    \For{each round $t = 1, \dots, T$}{    
        Sample $b_t \sim \mathrm{Ber}(\min(\frac{K\ln(T)}{t}, 1))$\;

        \If{$b_t = 1$}
        {
            \If{$t \leq K\ln(T)$}
            {
                Draw $I_t = t \mod K + 1$ and observe $\ell_{t, I_t}$\;
                
                Add $\ell_{t,I_t}$ to the reservoir $\sS_{I_t}$ of arm $I_t$\;
            }
            \Else{
                Draw $I_t \sim \mathrm{Unif}([K])$ and observe $\ell_{t,I_t}$\;
                
                Draw a random element by $\mathrm{Unif}(\sS_{I_t})$ and replace it by $\ell_{t,I_t}$\;
            }
            
            Update the mean estimate $\tilde{\mu}_{t, I_t}$ in the reservoir $\sS_{I_t}$~\citep{HazanAndKale11a}\;
            
            Compute $m_t = \tilde{\mu}_{t}$\;
        }
        \Else{
            Compute $m_t = m_{t-1}$ \;
            
            Compute $q_t = \argmin_{x \in \Delta_K} \inp{m_t + L_{t-1}}{x} + \beta_t\left(\frac{1}{\alpha}(1-\sum_{i=1}^K x_i^\alpha)\right) - \gamma\sum_{i=1}^K \ln(x_i)$\;
            
            Compute $p_t = \left(1 - \frac{K}{T}\right)q_t + \frac{1}{T}\1$\;
            
            Draw $I_t \sim p_t$ and observe $\ell_{t, I_t}$\;
            
            Compute loss estimate $\hat{\ell}_{t, i} = m_{t,i} + \frac{(\ell_{t,i} - m_{t,i})\I{I_t = i}}{p_{t,i}}$ \;
            
            Update $L_{t,i} = L_{t-1,i} + \hat{\ell}_{t,i}$\;		
            
            Compute $z_t =  \min\left( \frac{(6d)^{2-\alpha}}{2(1-\alpha)}\tilde{p}_{t,I_t}^{2-\alpha}(\hat{\ell}_{t,I_t} - m_{t,I_t})^2, \frac{\beta_t18d^2}{\gamma}(\ell_{t,I_t} - m_{t,I_t})^2  \right)$ \;
            
            Compute $h_t = \left(\frac{1}{\alpha}(\sum_{i=1}^K p_{t,i}^\alpha - 1)\right)$\;
            
            Compute $\beta_{t+1} = \beta_t + \frac{z_t}{\beta_t h_t}$\;
        }
	}
	\caption{SPM with Optimistic FTRL and Reservoir Sampling for losses in $[0,1]$}
	\label{algo:OFTRLReservoirSampling}
\end{algorithm}
Let $\mu_t = \frac{1}{s}\sum_{s=1}^{t}\ell_s$ and 
\begin{align*}
    Q = \sum_{t=1}^T \norm{\ell_t - \mu_T}_2^2.
\end{align*}

Using the reservoir sampling technique in~\cite{HazanAndKale11a}, we can use a prediction vector $m_t$ satisfying $\E[m_t] = \mu_t$ and $\Var[m_t] \leq \frac{Q}{t\ln(T)}$.

\subsubsection*{Regret Analysis}
Without loss of generality, assume $\ln(T) \in \sN$ (otherwise, this increases at most a constant factor in the regret bound).
For any fixed $u \in \Delta_K$, we have,
\begin{align}
    \sum_{t=1}^T \inp{\ell_t}{p_t - u} &= \sum_{t = K\ln(T) + 1}\inp{\ell_t}{p_t - u} + \sum_{t=1}^{K\ln(T)} \inp{\ell_t}{p_t - u} \\
    &\leq \sum_{t = K\ln(T) + 1}\inp{\ell_t}{p_t - u} + K\ln(T) \\
    &= \underbrace{ \sum_{t = K\ln(T) + 1}\I{b_t = 1}\inp{\ell_t}{p_t - u} }_{(A)} + \underbrace{ \sum_{t=1}^T \I{b_t = 0}\inp{\ell_t}{p_t - u} }_{(B)} + K\ln(T).
\end{align}
Recall that $b_t = 1$ indicates a reservoir sampling round where, for $t > K\ln(T)$, the sampling probability is the uniform distribution $p_t = \frac{1}{K}\1$, and $b_t = 0$ indicates an FTRL round. Next, we bound the expectation of $A$ and $B$ in the equation above. 
First, we have
\begin{align}
    \E[A] \leq \E\left[\sum_{t=K\ln(T)+1}^T \I{b_t = 1}\right] = \sum_{t=K\ln(T)+1}^T \P[b_t = 1] \leq \sum_{t=1}^T \frac{K\ln(T)}{t} \leq O(K(\ln(T))^2).
\end{align}
Next, the set of rounds with $b_t = 0$ are the Optimistic FTRL rounds; hence, we can apply~\eqref{eq:stochasticboundOptimisticFTRL} and~\eqref{eq:boundhmaxsumztwithmt}. 
In the adversarial regime with a self-bounding constraint, we have
\begin{align*}
    \E[B] &\lesssim \sqrt{\ln(T) \E[\I{b_t = 0}h_t z_t]} \\
    &\lesssim \sqrt{\frac{\ln(T)}{\alpha (1-\alpha)} \E\left[\I{b_t = 0} \left(\sum_{i=1}^K p_{t,i}^\alpha - 1\right)\left(\sum_{i=1}^K \tilde{p}_{t,i}^{1-\alpha}\right)\right]} \\
    &\leq  \sqrt{\frac{\ln(T)}{\alpha (1-\alpha)} \E\left[\left(\sum_{i=1}^K p_{t,i}^\alpha - 1\right)\left(\sum_{i=1}^K \tilde{p}_{t,i}^{1-\alpha}\right)\right]},
\end{align*}
where the last inequality is from $\left(\sum_{i=1}^K p_{t,i}^\alpha - 1\right)\left(\sum_{i=1}^K \tilde{p}_{t,i}^{1-\alpha}\right) \geq 0$ in the rounds where $b_t = 1$ (in such rounds, $p_t$ is either a one-hot vector if $t \leq K\ln{T}$ or $\frac{1}{K}\1$ if $t > K\ln{T}$). 
By~\eqref{eq:stochasticboundOptimisticFTRL}, we obtain
\begin{align*}
    \E[B] \lesssim O\left(\frac{(K-1)^{1-\alpha}K^\alpha \ln(T)}{\alpha(1-\alpha)\Delta_{\min}} + \sqrt{C\frac{(K-1)^{1-\alpha}K^\alpha \ln(T)}{\alpha(1-\alpha)\Delta_{\min}}} + \sqrt{\frac{(K-1)^{1-\alpha}K^{\alpha}}{\alpha(1-\alpha)}}\right).
\end{align*}

In the adversarial regime, we have
\begin{align*}
    \E[\sum_{t=1}^T \I{b_t = 0}z_t] &\lesssim \E\left[\sum_{t=1}^T \I{b_t = 0}\sum_{i=1}^K (\ell_{t,i} - m_{t,i})^2\right] \\
    &= \E\left[\sum_{t=1}^T \I{b_t = 0}\sum_{i=1}^K(\ell_{t,i} - \tilde{\mu}_{t,i})^2 \right] \\
    &=  \E\left[\sum_{t=1}^T \I{b_t = 0}\norm{\ell_{t} - \tilde{\mu}_{t}}_2^2 \right] \\
    &\leq \left(\E\left[\sum_{t=1}^T \I{b_t = 0}\norm{\ell_{t} - \mu_{t}}_2^2 \right] + \E\left[\sum_{t=1}^T \I{b_t = 0}\norm{\tilde{\mu}_t - \mu_{t}}_2^2 \right]\right) \\
    &\leq \left(\E\left[\sum_{t=1}^T \I{b_t = 0}\norm{\ell_{t} - \mu_{T}}_2^2 \right] + \E\left[\sum_{t=1}^T \I{b_t = 0}\norm{\tilde{\mu}_t - \mu_{t}}_2^2 \right]\right) \\
    &\leq \left(Q + \sum_{t=1}^T \frac{Q}{t\ln(T)}\right) \leq 3 Q,
\end{align*}
 where the first inequality is triangle inequality, the second inequality is $\E\left[\sum_{t=1}^T \norm{\ell_{t} - \mu_{t}}_2^2 \right] \leq \E\left[\sum_{t=1}^T \norm{\ell_{t} - \mu_{T}}_2^2 \right]$ by Lemma 10 in~\cite{HazanAndKale11a}, the third inequality is by Lemma 11 in~\cite{HazanAndKale11a}, and the last inequality is due to $\sum_{t=1}^T \frac{1}{t} \leq \ln(T) + 1$.
Overall, the regret for adversarial bandits is
\begin{align*}
    R_T \lesssim \sqrtfrac{(K^{1-\alpha}-1) Q}{\alpha(1-\alpha)}.
\end{align*}

\section{Proofs for Section~\ref{sec:CoordinateWiseSPM}}
\label{appendix:ProofsForCoordinateWiseSPM}
We have
\begin{align*}
    \E[\sum_{t=1}^T \inp{\ell_t}{p_t} - u] &= \E[\sum_{t=1}^T \inp{\ell_t}{q_t - u + \frac{\1 - Kq_t}{T}}] \\
    &\leq \E[\sum_{t=1}^T \inp{\ell_t}{q_t - u}] + K \\
    &= \E[\sum_{t=1}^T \inp{\hat{\ell}_t}{q_t - u}] + K.
\end{align*}
Furthermore, let
\begin{align*}
    r_{t+1} = \argmin_{x \in \Delta_K} \inp{L_{t-1} + \hat{\ell}_t}{x} + \sum_{i=1}^K \beta_{t,i}\left(\frac{1}{\alpha}(-x_i^\alpha) + (1-x_i)\ln(1-x_i) + x_i \right) - \gamma\sum_{i=1}^K \ln(x_i).
\end{align*}
Then,
\begin{align*}
    &\sum_{t=1}^T \inp{\hat{\ell}_t}{q_t - u} \\
    &\leq \phi_{T+1}(u) - \phi_1(r_1) + \sum_{t=1}^T \sum_{i=1}^K  (\beta_{t+1, i} - \beta_{t,i})\left( \frac{p_{t+1,i}^\alpha}{\alpha} + (p_{t+1,i}-1)\ln(1-p_{t+1,i}) - p_{t+1, i} \right) + \sum_{t=1}^T \frac{z_{t,I_t}}{\beta_{t,I_t}} \\
    &\leq \phi_{T+1}(u) - \phi_1(r_1) + \sum_{t=1}^T \frac{2}{\alpha} (\beta_{t + 1,I_t} - \beta_{t,I_t}) p_{t+1,I_t}^\alpha + \sum_{t=1}^T \frac{z_{t,I_t}}{\beta_{t,I_t}} \\
    &\leq \phi_{T+1}(u) - \phi_1(r_1) + \sum_{t=1}^T \frac{12}{\alpha} (\beta_{t + 1,I_t} - \beta_{t,I_t}) p_{t,I_t}^\alpha + \sum_{t=1}^T \frac{z_{t,I_t}}{\beta_{t,I_t}} \\
    &= \phi_{T+1}(u) - \phi_1(r_1) + \sum_{t=1}^T 12(\beta_{t + 1,I_t} - \beta_{t,I_t}) h_{t,I_t} + \sum_{t=1}^T \frac{z_{t,I_t}}{\beta_{t,I_t}} \\
    &= \phi_{T+1}(u) - \phi_1(r_1) + 13\sum_{t=1}^T \frac{z_{t,I_t}}{\beta_{t,I_t}} \\
    &= \phi_{T+1}(u) - \phi_1(r_1) + 13\sum_{i=1}^K \sum_{t=1}^T  \frac{z_{t,i}}{\beta_{t,i}},
\end{align*}
where the first inequality is from $p_{t+1,i} \geq 0$ and Lemma~\ref{lemma:xalphadominatesxminus1lnxminus1}, the second inequality is from $p_{t+1, I_t}^\alpha \leq (6 p_{t,I_t})^\alpha \leq 6p^\alpha_{t,I_t}$ and the last equality is $z_{t,i} = 0$ for all $i \neq I_t$.

From the previous section, we have for all $i \in [K]$,
\begin{align}
    \sum_{t=1}^T \frac{z_{t,i}}{\beta_{t,i}} &\lesssim \min\left\{ \sqrt{\E\left[\ln(T)\sum_{t=1}^T h_{t,i} z_{t,i}\right]} + \sqrt{\frac{1}{T}\E\left[h_{i, \max}\sum_{t=1}^T z_{t,i}\right]}, \sqrt{\E\left[h_{i, \max}\sum_{t=1}^T z_{t,i}\right]} \right\}.\
    \label{eq:coSPMregretboundMinofTwoA}
\end{align}
First, we have $h_{i, \max} = \max_{t} h_{t,i} = \frac{1}{\alpha} \max_{t} p_{t,i}^\alpha \leq \frac{1}{\alpha}$. 
In addition, $\E_{I_t}[z_{t,i}] \leq \E_{I_t}[\I{I_t = i}p_{t,i}^{-\alpha}(\ell_{t,i} - m_{t,i})^2] = p_{t,i}^{1-\alpha}(\ell_{t,i} - m_{t,i})^2 \leq 1$. 
Therefore, the sum $\frac{1}{T}\E\left[h_{i, \max}\sum_{t=1}^T z_{t,i}\right]$ is bounded by $\frac{1}{\alpha}$. We can simplify~\eqref{eq:coSPMregretboundMinofTwoA} by
\begin{align}
    \sum_{t=1}^T \frac{z_{t,i}}{\beta_{t,i}} &\lesssim \min\left\{\sqrt{\ln(T)\E\left[\sum_{t=1}^T h_{t,i} z_{t,i}\right]} + \sqrtfrac{1}{\alpha}, \sqrt{\E\left[\sum_{t=1}^T z_{t,i}\right]} \right\}.
    \label{eq:coSPMregretboundMinofTwoB}
\end{align}

\subsection*{A Bound for Stochastic Bandits from $\sqrt{\ln(T)\E\left[\sum_{t=1}^T h_{t,i} z_{t,i}\right]}$}

We have
\begin{align*}
    h_{t,i} z_{t,i} &\lesssim \I{I_t = i} (\ell_{t,i} - m_{t,i})^2 p_{t,i}^\alpha \min\left( \frac{(6d)^{2-\alpha}}{2(1-\alpha)}\min\left\{p_{t,I_t}^{-\alpha}, \frac{1 - p_{t,I_t}}{p_{t,I_t}^2}\right\}\right).
\end{align*}
Therefore, by Lemma~\ref{lemma:CoordinateWiseSPMztiIsGoodForStochasticBandits},
\begin{align*}
    \E_{I_t}[h_{t,i} z_{t,i}] \lesssim \frac{1}{1-\alpha}\tilde{p}_{t,i}(\ell_{t,i} - m_{t,i})^2.
\end{align*}

Denote $P_{i} = \E[\sum_{t=1}^T \I{I_t = i}]$. Bounding $(\ell_{t,i} - m_{t,i})^2 \leq 1$ for any $m_{t} \in [0,1]^K$, we obtain 
\begin{align*}
    \mathrm{Reg}_T &\lesssim \sum_{i=1}^K \sqrt{\E\left[\ln(T)\sum_{t=1}^T h_{t,i} z_{t,i}\right]} \\
    &\lesssim \sqrtfrac{\ln(T)}{\alpha(1-\alpha)} \left(\sum_{i \neq i^*}\sqrt{P_i} + \sqrt{(\sum_{i \neq i^*} P_i)}\right) + K\ln(T) \\
    &\leq \sqrtfrac{\ln(T)}{\alpha(1-\alpha)} \left(\sum_{i \neq i^*}\sqrt{P_i} + \frac{1}{\sqrt{K-1}}\sum_{i \neq i^*}\sqrt{P_i}\right) + K\ln(T) \\
    &\lesssim \sqrtfrac{\ln(T)}{\alpha(1-\alpha)} \left(\sum_{i \neq i^*} \sqrt{P_i}  \right) + K\ln(T).
\end{align*}
Similar to~\cite{ItoCOLT2022aVariance}, by using $\mathrm{Reg}_T = 2\mathrm{Reg}_T  - \mathrm{Reg}_T$ and $2\sqrt{ax} - bx \leq \frac{a}{b}$ for $a = \frac{\ln(T)}{\alpha(1-\alpha)}, b = \Delta_i$ and $x = P_i$, we obtain (note that we set $\alpha = \frac{1}{2}$)
\begin{align*}
    \mathrm{Reg}_T \lesssim \frac{1}{\alpha (1-\alpha)}\sum_{i \neq i^*}\frac{ \ln(T) }{\Delta_i}.
\end{align*}

\subsection*{Bounds for Adversarial Bandits}
Fix $i \in [K]$. Since $\tilde{p}_{t,i} \leq p_{t,i}$, we have $h_{t,i}z_{t,i} \lesssim \I{I_t = i}(\ell_{t,i} - m_{t,i})^2$. 
Therefore, by setting $m_{t,i}$ to be the output of an online learning algorithm with fully-observable squared loss as in~\cite{ItoCOLT2022aVariance}, i.e.,
\begin{align*}
    m_{t,i} = \frac{1}{1 + \sum_{s=1}^{t-1}\I{I_s = i}}\left(\frac{1}{2} + \sum_{s=1}^{t-1} \I{I_s = i} \ell_{t,i}\right)
\end{align*}
and then applying their Lemma 3, we obtain for any fixed $m^{*} \in [0,1]^K$,
\begin{align*}
    \sum_{t=1}^T \I{I_t = i}(\ell_{t,i} - m_{t,i})^2 \lesssim \sum_{t=1}^T \I{I_t = i}(\ell_{t,i} - m^{*}_{i})^2 + \ln(1 + \sum_{t=1}^T \I{I_t = i}).
\end{align*}
As already shown in~\cite{ItoCOLT2022aVariance}, for each appropriately chosen $m^{*}$, we would recover the data-dependent bounds of order $\sqrt{KQ_\infty \ln(T)}$ (with $m^* \in \argmin_{\bar{\ell} \in \R^K}\sum_{t=1}^T \norm{\ell_t - \bar{\ell}}_2^2$), $\sqrt{KL^*\ln(T)}$ (with $m^* = \vect{0}$) and $\sqrt{K(T-L^*)\ln(T)}$ (with $m^* = \1$).

On the other hand, from the quantity $\sqrt{\E\left[\sum_{t=1}^T z_{t,i}\right]}$ and Jensen's inequality, we obtain 
\begin{align*}
    \mathrm{Reg}_T &\lesssim \sqrt{K\E\left[\sum_{t=1}^T \sum_{i=1}^K z_{t,i}\right]} \\
    &\lesssim \sqrt{K\E\left[\sum_{t=1}^T \sum_{i=1}^K p_{t,i}^{1-\alpha}(\ell_{t,i} - m_{t,i})^2\right]} \\
    &\lesssim \sqrt{K\E\left[\sum_{t=1}^T \sum_{i=1}^K p_{t,i}^{1-\alpha}\right]} \\
    &\leq \sqrt{K\E\left[\sum_{t=1}^T K^\alpha\right]} \\
    &= \sqrt{K^{1+\alpha}T} = K^{\frac{1}{4}}\sqrt{KT} \qquad (\alpha = 1/2),
\end{align*}
which grows with $\sqrt{T}$ in the worst-case.

\subsection{Stability Proofs}
\label{sec:CoordinateWiseSPMStabilityProofs}
In this section, we define the following function
\begin{align}
    g(x) = x_i^{\alpha-1} + \ln(1-x).
\end{align}
Note that $g$ is decreasing.
In addition, let $d_f(y, x) = f(y) - f(x) - f'(x)(y-x)$ denote the Bregman divergence associated with a one-dimensional strictly convex function $f: \R \to \R$. Note that $d_f(y,x) \geq 0$ for all $x, y \in \R$.

\begin{lemma}
For any $L \in \R^K, \beta \in R^K_{+}, \gamma > 0$ and $d \geq 2$, let 
\begin{align*}
    x &= \argmin_{p \in \Delta_K} \inp{L}{p} + \sum_{i=1}^K \beta_i \left(\frac{-p_i^\alpha}{\alpha} + (1-p_i)\ln(1-p_i) + p_i\right) - \gamma \sum_{i=1}^K \ln(p_i) \\
    y &= \argmin_{p \in \Delta_K} \inp{L}{p} + \sum_{i=1}^K \beta'_i \left(\frac{-p_i^\alpha}{\alpha} + (1-p_i)\ln(1-p_i) + p_i\right) - \gamma \sum_{i=1}^K \ln(p_i).
\end{align*}
If $0 \leq \beta'_1 - \beta_1 \leq \left(1-\frac{1}{d}\right)\gamma x_1^{-\alpha}$ and $\beta'_i = \beta_i$ for $i > 1$, then $y_1 \leq dx_1$.
\label{lemma:CoordinateWiseSPMSameLossDiffBeta}
\end{lemma}
\begin{proof}
    If $dx_1 \geq 1$ then $y_1 \leq 1 \leq dx_1$ trivially. Hence, we assume $dx_1 \leq 1$. By the Lagrange multiplier method, we have for $i = 2, \dots, K$ and some $\lambda, \lambda' \in \R$,
    \begin{align*}
        & L_i - \beta_i(x_i^{\alpha-1} + \ln(1-x_i)) - \frac{\gamma}{x_i} = \lambda, \\
        & L_i - \beta_i(x_i^{\alpha-1} + \ln(1-y_i)) - \frac{\gamma}{y_i} = \lambda'.
    \end{align*}
    Similary, for $i = 1$, we have
    \begin{align*}
        & L_1 - \beta_1(x_1^{\alpha-1} + \ln(1-x_1)) - \frac{\gamma}{x_1} = \lambda, \\
        & L_1 - \beta'_1(x_1^{\alpha-1} + \ln(1-y_1)) - \frac{\gamma}{y_1} = \lambda'.
    \end{align*}
    Taking $Z = \lambda' - \lambda$ over all $K$ pairs of equations, we obtain
    \begin{align}
        & \beta_i (g(x_i) - g(y_i)) + \gamma\left(\frac{1}{x_i} - \frac{1}{y_i}\right) = Z \\
        & \beta_1 g(x_1) - \beta'_1 g(y_1) + \gamma \left(\frac{1}{x_1} - \frac{1}{y_1}\right) = Z.
        \label{eq:CoordinateWiseSPMLagrangeDiffSameLossDiffBeta}
    \end{align}
    
    If $Z \geq 0$, then since $\beta_i > 0$ and both $g(x)$ and $\frac{\gamma}{x}$ are decreasing, we have $y_i \geq x_i$ for all $i \neq 1$. This straightforwardly implies that $y_1 \leq x_1$. Thus, we focus on the case $Z < 0$. In this case, we have $y_i < x_i$ for all $i \neq 1$ and $y_1 > x_1$. We consider two cases:
    \begin{itemize}
        \item If $g(x_1) \leq 0$: from $y_1 \geq x_1$, we have $g(y_1) \leq g(x_1) \leq 0$. Hence, from $0 < \beta_1 \leq \beta'_1$, we obtain
        \begin{align*}
            \beta_1'g(y_1) \leq \beta_1 g(y_1) \leq \beta_1 g(x_1).
        \end{align*}
        This implies that $0 > Z = \beta_1 g(x_1) - \beta'_1 g(y_1) +  \gamma \left(\frac{1}{x_1} - \frac{1}{y_1}\right) \geq 0$, a contradiction. 
        \item If $g(x_1) >0$: in this case,~\eqref{eq:CoordinateWiseSPMLagrangeDiffSameLossDiffBeta} and $Z < 0$ implies $\beta'_1 g(y_1) = \beta_1 g(x_1) + \gamma \left(\frac{1}{x_1} - \frac{1}{y_1}\right) - Z > 0$. Furthermore, by re-arranging, we obtain
        \begin{align*}
            \beta'_1 g(y_1) + \frac{\gamma}{y_1} &\geq \beta_1 g(x_1) + \frac{\gamma}{x_1} \\
            &\geq \left(\beta'_1 - \left(1 - \frac{1}{d}\right)\gamma x_1^{-\alpha}\right)(x_1^{\alpha - 1} + \ln(1-x_1)) + \frac{\gamma}{x_1} \\
            &= \beta'_1(x_1^{\alpha - 1} + \ln(1-x_1)) - \left(1 - \frac{1}{d}\right)\gamma x^{-1} - \left(1 - \frac{1}{d}\right)\gamma x_1^{-\alpha} \ln(1-x_1) + \frac{\gamma}{x_1} \\
            &= \beta'_1(x_1^{\alpha - 1} + \ln(1-x_1)) + \frac{\gamma}{dx_1} - \left(1 - \frac{1}{d}\right)\gamma x_1^{-\alpha} \ln(1-x_1) \\
            &\geq \beta'_1(x_1^{\alpha - 1} + \ln(1-x_1)) + \frac{\gamma}{dx_1} \\
            &\geq \beta'_1((dx_1)^{\alpha - 1} + \ln(1-dx_1)) + \frac{\gamma}{dx_1} \\
            &= \beta'_1 g(dx_1) + \frac{\gamma}{dx_1},
        \end{align*}
        where the second inequality is from $\beta_1 \geq \beta'_1 - \left(1 - \frac{1}{d}\right)\gamma x_1^{-\alpha}$, the third inequality is due to $\ln(1-x_1) < 0$ and the last inequality is from $d \geq 2 > 1$. Since $\beta g(x) + \frac{\gamma}{x}$ is decreasing for all $\beta > 0, \gamma > 0$, we conclude that $y_1 \leq dx_1$.
    \end{itemize}
\end{proof}

\begin{lemma}
    For any $L \in \R^K, \beta \in R^K_{+}, \gamma > 0$ and $h \in [-1,1]$, let
    \begin{align*}
        x &= \argmin_{p \in \Delta_K} \inp{L}{p} + \sum_{i=1}^K \beta_i \left(\frac{-p_i^\alpha}{\alpha} + (1-p_i)\ln(1-p_i) + p_i\right) - \gamma \sum_{i=1}^K \ln(p_i) \\
        y &= \argmin_{p \in \Delta_K} \inp{L + \frac{h}{x'_1}}{p} + \sum_{i=1}^K \beta_i \left(\frac{-p_i^\alpha}{\alpha} + (1-p_i)\ln(1-p_i) + p_i\right) - \gamma \sum_{i=1}^K \ln(p_i),
    \end{align*}
    where $4x'_1 \geq x_1$. If $\gamma \geq 6$ then $y_1 \leq 3x_1$.
    \label{lemma:CoordinateWiseSPMSameBetaDiffLoss}
\end{lemma}
\begin{proof}
    Using the Lagrange multiplier method, we have the following equalities that hold for some $Z \in \R$,
    \begin{align}
        \beta_1\left( g(y_1) - g(x_1)\right) + \gamma\left(\frac{1}{y_1} - \frac{1}{x_1}\right) = Z + \frac{h}{x'_1}
        \label{eq:CoordinateWiseSPMg1diff}
    \end{align}
    and for all $i \neq 1$,
    \begin{align}
        \beta_i\left( g(y_i)- g(x_i)\right) + \gamma\left(\frac{1}{y_i} - \frac{1}{x_i}\right) = Z.
        \label{eq:CoordinateWiseSPMgidiff}
    \end{align}
    First, we show that $Z$ and $y_1 - x_1$ has the opposite sign to $h$. 
    We consider two cases:
    \begin{itemize}
        \item If $Z \geq 0$ then from~\eqref{eq:CoordinateWiseSPMgidiff} and the monotonic decreasing property of $\beta g(x) + \frac{\gamma}{x}$, we have $y_i \leq x_i$ and this leads to $y_1 \geq x_1$. Combining $y_1 \geq x_1$ and~\eqref{eq:CoordinateWiseSPMg1diff}, we have $Z + \frac{h}{x'_1} \leq 0$. Since $Z \geq 0$, this implies $h \leq 0$.
        \item If $Z \leq 0$ then by the same argument, we have $y_i \geq x_i$ and $y_1 \leq x_1$. Therefore, $Z + \frac{h}{x'_1} \geq 0$. Due to $Z \leq 0$, we must have $h \geq 0$.
    \end{itemize}
    In both cases, we have $Zh \leq 0$ and $Z(y_1 - x_1) \geq 0$. It follows that if $h \geq 0$ then we have $y_1 \leq x_1 \leq 3x_1$. If $h < 0$ then $y_1 \geq x_1$, and by rearranging~\eqref{eq:CoordinateWiseSPMg1diff}, we obtain
    \begin{align*}
        \frac{4}{x_1} \geq -\frac{h}{x'_1} &= \underbrace{Z}_{\geq 0} + \underbrace{\gamma\left(\frac{1}{x_1} - \frac{1}{y_1}\right)}_{\geq 0} + \underbrace{\beta(g(x_1) - g(y_1))}_{\geq 0} \\
        &\geq \gamma\left(\frac{1}{x_1} - \frac{1}{y_1}\right) \\
        &\geq 6\left(\frac{1}{x_1} - \frac{1}{y_1}\right),
    \end{align*}
    where the last inequality is due to $\gamma \geq 6$. This implies that $\frac{3}{y_1} \geq \frac{1}{x_1}$, thus $y_1 \leq 3x_1$.
\end{proof}
By combining Lemma~\ref{lemma:CoordinateWiseSPMSameLossDiffBeta} and Lemma~\ref{lemma:CoordinateWiseSPMSameBetaDiffLoss}, we obtain the following corollary. The proof of this corollary is nearly identical to that of Corollary~\ref{corollary:stable}.
\begin{corollary}
    For any $t \in [T]$, Algorithm~\ref{algo:CoordinateWiseSPM} guarantees that
    \begin{align*}
        r_{t+1, I_t} \leq 3d q_{t,I_t}.
    \end{align*}
    \label{corollary:rtplus1Is3dqtIt}
\end{corollary}
\begin{lemma}
    For all $t \in [T]$, Algorithm~\ref{algo:CoordinateWiseSPM} guarantees that 
    \begin{align*}
        \inp{\hat{\ell}_t - m_t}{q_t - r_{t+1}} - D_t(r_{t+1}, q_t) \leq \frac{z_{t,I_t}}{\beta_{t,I_t}},
    \end{align*}
    where 
    \begin{align*}
        z_{t,I_t} = (\ell_{t,I_t} - m_{t,I_t})^2 \min\left\{ \frac{(6d)^{2-\alpha}}{2(1-\alpha)}\min\left\{p_{t,I_t}^{-\alpha}, \frac{1 - p_{t,I_t}}{p_{t,I_t}^2}\right\}, \frac{\beta_{t,I_t} 18d^2}{\gamma}\right\}.
    \end{align*}
\end{lemma}
\begin{proof}
    Let 
    \begin{align*}
        f_1(x) &= \frac{-x^\alpha}{\alpha}, \\
        f_2(x) &= (1-x)\ln(1-x) + x, \\
        f_3(x) &= -\ln(x).
    \end{align*}
    Since $\phi_t(x) = \sum_{i=1}^K \left(\beta_{t,i}(f_1(x_i) + f_2(x_i)) + \gamma f_3(x_i)\right)$, we have
    \begin{align*}
        D_t(r_{t+1}, q_t) &= \sum_{i=1}^K \left( \beta_{t,i}(d_{f_1}(r_{t+1,i}, q_{t,i}) + d_{f_2}(r_{t+1,i}, q_{t,i})) + \gamma d_{f_3}(r_{t+1,i}, q_{t,i}) \right) \\
        &\geq \beta_{t,I_t}(d_{f_1}(r_{t+1,I_t}, q_{t,I_t}) + d_{f_2}(r_{t+1,I_t}, q_{t,I_t})) + \gamma d_{f_3}(r_{t+1,I_t}, q_{t,I_t}).
    \end{align*}
    Furthermore, as $\hat{\ell}_{t,i} - m_{t,i} = 0$ for all $i \neq I_t$, we have
    \begin{align*}
        \inp{\hat{\ell}_t - m_t}{q_t - r_{t+1}} - D_t(r_{t+1}, q_t) &= \frac{\ell_{t,I_t} - m_{t,I_t}}{p_{t,I_t}}(r_{t+1, I_t} - q_{t,I_t}) - D_t(r_{t+1}, q_t) \\
        &\leq \min(A, B, C),
    \end{align*}
    where 
    \begin{align*}
        A &= \frac{\ell_{t,I_t} - m_{t,I_t}}{p_{t,I_t}}(r_{t+1, I_t}- q_{t,I_t}) - \beta_{t,I_t}d_{f_1}(r_{t+1,I_t}, q_{t,I_t}), \\
        B &= \frac{\ell_{t,I_t} - m_{t,I_t}}{p_{t,I_t}}(r_{t+1, I_t}- q_{t,I_t}) - \beta_{t,I_t} d_{f_2}(r_{t+1,I_t}, q_{t,I_t}), \\
        C &= \frac{\ell_{t,I_t} - m_{t,I_t}}{p_{t,I_t}}(r_{t+1, I_t}- q_{t,I_t}) - \gamma d_{f_3}(r_{t+1,I_t}, q_{t,I_t}).
    \end{align*}
    Here, we used $x - (a+b+c) \leq \min(x - a, x - b, x - c)$ for $a, b, c \geq 0$. 
    
    Note that $\ell_{t,I_t} - m_{t,I_t} \in [-1, 1]$ for $0 \leq \ell_{t,I_t}, m_{t,I_t} \leq 1$. 
    By Corollary~\ref{corollary:rtplus1Is3dqtIt} and the fact that $r_{t+1} \in \Delta_K$, we have $0 \leq r_{t+1,I_t} \leq 3dq_{t,I_t}$. Combining this with Lemma~\ref{lemma:CoordinateWiseSPMStabilityBetaTerms}, we have 
    \begin{nalign}
        \min(A, B) &\leq \frac{(3d)^{2-\alpha}(\ell_{t,I_t} - m_{t,I_t})^2}{\beta_{t,I_t} p_{t,I_t}^2} \min\left\{ \frac{q_{t,I_t}^{2-\alpha}}{2(1-\alpha)}, 1-q_{t,I_t} \right\} \\
        &\leq \frac{(6d)^{2-\alpha}(\ell_{t,I_t} - m_{t,I_t})^2}{\beta_{t,I_t} p_{t,I_t}^2} \min\left\{ \frac{p_{t,I_t}^{2-\alpha}}{(1-\alpha)}, 2(1-p_{t,I_t}) \right\} \\
        &\leq \frac{(6d)^{2-\alpha}(\ell_{t,I_t} - m_{t,I_t})^2}{2(1-\alpha)\beta_{t,I_t}}\min\left\{ p_{t,I_t}^{-\alpha}, \frac{(1-p_{t,I_t})}{p_{t,I_t}^2} \right\},
        \label{eq:minAB}
    \end{nalign}
    where the second inequality is due to $q_{t,I_t} \leq 2p_{t,I_t}$ and $1-q_{t,I_t} \leq 2(1-p_{t,I_t})$ by Lemma~\ref{lemma:1minusqByTwoTimes1minusp} and the last inequality is $1-\alpha \leq 1$.

    The second-order derivative of $\gamma f_3(x)$ is $\frac{\gamma}{x^2}$. Therefore, by Lemma~\ref{lemma:onedimensionalLocalNorm}, we have 
    \begin{align}
        C \leq \frac{(\ell_{t,I_t} - m_{t,I_t})^2 v^2}{2 \gamma p_{t,I_t}^2} \leq \frac{(\ell_{t,I_t} - m_{t,I_t})^2 18d^2}{\gamma},
        \label{eq:minC}
    \end{align}
    where $v \leq 3dq_{t,I_t} \leq 6dp_{t,I_t}$ is a point between $r_{t+1,I_t}$ and $q_{t,I_t}$.

    Combining~\eqref{eq:minAB} and~\eqref{eq:minC} implies the desired bound.
\end{proof}

\subsection{Technical Lemmas}
\label{sec:CoordinateWiseSPMTechnicalLemmas}
\begin{lemma}
    For any $\alpha \in [0,1]$ and $x \in [0,1]$, 
    \begin{align}
        x^{\alpha} \geq (x-1)\ln(1-x).
    \end{align}
    \label{lemma:xalphadominatesxminus1lnxminus1}
\end{lemma}
\begin{proof}
    For $\alpha \in [0,1]$, we have $x^\alpha \geq x$. Let
    \begin{align}
        h(x) = x + (1-x)\ln(1-x).
    \end{align}
    Its derivative is $h'(x) = 1 - \ln(1-x) - 1 = -\ln(1-x) \geq 0$ for all $x \in [0,1]$. Therefore, $h(x) \geq h(0) = 0$ for all $x \in [0,1]$. We conclude that
    \begin{align}
        x^{\alpha} - (x-1)\ln(1-x) \geq x - (x-1)\ln(1-x) = h(x) \geq 0.
    \end{align}
\end{proof}
Recall that $d_f(y, x) = f(y) - f(x) - f'(x)(y-x)$ denote the Bregman divergence associated with a one-dimensional strictly convex function $f: \R \to \R$. The following lemma is essentially a one-dimensional local-norm analysis of FTRL, whose proof can be found in standard literature. We provide a proof here for completeness.

\begin{lemma}
    For any $a \in \R, x, y \in (0,1)$ and strictly convex function $f: (0,1) \to \R$, we have
    \begin{align*}
        a(x - y) - d_f(y, x) \leq \frac{1}{2}\frac{a^2}{f''(z)}
    \end{align*}
    for some $z$ between $x$ and $y$.
    \label{lemma:onedimensionalLocalNorm}
\end{lemma}
\begin{proof}
    The inequality trivally holds when $x = y$, hence we assume $x \neq y$. By Taylor's theorem, we have $d_f(y,x) = \frac{f''(z)}{2}(x-y)^2$ for some $z$ between $x$ and $y$. Note that the strict convexity of $f$ implies $f''(z) > 0$. We have
    \begin{align*}
        a(x-y) - d_f(x,y) &= a(x-y) - \frac{f''(z)}{2}(x-y)^2 \\
        &= \frac{1}{2}\left( -\left((x-y)\sqrt{f''(z)} - \frac{a}{\sqrt{f''(z)}}\right)^2 + \frac{a^2}{f''(z)} \right) \leq \frac{a^2}{2f''(z)}.
    \end{align*}
\end{proof}
Next, the following two lemma establish the foundation for choosing $z_{t,i}$ in Algorithm~\ref{algo:CoordinateWiseSPM}.
\begin{lemma}
    Let $\alpha \in (0,1), \beta > \frac{4}{1-\alpha}$ be fixed and 
    \begin{align*}
        f_1(x) &= \left(\frac{-x^\alpha}{\alpha}\right), \\
        f_2(x) &= ((1-x)\ln(1-x) + x).
    \end{align*}
    For any $d \geq 1, h \in [-1, 1], q \in (0,1]$ and $p \geq \frac{q}{2}$, we have
    \begin{align*}
        &\min\left\{ \max_{0 \leq u \leq dq}\left(\frac{h}{p}(q - u) - \beta d_{f_1}(u, q)\right), \max_{u \in \R}\left(\frac{h}{p}(q - u) - \beta d_{f_2}(u, q)\right) \right\} \\
        &\leq \frac{d^{2-\alpha}h^2}{\beta p^2} \min\left\{ \frac{q^{2-\alpha}}{2(1-\alpha)}, 1-q \right\}.
    \end{align*}
    \label{lemma:CoordinateWiseSPMStabilityBetaTerms}
\end{lemma}
\begin{proof}
    First, we bound $\max_{0 \leq u \leq dq}\left(\frac{h}{p}(q - u) - \beta d_{f_1}(u, q)\right)$. For any $u \geq 0$, by Lemma~\ref{lemma:onedimensionalLocalNorm}, we have for some $v$ between $q$ and $u$:
    \begin{align}
        \left(\frac{h}{p}(q - u) - \beta d_{f_1}(u, q)\right) &\leq \frac{1}{2}\frac{h^2}{p^2}\frac{v^{2-\alpha}}{\beta(1-\alpha)} \\
        &\leq \frac{h^2}{\beta p^2} \frac{d^{2-\alpha} q^{2-\alpha}}{2(1-\alpha)},        
    \end{align}
    where we used the fact that the second-order derivative of $\beta f_1(v)$ is $\beta (1-\alpha) v^{\alpha - 2}$ and $v \leq \max(q, u) \leq dq$. It follows that
    \begin{align}
        \max_{0 \leq u \leq dq}\left(\frac{h}{p}(q - u) - \beta d_{f_1}(u, q)\right) \leq \frac{h^2}{\beta p^2} \frac{d^{2-\alpha} q^{2-\alpha}}{2(1-\alpha)}.
        \label{eq:boundStabilityByLocalNormCoordinateWiseSPM1}
    \end{align}

    Next, using Lemma 5 in~\cite{ItoCOLT2022aVariance}, we have
    \begin{align*}
        \max_{u \in \R}\left(\frac{h}{p}(q - u) - \beta d_{f_2}(u, q)\right) &= \beta \max_{u \in \R}\left(\frac{h}{\beta p}(q - u) - d_{f_2}(u, q)\right) \\
        &= \beta(1-q)\left(\exp\left(\frac{h^2}{\beta^2 p^2}\right) - \frac{h}{\beta p} - 1 \right).
    \end{align*}
    We consider two cases:
    \begin{itemize}
        \item If $\beta p \geq 1$: in this case, we have $\frac{h}{\beta p} \leq 1$ for any $h \in [-1,1]$. From the inequality $\exp(a) - a - 1 \leq a^2$ for $a \leq 1$, we have $\exp\left(\frac{h^2}{\beta^2 p^2}\right) - \frac{h}{\beta p} - 1 \leq \frac{h^2}{\beta^2 p^2}$. Therefore,
        \begin{align*}
            \max_{u \in \R}\left(\frac{h}{p}(q - u) - \beta d_{f_2}(u, q)\right) \leq (1-q)\frac{h^2}{\beta p^2}.
        \end{align*}
        This implies that
        \begin{align*}
            &\min\left\{ \max_{0 \leq u \leq dq}\left(\frac{h}{p}(q - u) - \beta d_{f_1}(u, q)\right), \max_{u \in \R}\left(\frac{h}{p}(q - u) - \beta d_{f_2}(u, q)\right) \right\} \\
            &\leq \frac{(d)^{2-\alpha}h^2}{\beta p^2} \min\left\{ \frac{q^{2-\alpha}}{2(1-\alpha)}, 1-q \right\}.
        \end{align*}

        \item If $\beta p < 1$: in this case, we have $q \leq 2p \leq \frac{2}{\beta} \leq \frac{1-\alpha}{2}$. This implies that $\frac{q}{1-\alpha} \leq \frac{1}{2}$ and also $q \leq \frac{1}{2}$.
        Combining this with $q^{1-\alpha} \leq 1$, we obtain
        \begin{align*}
            \frac{q^{2-\alpha}}{2(1-\alpha)} \leq \frac{q}{2(1-\alpha)} \leq \frac{1}{4} \leq 1-q.
        \end{align*}
        It follows that by~\eqref{eq:boundStabilityByLocalNormCoordinateWiseSPM1},
        \begin{align*}
            \max_{0 \leq u \leq dq}\left(\frac{h}{p}(q - u) - \beta d_{f_1}(u, q)\right) &\leq \frac{h^2}{\beta p^2} \frac{(d)^{2-\alpha} q^{2-\alpha}}{2(1-\alpha)} \\
            &= \frac{(d)^{2-\alpha} h^2}{\beta p^2} \min\left\{ \frac{ q^{2-\alpha}}{2(1-\alpha)}, 1-q \right\}.
        \end{align*}
    \end{itemize}
    In both cases, we have
    \begin{align*}
        &\min\left\{ \max_{0 \leq u \leq dq}\left(\frac{h}{p}(q - u) - \beta d_{f_1}(u, q)\right), \max_{u \in \R}\left(\frac{h}{p}(q - u) - \beta d_{f_2}(u, q)\right) \right\} \\
        &\leq \frac{(d)^{2-\alpha}h^2}{\beta p^2} \min\left\{ \frac{q^{2-\alpha}}{2(1-\alpha)}, 1-q \right\}.
    \end{align*}
\end{proof}

\begin{lemma}
    For any $K \geq 3, T \geq 4K$ and $q \in [0,1]$, let 
    \begin{align*}
        p = \left(1 - \frac{K}{T}\right)q + \frac{1}{T}.
    \end{align*}
    Then, we have $1-q \leq 2(1-p)$. 
    \label{lemma:1minusqByTwoTimes1minusp}
\end{lemma}
\begin{proof}
    The desired inequality is equivalent to $2p - q \leq 1$. By the definition of $p$, we have
    \begin{align*}
        2p - q &= 2\left(1 - \frac{K}{T}\right)q + \frac{2}{T} - q \\
        &= \left(1 - \frac{2K}{T}\right)q + \frac{2}{T} \\
        &\leq \left(1 - \frac{2K}{T}\right) + \frac{2}{T} \\
        &\leq 1.
    \end{align*}
\end{proof}

\begin{lemma}
    Fix an index $i \in [K]$ and let $p \in \Delta_K$ be an arbitrary vector in $\Delta_K$. Let $\alpha \in (0,1)$ be a constant and $I \sim p$ be a random variable distributed according to $p$. We have
    \begin{align*}
        \E_{I \sim p}\left[\I{I = i}p_i^\alpha \min\left( p_i^{-\alpha}, \frac{1-p_i}{p_i^2} \right)\right] \leq 2\min(p_i, 1-p_i).
    \end{align*}
    \label{lemma:CoordinateWiseSPMztiIsGoodForStochasticBandits}
\end{lemma}
\begin{proof}
    By the definition of $I$, the left-hand side is equal to
    \begin{align*}
        \E_{I \sim p}\left[\I{I = i}p_i^\alpha \min\left( p_i^{-\alpha}, \frac{1-p_i}{p_i^2} \right)\right] &= p_i^{1+\alpha} \min\left( p_i^{-\alpha}, \frac{1-p_i}{p_i^2} \right) \\
        &= \min\left(p_i, \frac{(1-p_i)}{p_i^{1-\alpha}}\right).
    \end{align*}
    We consider two cases: $p_i \leq \frac{1}{2}$ and $p_i > \frac{1}{2}$.
    \begin{itemize}
        \item If $p_i \leq \frac{1}{2}$: since $p_i^{1-\alpha} \leq 1$, we have $\frac{1-p_i}{p_i^{1-\alpha}} \geq 1-p_i \geq p_i$. Hence, $\min\left(p_i, \frac{(1-p_i)}{p_i^{1-\alpha}}\right) = p_i \leq 2p_i = 2\min(p_i, 1-p_i)$.
        \item If $p_i > \frac{1}{2}$: we then have $\frac{(1-p_i)}{p_i^{1-\alpha}} \leq \frac{1-p_i}{p_i} \leq 2(1-p_i)$. Therefore, $\min\left(p_i, \frac{(1-p_i)}{p_i^{1-\alpha}}\right) \leq \min(p_i, 2(1-p_i)) \leq 2\min(p_i, 1-p_i)$.
    \end{itemize}
\end{proof}

\section{SPM for Adversarial Sleeping Bandits}
\label{appendix:SPMSleepingBandits}
\begin{algorithm}[t]
	\KwIn{$K \geq 1, T \geq 4K, \alpha \in (0, 1), \beta_1 = \frac{8K}{1-\alpha}, \gamma = \max\left(6, 48\sqrtfrac{\alpha}{1-\alpha}\right), d = 2$.}
    Initialize $R_{0,i} = 0$ for $i \in [K]$ \;
	
    \For{each round $t = 1, \dots, T$}{    
        The adversary reveals the set of active arms $\sA_t$\;    
        
        Compute $q_t = \argmin_{x \in \Delta_K} \inp{-R_{t-1}}{x} + \beta_t\left(\frac{1}{\alpha}(1-\sum_{i=1}^K x_i^\alpha)\right) - \gamma\sum_{i=1}^K \ln(x_i)$\;
        
        \For{arm $i \in [K]$}
        {
            Compute $p_{t,i} = \frac{q_{t,i}\I{i \in \sA_t}}{\sum_{j=1}^{K}q_{t,j}\I{j \in \sA_t}}$\;
        }        
		
        Draw $I_t \sim p_t$ and observe $\ell_{t, I_t}$\;
        
        \For{active arm $i \in [K]$}
        {
            \If{$i \in \sA_t$}
            {
                Compute loss estimate $\hat{\ell}_{t, i} = \frac{\ell_{t,i}\I{I_t = i}}{p_{t,i}}$ \;
            }
            \Else{
                Set $\hat{\ell}_{t,i} = \ell_{t,I_t}$\;
            }
            
            Compute $r_{t,i} = \ell_{t,I_t} - \hat{\ell}_{t, i}$\;
            
            Update $R_{t,i} = R_{t-1,i} + r_t$\;
        }	
        Compute $z_t =  \min\left(\frac{(4d)^{2-\alpha}}{(1-\alpha)}\sum_{i \in \sA_t}\min(p_{t,i}, 1-p_{t,i})^{1-\alpha}, \beta_t\frac{18d^2}{\gamma}\sum_{i \in \sA_t}\tilde{p}_{t,i}\right)$ \;
        
        Compute $h_t = \frac{1}{\alpha}(\sum_{i=1}^K q_{t,i}^\alpha - 1)$\;
        
        Compute $\beta_{t+1} = \beta_t + \frac{z_t}{\beta_t h_t}$\;
	}
	\caption{SPM Approach for Fully-Adversarial Sleeping Bandits}
	\label{algo:SPMSleepingBandit}
\end{algorithm}


\textbf{Algorithm:} 
We use the same regularization function in Algorithm~\ref{algo:optimalBOBWSparsity},
\begin{align*}
    \Phi_t(p) &= \beta_t\psi_{TE}(p) - \gamma \psi_{LB}(p) \\
    &= \frac{\beta_t}{\alpha}\left(1 - \sum_{i=1}^K p_i^{\alpha}\right) - \gamma\sum_{i=1}^K \ln(p_i).
\end{align*}
Instead of running FTRL on the sequence of losses, we run FTRL on the sequence of \emph{estimated regret} $R_{t,i} = \sum_{s=1}^t \I{i \in \sA_s}(\ell_{s,I_s} - \hat{\ell}_{s,i})$, i.e., 
\begin{align*}
    q_t = \argmin_{x \in \Delta_K}F_t(x) &\defeq \argmin_{x \in \Delta_K} \inp{-R_{t-1}}{x} + \Phi_t(x) \\
    &= \argmin_{x \in \Delta_K} \inp{-R_{t-1}}{x} + \beta_t\left(\frac{1}{\alpha}(1-\sum_{i=1}^K x_i^\alpha)\right) - \gamma\sum_{i=1}^K \ln(x_i).
\end{align*}
Given the set of active arms $\sA_t$, the sampling probability $p_t$ is $p_{t,i} = \frac{\I{i \in \sA_t}}{\sum_{j=1}^K \I{j \in \sA_t}q_{t,j}}$. An arm $I_t \sim p_t$ is drawn.
The learning rates are set by SPM rules~\citep{ItoCOLT2024}: $\beta_{t+1} = \beta_t + \frac{z_t}{\beta_t h_t}$, where 
\begin{align*}
    z_t &= \min\left(\frac{(4d)^{2-\alpha}}{(1-\alpha)}\sum_{i \in \sA_t}\min(p_{t,i}, 1-p_{t,i})^{1-\alpha}, \beta_t\frac{18d^2 }{\gamma}\sum_{i \in \sA_t}\tilde{p}_{t,i}\right), \\
    h_{t} &= (-\psi_{TE}(q_t)) = \frac{1}{\alpha}\left(\sum_{i=1}^K q_{t,i}^{\alpha} - 1\right).
\end{align*}

\subsection{Regret Analysis}
In this section, we prove Theorem~\ref{thm:SPMSleepingBanditBound}.
\begin{proof}
Let $I_{a,t} = \I{a \in \sA_t}$. By the definition of $\hat{\ell}_t$ and the fact that $p_{t,i} = 0$ for $i \notin \sA_t$, for any $a \in \sA_t$, we have 
\begin{align}
    \E_{I_t}[\hat{\ell}_{t,a}] = \sum_{i=1}^K p_{t,i}\frac{\ell_{t,a}\I{a = i}}{p_{t,a}} = \sum_{i \in \sA_t}p_{t,i} \frac{\ell_{t,a}\I{a = i}}{p_{t,a}} = \ell_{t,a}.
\end{align}
Therefore, the per-action regret with respect to $a \in [K]$ is
\begin{align*}
    R_{T,a} &= \E\left[\sum_{t=1}^T I_{a,t}(\ell_{t,I_t} - \ell_{t,a})\right] \\
    &= \E\left[ \sum_{t=1}^T I_{a,t}(\inp{\hat{\ell}_t}{q_t} - \inp{\ell_t}{e_a}) \right] \\
    &= \E\left[ \sum_{t=1}^T I_{a,t}(\inp{\hat{\ell}_t}{q_t-e_a}) \right] \\
    &= \E\left[ \sum_{t=1}^T I_{a,t}(\inp{\hat{\ell}_t - \ell_{t,I_t}\1}{q_t-e_a}) \right] \\
    &= \E\left[ \sum_{t=1}^T \inp{-r_t}{q_t-e_a} \right],
\end{align*}
where 
\begin{itemize}
    \item the second equality uses
    \begin{align*}
        \inp{\hat{\ell}_t}{q_t} &= \sum_{i=1}^K \hat{\ell}_{t,i}q_{t,i} \\
        &= \sum_{i \in \sA_t} \hat{\ell}_{t,i}q_{t,i} + \sum_{i \notin \sA_t} \hat{\ell}_{t,i}q_{t,i} \\
        &= \hat{\ell}_{t,I_t}q_{t,I_t} + \ell_{t,I_t}\sum_{i \notin \sA_t}q_{t,i} \\
        &= \frac{{\ell}_{t,I_t}}{p_{t,I_t}}q_{t,I_t} + \ell_{t,I_t}\sum_{i \notin \sA_t}q_{t,i} \\
        &= \ell_{t,I_t}\sum_{j \in \sA_t}q_{t,j} + \ell_{t,I_t}\sum_{i \notin \sA_t}q_{t,i} \\
        &= \ell_{t,I_t}.
    \end{align*}
    \item the fourth equality uses $\inp{\ell_{t,I_t}\1}{q_t - e_a} = \ell_{t,I_t}(\sum_{i=1}^K q_{t,i} - e_{a,i}) = 0$.
    \item the last equality uses 
    \begin{align*}
        r_{t,i} = \begin{cases}
            \ell_{t,I_t} - \hat{\ell}_{t,i} &\qquad i \in \sA_t \\
            0 &\qquad i \notin \sA_t.
        \end{cases}
    \end{align*}
\end{itemize} 
Let $u_a = (1 - \frac{K}{T})e_a + \frac{1}{T}\1$. We have
\begin{align*}
    \sum_{t=1}^T \inp{-r_t}{q_t - e_a} = \sum_{t=1}^T \inp{-r_t}{q_t - u_a} + \sum_{t=1}^T \inp{-r_t}{u_a - e_a}.
\end{align*}
The expectation of the second term is bounded by
\begin{align*}
    \E[\sum_{t=1}^T \inp{-r_t}{u_a - e_a}] &= \sum_{t=1}^T \inp{\E[-r_t]}{u_a - e_a} \\
    &= \sum_{t=1}^T \sum_{j=1}^K \E[-r_{t,j}](u_{a,i}- e_{a,i}) \\
    &= \sum_{t=1}^T \sum_{j=1}^K \E[\hat{\ell}_{t,j} - \ell_{t,I_t}](u_{a,i}- e_{a,i}) \\
    &= \sum_{t=1}^T \sum_{j=1}^K (\ell_{t,j} - \E_{i \sim p_t}[\ell_{t,i}])(u_{a,i}- e_{a,i}) \\
    &\leq 2K.
\end{align*}
Therefore, we only need to focus on the first term $\sum_{t=1}^T \inp{-r_t}{q_t - u_a}$.
Next, by the definition of $F_t(x) = \inp{\sum_{s=1}^{t-1}-r_s}{x} + \phi_{t}(x)$ and $q_t = \argmin_{x \in \Delta_K}F_t(x)$, we have 
\begin{align*}
    &\sum_{t=1}^T \inp{-r_t}{q_t - u_a} \\
    &= \left(\sum_{t=1}^T\inp{-r_t}{q_t}\right) - (F_{T+1}(u_a) - \phi_{T+1}(u_a)) \\
    &= \left(\sum_{t=1}^T\inp{-r_t}{q_t}\right) - F_1(q_1) + \left(\sum_{t=1}^T F_t(q_t) - F_{t+1}(q_{t+1}) \right) + \phi_{T+1}(u_a) + F_{T+1}(q_{T+1}) - F_{T+1}(u_a) \\
    &\leq \left(\sum_{t=1}^T\inp{-r_t}{q_t}\right) - F_1(q_1) + \left(\sum_{t=1}^T F_t(q_t) - F_{t+1}(q_{t+1}) \right) + \phi_{T+1}(u_a)  \\
    &= \phi_{T+1}(u_a) - F_1(q_1) + \left( \sum_{t=1}^T F_t(q_t) - F_{t+1}(q_{t+1}) + \inp{-r_t}{q_t} \right) \\
    &\leq \gamma K\ln(T) + \frac{\beta_1}{\alpha}(K^{1-\alpha}-1) + \underbrace{\left( \sum_{t=1}^T F_t(q_t) - F_{t+1}(q_{t+1}) + \inp{-r_t}{q_t} \right)}_{\heartsuit}.
\end{align*}
We bound each term in $\heartsuit$ as follows:
\begin{align*}
    &F_t(q_t) - F_{t+1}(q_{t+1}) + \inp{-r_t}{q_t} \\
    &= \inp{-R_{t-1}}{q_t} + \inp{R_t}{q_{t+1}} + \phi_t(q_t) - \phi_{t+1}(q_{t+1}) + \inp{-r_t}{q_t} \\
    &= \inp{-R_{t-1}}{q_t - q_{t+1}} + \phi_t(q_t) - \phi_{t}(q_{t+1}) + \phi_{t}(q_{t+1}) - \phi_{t+1}(q_{t+1}) + \inp{-r_t}{q_t - q_{t+1}} \\
    &= (\beta_{t+1} - \beta_t)h_{t+1} + (\inp{-R_{t-1}}{q_t - q_{t+1}} + \phi_t(q_t) - \phi_{t}(q_{t+1})) \\
    &\quad + \inp{-r_t}{q_t - q_{t+1}}\\
    &\leq (\beta_{t+1} - \beta_t)h_{t+1} + \inp{-r_t}{q_t - q_{t+1}} - D_t(q_{t+1}, q_t),
\end{align*}
where the inequality is from $\inp{-R_{t-1} + \nabla \phi_t(q_t)}{q_{t+1} - q_t} \geq 0$ by the optimality of $q_t$ and hence,
\begin{align*}
    -D_t(q_{t+1}, q_t) &= \phi_t(q_t)  -\phi_t(q_{t+1}) + \inp{\nabla \phi_t(q_t)}{q_{t+1} - q_t} \\
    &\geq \phi_t(q_t) - \phi_t(q_{t+1})  + \inp{-R_{t-1}}{q_t - q_{t+1}}.
\end{align*}
We have $q_{t+1,i} \leq 4dq_{t,i}$ for all $i \in [K]$ from the combination of the results of Lemma~\ref{lemma:betatplus1isgoodSB}, Lemma~\ref{lemma:stableSameLossDiffBeta} and Lemma~\ref{lemma:stableSameBetaDiffLossSB}.
It follows that 
\begin{align*}
    &\sum_{t=1}^T \inp{-r_t}{q_t - u_a} \\
    &\leq \gamma K\ln(T) + \frac{\beta_1}{\alpha}(K^{1-\alpha}-1) + \sum_{t=1}^T (\beta_{t+1} - \beta_t)h_{t+1} + \sum_{t=1}^T\inp{-r_t}{q_t - q_{t+1}} - D_t(q_{t+1}, q_t) \\
    &\leq  \gamma K\ln(T) + \frac{\beta_1}{\alpha}(K^{1-\alpha}-1) + 4d\sum_{t=1}^T (\beta_{t+1} - \beta_t)h_{t} +  \sum_{t=1}^T\inp{-r_t}{q_t - q_{t+1}} - D_t(q_{t+1}, q_t) \\
    &= \gamma K\ln(T) + \frac{\beta_1}{\alpha}(K^{1-\alpha}-1) + 4d\sum_{t=1}^T \frac{z_t}{\beta_t} +  \sum_{t=1}^T\inp{-r_t}{q_t - q_{t+1}} - D_t(q_{t+1}, q_t),
\end{align*}
where the second inequality is from Lemma~\ref{lemma:boundhtplus1byht}.

Using Lemma~\ref{lemma:stableItNotMaxSB} and noting that $\beta_t$ is fixed before round $t$, we have
\begin{align*}
    \E_{I_t}\left[ \inp{-r_t}{q_t - q_{t+1}} - D_t(q_{t+1}, q_t) \right] &\leq \E_{I_t}\left[\min\left( \frac{(4d)^{2-\alpha}}{\beta_t(1-\alpha)}\tilde{p}_{t,I_t}^{2-\alpha}\hat{\ell}^2_{t,I_t}, \frac{18d^2}{\gamma} \tilde{p}_{t,I_t}^2 \hat{\ell}_{t,I_t}^2  \right)\right] \\
    &\leq \min\left(\E_{I_t}\left[ \frac{(4d)^{2-\alpha}}{\beta_t(1-\alpha)}\tilde{p}_{t,I_t}^{2-\alpha}\hat{\ell}^2_{t,I_t} \right], \E_{I_t}\left[ \frac{18d^2}{\gamma} \tilde{p}_{t,I_t}^2 \hat{\ell}_{t,I_t}^2 \right]\right) \\
    &\leq \min\left( \frac{(4d)^{2-\alpha}}{\beta_t(1-\alpha)}\sum_{i \in \sA_t}\tilde{p}_{t,i}^{1-\alpha}, \frac{18d^2}{\gamma}\sum_{i \in \sA_t}\tilde{p}_{t,i} \right) \\
    &= z_t.
\end{align*}
It follows that
\begin{align*}
    \E\left[\sum_{t=1}^T \inp{-r_t}{q_t - u_a}\right] &\leq \gamma K\ln(T) + \frac{\beta_1}{\alpha}(K^{1-\alpha}-1) + 6\E\left[\sum_{t=1}^T \frac{z_t}{\beta_t} \right].
\end{align*}
Let 
\begin{align*}
    z_{\max} &= \max_{t \in [T]}z_t, \\
    h_{\max} &= \max_{t \in [T]}h_t \\
    A &= \max_{t \in [T]}\abs{\sA_t}
\end{align*}
The quantity $z_{\max}$ is bounded by
\begin{align*}
    z_{\max} &\leq \frac{(4d)^{2-\alpha}}{1-\alpha}\max_{t \in [T]} \sum_{i=1}^K \tilde{p}_{t,i}^{1-\alpha} \\
    &\leq  \frac{(4d)^{2-\alpha}}{1-\alpha}\max_{t \in [T]} \sum_{i \in \sA_t} \tilde{p}_{t,i}^{1-\alpha} \\
    &\leq \frac{(4d)^{2-\alpha}A^\alpha}{1-\alpha}.
\end{align*}
Hence, we can bound $\E\left[\sum_{t=1}^T \frac{z_t}{\beta_t} \right]$ using the same analysis for SPM learning rates in~\cite{ItoCOLT2024} and obtain
\begin{nalign}
    &\E\left[\sum_{t=1}^T \frac{z'_t}{\beta_t} \right] \\ &\leq O\left(\min\left\{ \inf_{J \in \sN}\E\left[\left\{ \sqrt{8J\sum_{t=1}^T h_t z_t} + 2\sqrt{2^{-J}Th_{\max}{z_{\max}}} \right\}\right], \E\left[\sqrt{Th_{\max}z_{\max} }\right] \right\} + \E\left[\frac{z_{\max}}{\beta_1}\right]\right) \\
    &\leq O\left( \min\left\{ \inf_{J \in \sN}\E\left[\left\{ \sqrt{8J\sum_{t=1}^T h_t z_t} + \sqrt{2^{-J}Th_{\max}{z_{\max}}} \right\}\right], \E\left[\sqrt{Th_{\max}z_{\max} }\right] \right\}\right),
    \label{eq:boundzprimetbetatSB}
\end{nalign}
where the second inequality is due to $\frac{z_{\max}}{\beta_1} \leq O\left(\frac{A^\alpha}{(1-\alpha)\frac{4K}{1-\alpha}}\right) = O(1)$. Here, we used
\begin{align*}
    h_{\max} &= \max_{t \in [T]}h_t \\
    &= \frac{1}{\alpha}(\sum_{i=1}^K q_{t,i}^{\alpha} - 1) \\
    &\leq \frac{K^{1-\alpha} - 1}{\alpha}.
\end{align*}
In the adversarial regime, we have 
\begin{align*}
    \E\left[\sum_{t=1}^T \frac{z'_t}{\beta_t} \right] &\leq O(\left(\E[\sqrt{T h_{\max} z_{\max}}]\right)) \\
    &\leq O\left(\sqrt{T\frac{(K^{1-\alpha} - 1)A^\alpha}{\alpha(1-\alpha)}}\right).
\end{align*}
Hence, the regret bound is of order 
\begin{align*}
    \E[R_{T,a}] \leq O\left(\sqrt{T\frac{(K^{1-\alpha} - 1)A^\alpha}{\alpha(1-\alpha)}}\right).
\end{align*}
\end{proof}

\subsection{Stability Proofs}
Recall from Section~\ref{sec:stabilityproofs} that the function $g: [0,1] \to \R_+$ defined by
\begin{align*}
    g(x) = \beta x^{\alpha-1} + \frac{\gamma}{x}
\end{align*}
is decreasing in $x \in [0, 1]$ for $\beta, \gamma > 0$.
\begin{lemma}
    For any $t \geq 1$, Algorithm~\ref{algo:SPMSleepingBandit} guarantees
    \begin{align}
        \beta_{t+1} - \beta_t \leq (1 - \frac{1}{d})\gamma q_{t*}^{-\alpha},
    \end{align}
    where $q_{t*} = \min(\max_{i \in \sA_t}q_{t,i}, 1 - \max_{i \in \sA_t}q_{t,i})$.
    \label{lemma:betatplus1isgoodSB}
\end{lemma}
\begin{proof}
    Equation~\eqref{eq:lowerboundht} shows that $h_t \geq \frac{1-\alpha}{4\alpha}q_{t*}^{\alpha}$. This implies that $\frac{1}{h_t} \leq \frac{4\alpha}{1-\alpha}q_{t*}^{-\alpha}$.
    By the definitions of $\beta_{t+1}, z_t$ and $h_t$, we have 
    \begin{align*}
        \beta_{t+1} - \beta_t &= \frac{z_t}{\beta_t h_t} \\
        &\leq \frac{4\alpha z_t}{(1-\alpha)\beta_t}q_{t*}^{-\alpha} \\
        &\leq \frac{4\alpha}{(1-\alpha)} \frac{18d^2}{\gamma}q_{t*}^{-\alpha} \\
        &\leq (1-\frac{1}{d})\gamma q_{t*}^{-\alpha}
    \end{align*}
    where the last inequality uses
    \begin{align}
        \frac{72\alpha d^2}{(1-\alpha)\gamma} \leq (1-\frac{1}{d})\gamma
    \end{align}
    for $d = 2$ and $\gamma \geq 48\sqrtfrac{\alpha}{1-\alpha}$.
\end{proof}
\begin{lemma}
    For any $K \geq 3, \alpha \in (0,1), \beta > 0, \gamma \geq 0, R \in \R^K$ and $h \in [-1,1]$, let $S \subseteq [K]$ be a subset of $[K]$ where $1 \in S$. Let $e_S \in \{0,1\}^K$ be a vector such that $e_{S,i} = \I{i \in S}$. Define 
    \begin{align*}
        x &= \argmin_{p \in \Delta_K}\inp{-R}{p} + \frac{\beta}{\alpha}(1 - \sum_{i=1}^K p_i^\alpha) - \gamma\sum_{i=1}^K \ln(p_i) \\
        y &= \argmin_{p \in \Delta_K}\inp{-R + \frac{h}{x_1'}e_1 - he_S}{p} + \frac{\beta}{\alpha}(1 - \sum_{i=1}^K p_i^\alpha) - \gamma\sum_{i=1}^K \ln(p_i),
    \end{align*}
    where $1 \geq x'_1 \geq x_1$. Fix an $\omega \in (1, 2]$. If $\gamma \geq 6$ and $\beta \geq \frac{4K}{(\omega - 1)(1 - \omega^{\alpha - 1})}$, then $y_i \leq 4x_i$ for all $i \in [K]$.
    \label{lemma:stableSameBetaDiffLossSB}
\end{lemma}
\begin{proof}
    Using the Lagrange multiplier methods, we have the following three equalities that hold for some $Z \in \R$:
    \begin{align}
        g(x_1) - g(y_1) &= Z + h - \frac{h}{x'_1}, \\
        g(x_i) - g(y_i) &= Z + h \qquad\text{for } i \in S \setminus \{1\}, \\
        g(x_i) - g(y_i) &= Z \qquad\text{for } i \notin S.
    \end{align}
    \subsection*{When $h \leq 0$:}
    First, we prove that $Z + h \leq 0$. Assume the contrary that $Z + h \geq 0$. Since $h \in [-1, 0]$, this implies that $Z > 0$ and $Z + h + \frac{(-h)}{x'_1} > 0$. Hence, $g(x_i) > g(y_i)$ for all $i \in [K]$, which is a contradiction since both $x$ and $y$ are in $\Delta_K$. Thus, we must have $Z + h \leq 0$.

    For any $i \in S \setminus \{1\}$, we have $g(x_i) - g(y_i) = Z + h \leq 0$ and therefore $y_i \leq x_i$. 
    Next, we consider two cases of $Z$: $Z \geq 0$ and $Z < 0$.
    \begin{itemize}
        \item If $Z \geq 0$: we have $0 \leq Z \leq -h \leq 1$. For all $i \notin S$, we have $g(y_i) = g(x_i) - Z \geq g(x_i) - 1 \geq g(2x_i)$ by Lemma~\ref{lemma:gxminus1Largerg2x}, which implies $y_i \leq 2x_i$. Thus, we only need to show that $y_1 \leq 2x_1$. If $y_1 \leq x_1$ then this is trivially true. If $y_1 \geq x_1$,
        \begin{align*}
            \frac{2}{x_1} \geq \frac{-h}{x'_1} &= (-(Z+h)) + \beta(x^{\alpha-1} - y^{\alpha-1}) + \gamma\left(\frac{1}{x_1} - \frac{1}{y_1}\right) \\
            &\geq \gamma\left(\frac{1}{x_1} - \frac{1}{y_1}\right) \\
            &\geq 4\left(\frac{1}{x_1} - \frac{1}{y_1}\right),
        \end{align*}
        which leads to $y_1 \leq 2x_1$. 
        \item If $Z < 0$: in this case, for all $i \notin S$, we have $g(x_i) \leq g(y_i)$ which implies $x_i \geq y_i$.  As $x_i \geq y_i$ for all $i \neq 1$, we must have $x_1 \leq y_1$. Therefore, 
        we have $y_1 \leq 2x_1$ by the same argument in the previous case.
    \end{itemize}

    \subsection*{When $h \geq 0$:}
    First, we prove that $Z + h \geq 0$. Assume the contrary that $Z + h < 0$. Since $h \in [0,1]$, this implies $Z < 0$ and $Z + h - \frac{h}{x'_1} < 0$. Hence, $g(x_i) < g(y_i)$ for all $i \in [K]$, which is a contradiction since both $x$ and $y$ are in $\Delta_K$. Thus, we must have $Z + h \geq 0$.

    For any $i \in S \setminus \{1\}$, we have $g(x_i) - g(y_i) = Z + h \geq 0$ and therefore $x_i \leq y_i$. Next, we consider two cases of $Z$: $Z \geq 0$ and $Z < 0$.
    \begin{itemize}
        \item If $Z \geq 0$: in this case, due to the monotonicity of the function $g$, we have $x_i \leq y_i$ for $i \neq 1$ and therefore, $x_1 \geq y_1$. This implies $Z + h - \frac{h}{x'_1} \leq 0$. Let 
        \begin{align*}
            \eps = \frac{1}{\beta(1 - \omega^{\alpha-1})} \leq \frac{\omega - 1}{4K} \leq \frac{1}{K}
        \end{align*}
        as in the proof of Lemma~\ref{lemma:stableBigBeta}. We further consider two cases of $x'_1$.
        \begin{itemize}
            \item If $x'_1 \geq \eps$: we have $Z \leq Z + h \leq \frac{h}{x'_1} \leq \frac{1}{\eps}$. Therefore, for all $i \neq 1$, 
            \begin{align*}
                g(y_i) &= g(x_i) - Z - h\I{i \in S}  \\
                &\geq g(x_i) - Z - h \\
                &\geq g(x_i) - \frac{1}{\eps} \\
                &= \beta x_i^{\alpha - 1} - \beta(1-\omega^{\alpha-1}) + \frac{\gamma}{x_i} \\
                &\geq \beta x_i^{\alpha - 1} - \beta x_i^{\alpha-1}(1-\omega^{\alpha-1}) + \frac{\gamma}{\omega x_i} \\
                &= \beta (\omega x_i)^{\alpha-1} + \frac{\gamma}{\omega x_i} = g(\omega x_i),
            \end{align*}
            where the last inequality is due to $x_i^{\alpha-1} \geq 1$ and $\omega > 1$. This implies that for all $i \neq 1$, $y_i \leq \omega x_i \leq 3x_i$ since $\omega \leq 2$.

            \item If $x'_1 < \eps$: we have $x_1 \leq \eps \frac{1}{2K}$ and $\sum_{i=1}^K (y_i - x_i) = x_1 - y_1 \leq \eps \frac{\omega-1}{K}$. Let $i* = \argmax_{i \in [K]}x_i$. We have $x_{i*} \geq \frac{1}{K} > \frac{1}{2K}$, hence $i* \neq 1$. Furthermore,
            \begin{align*}
                \frac{1}{K}\left(\frac{y_{i*}}{x_{i*}} - 1\right) &\leq x_{i*}\left(\frac{y_{i*}}{x_{i*}} - 1\right) \\ 
                &= y_{i*} - x_{i*} \\
                &\leq \sum_{i \neq 1}(y_i - x_i) \\
                &\leq \frac{\omega-1}{K},
            \end{align*}
            which implies that $y_{i*} \leq \omega x_{i*}$. Therefore, using the fact that $g(x) - g(\omega x)$ is also decreasing in $x$, for all $i \neq 1$, we have
            \begin{align*}
                g(y_i) &= g(x_i) - (Z + h\I{i \in S}) \\
                &\geq g(x_i) - Z - 1 \qquad\text{ since } h \in [0,1]\\
                &\geq g(x_i) - (g(x_{i*}) - g(y_{i*})) - 1 \\
                &\geq g(x_i) - (g(x_{i*}) - g(\omega x_{i*})) - 1 \\
                &\geq g(x_i) - (g(x_i) - g(\omega x_i)) - 1 \\
                &= g(\omega x_i) - 1 \\
                &\geq g(2\omega x_i),
            \end{align*}
            where the second inequality is from $Z \leq Z + h\I{i^* \in S} = g(x_{i*}) - g(y_{i*})$, the third inequality is from $g(y_{i*}) \geq g(\omega x_{i*})$, and the last inequality is $g(x) - 1 \geq g(2x)$ by Lemma~\ref{lemma:gxminus1Largerg2x}.
            From $g(y_i) \geq g(2\omega x_i)$, we conclude that $y_i \leq 2\omega x_i \leq 4x_i$ for all $i \neq 1$.
        \end{itemize} 
        \item If $Z < 0$: since $x'_1 \leq 1$ and $h \in [0,1]$, we have $Z + h - \frac{h}{x'_1} < 0$. It follows that $g(x_1) - g(y_1) < 0$, hence $x_1 \geq y_1$. Moreover, for $i \notin S$, we also have $x_i \geq y_i$ due to $0 > Z = g(x_i) - g(y_i)$. Thus, we only need to show $y_i \leq 3x_i$ for $i \in S \setminus \{1\}$. For such $i$, we have 
        \begin{align*}
            g(y_i) = g(x_i) - (h + Z) \geq g(x_i) - h \geq g(x_i) - 1 \geq g(2x_i),
        \end{align*}
        where the last inequality is from Lemma~\ref{lemma:gxminus1Largerg2x}. This implies $y_i \leq 2x_i$.
    \end{itemize}
\end{proof}
\begin{lemma}
    For any $t \in [T]$ and constant $c \in [0,1]$, Algorithm~\ref{algo:SPMSleepingBandit} guarantees
    \begin{align*}
        \sum_{i=1}^K (q_{t,i})^{2-c} r_{t,i}^2 \leq 2(\tilde{p}_{t,I_t})^{2-c}\hat{\ell}_{t,I_t}^2.
    \end{align*}
    \label{lemma:boundqtirtibypItellIt}
\end{lemma}
\begin{proof}
    Since $r_{t,i} = 0$ for $i \notin \sA_t$, $r_{t,i} = \ell_{t,I_t}$ for $i \in \sA_t \setminus \{I_t\}$ and $\hat{\ell}_{t,I_t} = \frac{\ell_{t,I_t}}{p_{t,I_t}}$, we have
    \begin{align*}
        \sum_{i=1}^K {q}_{t,i}^{2-c} r_{t,i}^2 &= \sum_{i \in \sA_t} {q}_{t,i}^{2-c} r_{t,i}^2  \\
        &= \ell_{t,I_t}^2 \sum_{i \in \sA_t, i \neq I_t}{q}_{t,i}^{2-c} + {q}_{t,I_t}^{2-c}\ell_{t,I_t}^2\left(1 - \frac{1}{p_{t,I_t}}\right)^2 \\
        &= \frac{\ell_{t,I_t}^2}{p_{t,I_t}^2}\left(p_{t,I_t}^2\sum_{i \in \sA_t, i \neq I_t}{q}_{t,i}^{2-c} + {q}_{t,I_t}^{2-c} \left(p_{t,I_t} - 1\right)^2\right) \\
        &\leq \frac{\ell_{t,I_t}^2}{p_{t,I_t}^2}\left(p_{t,I_t}^2\sum_{i \in \sA_t, i \neq I_t}p_{t,i}^{2-c} + p_{t,I_t}^{2-c} \left(p_{t,I_t} - 1\right)^2\right) \\
        &\leq \frac{\ell_{t,I_t}^2}{p_{t,I_t}^2}\left(p_{t,I_t}^2\left(\sum_{i \in \sA_t, i \neq I_t}p_{t,i}\right)^{2-c} + p_{t,I_t}^{2-c} \left(p_{t,I_t} - 1\right)^2\right) \\
        &= \hat{\ell}_{t,I_t}^2 (p_{t,I_t}(1-p_{t,I_t}))^{2-c}\left(p_{t,I_t}^c + (1-p_{t,I_t})^c\right) \\
        &\leq 2\hat{\ell}_{t,I_t}^2 (p_{t,I_t}(1-p_{t,I_t}))^{2-c} \\
        &\leq 2\tilde{p}_{t,I_t}^{2-c}\hat{\ell}_{t,I_t}^2,
    \end{align*}
    where the first inequality is due to $q_{t,i} \leq p_{t,i}$ for $i \in \sA_t$, the second inequality is from repeatedly applying $a^x + b^x \leq (a+b)^x$ for $x = 2 - c \geq 1$ by Lemma~\ref{lemma:axbxlargerthanaplusbx}, the third inequality is $p_{t,I_t}^c \leq 1$ and $(1-p_{t,I_t})^c \leq 1$, and the last inequality is $x(1-x) \leq \min(x, 1-x)$ for $x \in [0,1]$.
\end{proof}
\begin{lemma}
    For any $t \in [T]$, Algorithm~\ref{algo:SPMSleepingBandit} guarantees
    \begin{align}
        \inp{-r_t}{q_t - q_{t+1}} - D_t(q_{t+1}, q_t) \leq  \min\left( \frac{(4d)^{2-\alpha}}{\beta_t(1-\alpha)}\tilde{p}_{t,I_t}^{2-\alpha}\hat{\ell}^2_{t,I_t}, \frac{18d^2}{\gamma} \tilde{p}_{t,I_t}^2 \hat{\ell}_{t,I_t}^2  \right)
    \end{align}
    \label{lemma:stableItNotMaxSB}
\end{lemma}
\begin{proof}
    Using standard local-norm analysis techniques for FTRL (for example, see Section 7.4 in~\cite{OrabonaIntroToOnlineLearningBook}), we have
    \begin{align}
        \inp{-r_t}{q_t - q_{t+1}} - D_t(q_{t+1}, q_t) \leq \frac{1}{2}\norm{r_t}^2_{(\nabla^2\phi_t(z_t))^{-1}},
        \label{eq:boundStabilityByLocalNormSB}
    \end{align}
    where $z_t$ is a point between $q_t$ and $q_{t+1}$. The Hessian matrix of $\phi_t$ is a diagonal matrix with entries
    \begin{align}
        \nabla^2\phi_t(z_t) = \mathrm{diag}\left(\left(\beta_t(1-\alpha)z_{t,i}^{\alpha - 2} + \frac{\gamma}{z_{t,i}^2}\right)_{i=1,2,\dots,K}\right).
    \end{align}
    Hence, its inverse is the following diagonal matrix
    \begin{align}
        (\nabla^2\phi(z_t))^{-1} = \mathrm{diag}\left(\left(\frac{1}{\beta_t(1-\alpha)z_{t,i}^{\alpha - 2} + \frac{\gamma}{z_{t,i}^2}}\right)_{i=1,2,\dots,K}\right).
    \end{align}
    It follows that 
    \begin{nalign}
        \norm{r_t}^2_{(\nabla^2\phi_t(z_t))^{-1}} &= \sum_{i=1}^{K} r_{t,i}^2 \frac{1}{\beta_t(1-\alpha)z_{t,i}^{\alpha - 2} + \frac{\gamma}{z_{t,i}^2}} \\
        &\leq \min\left(\frac{1}{\beta_t(1-\alpha)}\sum_{i=1}^{K} z_{t,i}^{2-\alpha}r_{t,i}^2, \frac{1}{\gamma}\sum_{i=1}^K z_{t,i}^2 r_{t,i}^2 \right) \\
        \label{eq:localnormboundSB}
    \end{nalign}
    where the last equality is due to $\hat{\ell}_{t,i} = 0$ for $i \neq I_t$. Combining~\eqref{eq:boundStabilityByLocalNormSB} and~\eqref{eq:localnormboundSB}, we obtain
    \begin{align}
        \inp{-r_t}{q_{t}  - q_{t+1}} - D_t(q_{t+1}, q_t) 
        &\leq  \min\left(\frac{1}{\beta_t(1-\alpha)}\sum_{i=1}^{K} z_{t,i}^{2-\alpha}r_{t,i}^2, \frac{1}{\gamma}\sum_{i=1}^K z_{t,i}^2 r_{t,i}^2 \right).
    \end{align}
    Since $z_t$ is between $q_{t}$ and $q_{t+1}$, we have $z_{t,I_t} \leq \max(q_{t, I_t}, q_{t+1,I_t})$. 
    The loss estimate in Algorithm~\ref{algo:SPMSleepingBandit} uses $p_{t,I_t}$ where $p_{t,I_t} \geq q_{t,I_t}$, therefore we can combine the results of Lemma~\ref{lemma:betatplus1isgoodSB}, Lemma~\ref{lemma:stableSameLossDiffBeta} and Lemma~\ref{lemma:stableSameBetaDiffLossSB} and obtain $q_{t+1, i} \leq 4dq_{t,i}$ for all $i \in [K]$. It follows that $z_{t, i} \leq 4dq_{t,i}$, and as a result,
    \begin{align}
        \inp{\hat{\ell}_t}{q_{t}  - q_{t+1}} - D_t(q_{t+1}, q_t) &\leq \min\left(\frac{(4d)^{2-\alpha}}{2\beta_t(1-\alpha)} \sum_{i=1}^K q_{t,i}^{2-\alpha}r_{t,i}^2, \frac{9d^2}{2\gamma} \sum_{i=1}^K q_{t,i}^2 r_{t,i}^2 \right) \\
        &\leq \min\left( \frac{(4d)^{2-\alpha}}{\beta_t(1-\alpha)}\tilde{p}_{t,I_t}^{2-\alpha}\hat{\ell}^2_{t,I_t}, \frac{18d^2}{\gamma} \tilde{p}_{t,I_t}^2 \hat{\ell}_{t,I_t}^2  \right),
    \end{align}
    where the last inequality is from applying Lemma~\ref{lemma:boundqtirtibypItellIt} twice: once with $c = \alpha$ and once with $c = 0$.
\end{proof}

\subsection{Technical Lemmas}
\begin{lemma}
  For any $x \in [0,1]$, if $\gamma \geq 2$ then $g(x) - 1 \geq g(2x)$.
  \label{lemma:gxminus1Largerg2x}
\end{lemma}
\begin{proof}
  We have
  \begin{align*}
      g(x) - 1 &= \beta x^{\alpha-1} + \frac{\gamma}{x} - 1 \\
      &\geq \beta 2^{\alpha-1}x^{\alpha-1} + \frac{\gamma}{2x} + \frac{\gamma}{2x} - 1 \\
      &= g(2x) + \frac{\gamma - 2x}{2x} \\
      &\geq g(2x).
  \end{align*}
\end{proof}

\section{Setting $\alpha$ appropriately close to $1$}
\label{appendix:setAlphaCloseto1}
Recall that we assume $K \geq 3$ in our algorithms.
Let $b = 1 - \alpha$. We will set $\alpha = 1 - \frac{0.5}{\ln(K)}$, which is equivalent to setting $b = \frac{0.5}{\ln(K)}$. Note that $\alpha \geq 1 - \frac{0.5}{\ln(3)} > 0.5.$
Taking exponent on both sides of 
\begin{align}
    \ln(1 + 2b\ln(K)) = \ln(2) \geq 0.5 = b\ln(K),
\end{align} 
we obtain $1 + 2b\ln(K) \geq K^b$. This implies
\begin{align}
    \frac{K^{1-\alpha} - 1}{\alpha(1-\alpha)} \leq \frac{2(K^b - 1)}{b}\leq 4\ln(K).
\end{align}
Furthermore, 
\begin{align}
    \frac{(K-1)^{1-\alpha}}{\alpha(1-\alpha)} \leq \frac{2(K-1)^{1-\alpha}}{1-\alpha} = \ln(K)(K-1)^{\frac{0.5}{\ln{K}}} \leq \ln(K)K^{\frac{0.5}{\ln{K}}} \leq 2\ln{K},
\end{align}
where the last inequality is due to $K^{\frac{0.5}{\ln(K)}} = (e^{\ln{K}})^{\frac{0.5}{\ln(K)}} = e^{0.5} < 2$.
In addition,
\begin{align}
    \frac{\alpha}{1-\alpha} \leq \frac{1}{1-\alpha} = \frac{1}{b} = 2\ln(K).
\end{align}
This implies that $\gamma = \max\left(6, 48\sqrtfrac{\alpha}{1-\alpha}\right) \lesssim \sqrt{\ln(K)}$.



\end{document}